\documentclass[11pt]{article}

\usepackage{enumitem}
\usepackage{amsthm} 
\RequirePackage{newtxtext}
\RequirePackage{newtxmath}
\RequirePackage{bm}
\RequirePackage{endnotes}
\usepackage{setspace}

\usepackage{mathtools}
\usepackage[ruled,vlined,linesnumbered]{algorithm2e}
\usepackage{float} 
\usepackage{algpseudocode}
\usepackage{booktabs}
\usepackage{tikz}
\usepackage{tabularx}
\usepackage{amsmath}
\usepackage{graphicx}
\usepackage{makecell}
\usepackage{multirow}
\usepackage{subcaption} 
\usepackage[margin=1in]{geometry}
\usepackage{natbib}
 \bibpunct[, ]{(}{)}{,}{a}{}{,}%
\newcommand{\CVaR}{\mathrm{CVaR}}

\newcommand{\E}{\mathbb{E}}

\renewcommand{\[}{\left[}

\renewcommand{\ge}{\geqslant}
\renewcommand{\le}{\leqslant}
\renewcommand{\geq}{\geqslant}
\renewcommand{\leq}{\leqslant}
\renewcommand{\epsilon}{\varepsilon}

\renewcommand{\cdots}{\dots}

\newtheorem{theorem}{Theorem}

\newtheorem{lemma}[theorem]{Lemma}
\newtheorem{remark}[theorem]{Remark}
\newtheorem{proposition}[theorem]{Proposition}
\newtheorem{definition}{Definition}[section]
\newtheorem{assumption}{Assumption}[section]
\newtheorem{example}{Example}[section]

\begin{document}
\onehalfspacing
\title{Conditional Risk Minimization with Side Information: A Tractable, Universal Optimal Transport Framework}

\author{Xinqiao Xie,\; Jonathan Yu-Meng Li\\
\\
Telfer School of Management\\
University of Ottawa \\
Ottawa, Ontario, Canada K1N 6N5\\  xxie087@uottawa.ca, jonathan.li@telfer.uottawa.ca
}
\maketitle

\begin{center}
\textbf{Abstract}
\end{center}
Conditional risk minimization arises in high-stakes decisions where risk must be assessed in light of side information--such as stressed economic conditions, specific customer profiles, or other contextual covariates. Constructing reliable conditional distributions from limited data is notoriously difficult, motivating a series of optimal-transport-based proposals that address this uncertainty in a distributionally robust manner. Yet these approaches remain fragmented, each constrained by its own limitations: some rely on point estimates or restrictive structural assumptions, others apply only to narrow classes of risk measures, and their structural connections are unclear. We introduce a universal framework for distributionally robust conditional risk minimization, built on a novel union-ball formulation in optimal transport. This framework offers three key advantages: interpretability, by subsuming existing methods as special cases and revealing their deep structural links; tractability, by yielding convex reformulations for virtually all major risk functionals studied in the literature; and scalability, by supporting cutting-plane algorithms for large-scale conditional risk problems. Applications to portfolio optimization with rank-dependent expected utility highlight the practical effectiveness of the framework, with conditional models converging to optimal solutions where unconditional ones clearly do not.






\section{Introduction}\label{sec:Intro}
Leveraging contextual side information--such as customer profiles, economic indicators, or individualized health records--has become integral to modern data-driven decision-making across a wide range of business domains. These problems are classically formulated through conditional expectation optimization (e.g., \cite{BK20}):
\begin{align} \label{pro1}
\min_{\alpha \in \mathcal{A}} \, \E_{\mathbb{Q}}\left[\ell(Y, \alpha) \mid X \in \mathcal{N} \right],
\end{align}
where $Y \in \mathbb{R}^m$ denotes an uncertain quantity of interest, $X \in \mathbb{R}^n$ represents covariates encoding side information, and $\alpha \in \mathcal{A}$ is a decision that minimizes the mean loss of $\ell(Y,\alpha)$, conditioned on the covariates $X$ lying in the region $\mathcal{N} \subseteq \mathbb{R}^n$. While analytically convenient, the conditional mean often fails to capture the realities of high-stakes decision-making, where performance is rarely judged by averages alone. Financial institutions use measures such as conditional value-at-risk (CoVaR) to monitor systemic fragility (see \cite{AB11} for more details of CoVaR); healthcare systems prioritize robustness to adverse scenarios (e.g., demand surges or rare severe outcomes); and personalized operations seek guarantees that account for tail risk beyond average performance. To address such diverse risk tasks, we formalize the notion of conditional risk minimization (CRM):
\begin{align} \label{pro2}
\min_{\alpha \in \mathcal{A}} \, \rho_{\mathbb{Q}}\left[\ell(Y, \alpha) \mid X \in \mathcal{N} \right],
\end{align}
where $\rho$ denotes a general risk-aware decision criterion that subsumes expectation as a special case and captures a broad spectrum of nonlinear objectives tailored to different risk preferences.

A fundamental challenge, already present in the classical expectation setting \eqref{pro1}, is that the conditional distribution is not directly observable; in practice, only samples from the joint distribution are available. To bridge this gap, methods to approximate the conditional law range from kernel regression \citep{BR19}, $k$-nearest neighbors, local-linear estimators, and tree-based methods \citep{BK20} to covariate-SAA frameworks that embed prediction models within sample-average approximations \citep{KBL25}. These techniques differ in how side information is incorporated—either by reweighting observed samples or by reshaping the support set—yet in all cases the resulting estimates remain vulnerable to modeling bias and statistical error. Consequently, decisions based on these approximations can undermine out-of-sample reliability, underscoring the need for a more principled framework that directly addresses ambiguity in the conditional distribution.

The success of distributionally robust optimization (DRO) in ensuring reliable out-of-sample performance in unconditional settings has inspired analogous efforts to robustify the conditional expectation problem \eqref{pro1} using optimal transport–based ambiguity sets. Two main modeling approaches have emerged. The first places ambiguity directly over the space of conditional distributions, a regime that we refer to as {\it predict-then-robustify}: a conditional distribution is first estimated using machine-learning techniques, and a Wasserstein ball is then formed around the estimate (see \cite{WCW21}, \cite{WCW24}, \cite{KBL24}). For example, \citet{WCW21} use a kernel-regression-based conditional distribution as the ball’s center, while \citet{KBL24} employ a residual-based construction as the center; \citet{WCW24} further consider the intersection of two Wasserstein balls to capture covariate shifts. While these formulations largely retain the tractability of unconditional Wasserstein DRO and can, in principle, accommodate general risk criteria, their reliance on conditional distribution estimates inherits the attendant bias and error, and they remain largely restricted to singleton conditions, ${\cal N}=\{x_0\}$, rather than arbitrary covariate sets.

The second approach bypasses the prediction step and constructs ambiguity sets directly on the joint distribution, typically centered at the empirical law. Along this line, \citet{N20} and \citet{N24} employ optimal transport balls augmented with constraints that guarantee positive probability mass on the conditioning set. Because conditional expectations are nonlinear in the joint distribution, the resulting problems are more challenging to solve, and extensions to risk criteria lead only to conservative approximations. \citet{E22} takes a different angle by restricting the admissible joint distributions to those supported on the conditioning set in the covariates. While this makes the conditional expectation linear in the distribution, it requires the use of partial mass transport metrics to construct ambiguity sets from the empirical law, whose interpretation remains less transparent. Despite the comparison made in \citet{E22} with the formulation of \citet{N20}, the two models are difficult to relate beyond relative conservativeness—each carrying its own interpretation, with structural connections and limitations not yet fully understood.

Despite their promise, both paradigms fall short of providing a universal, robust, and tractable framework for the conditional risk minimization problem \eqref{pro2}. Predict-then-robustify methods retain the flexibility and tractability of unconditional DRO but are brittle, as their performance hinges on the quality of the prediction step. Joint-distribution approaches avoid this dependence but remain narrow in scope: they are largely confined to conditional expectation problems \eqref{pro1}, and their extensions to richer risk criteria typically yield only conservative relaxations. For instance, \citet{N24} extend their framework to risk measures such as mean–variance and mean–conditional value-at-risk (CVaR), but their formulation applies only to those admitting a finite-dimensional expectation representation—a condition not necessarily satisfied by certain important coherent risk measures. Moreover, their tractable reformulations provide no guarantees of tightness. These challenges underscore the difficulty of extending methods designed for the conditional expectation problem \eqref{pro1} to the broader CRM problem \eqref{pro2}. More fundamentally, existing approaches are built on structurally distinct formulations, leaving open how they can be compared, related, or generalized. These limitations motivate our development of a universal and tractable framework for conditional risk minimization—one that recovers existing paradigms as special cases and enables principled extensions to a broader class of risk criteria.

Our key innovation lies in uncovering a previously unexplored ambiguity-set structure in DRO: the union of optimal transport balls. While other constructions--such as intersections of balls--have been explored, the use of a union as an ambiguity set remains, to our knowledge, entirely absent. Our formulation arises naturally from a core limitation of the predict-then-robustify paradigm, which fixes both the reference distribution and robustness radius despite uncertainty in each. We instead treat both as variables and define the ambiguity set as the union of balls over all admissible (distribution, radius) pairs. This union-ball structure yields three key breakthroughs: (1) it eliminates the brittleness of predict-then-robustify methods; (2) it provides a unified framework that subsumes existing paradigms for conditional optimization as special cases, enabling exact and transparent comparisons through their corresponding union-ball representations; and (3) most importantly, it yields a structural reformulation of worst-case conditional risk in CRM as a maximization over unconditional counterparts, which, as shown throughout this paper, can in turn be reformulated into tractable convex programs for a broad class of risk measures—including those previously beyond reach.

We demonstrate the universality and tractability of our union-ball formulation through a comprehensive suite of conditional risk minimization problems spanning a wide range of risk measures. This includes problems studied in \citet{N24}, such as mean–variance and mean–CVaR, where their approach yields only conservative approximations, while ours delivers exact and simpler solutions. More broadly, our framework applies to risk criteria entirely outside the scope of existing methods, including those lacking finite-dimensional expectation representations. In particular, we establish a tractable distributionally robust conditional formulation of the celebrated rank-dependent expected utility (RDEU) from prospect theory—an important case for which no such formulation was previously available. Finally, we show how our universal framework enables scalable cutting-plane algorithms to solve large-scale conditional RDEU problems with practical efficiency.

\subsection{Related Literature}
Beyond the methods already discussed in the introduction, two additional strands of work are worth noting. The first is the \emph{predict-then-optimize} regime, which does not attempt to estimate the full conditional distribution but instead predicts conditional statistics—such as conditional expectations or risk measures—from side information, and then uses these predictions as inputs to an optimization problem. Examples include \citet{DK17}, \citet{EG22}, and \citet{MPM22}, with further analysis of generalization and stability properties in \citet{EEGT19} and \citet{HKM22}. While effective in certain settings, these methods inherit the fragility of the prediction step, provide no explicit safeguard against distributional ambiguity, and are difficult to extend to general risk minimization problems that require a full distributional specification. 

A second, related but distinct, line of work transforms conditional expectation problems into unconditional ones by embedding side information directly into the decision rules. Notably, \citet{ZYG24} and \citet{YZG22} adopt this paradigm, formulating objectives as expectations under the joint distribution rather than conditional criteria. In particular, \citet{YZG22} introduce a causal transport distance within this decision-rule framework. While these approaches provide an alternative perspective, they shift the problem away from the CRM formulation studied here.

\section{A Universal OT Framework for Conditional Optimization with Side Information}\label{main-result}
\subsection{Preliminaries and Existing OT-Based Approaches}\label{other_model}
To set the stage for our universal framework, we first recap the three existing {optimal transport (OT)-based} approaches through their precise formulations, along with the minimal preliminaries needed to state them formally. Recall the following definition of optimal transport cost:
\begin{definition}[Optimal transport cost]
Let $\Xi \subseteq \mathbb{R}^d$ be a Polish space for some integer $d$, and let $\mathbb{D} : \Xi \times \Xi \to \mathbb{R}_+$ be a continuous cost function. The optimal transport cost between two probability distributions $\mathbb{Q}_1$ and $\mathbb{Q}_2$ supported on $\Xi$ is defined as
$$
\mathcal{W}(\mathbb{Q}_1, \mathbb{Q}_2) \triangleq \min_{\pi \in \Pi(\mathbb{Q}_1, \mathbb{Q}_2)} \mathbb{E}_{\pi}[\mathbb{D}(\xi_1, \xi_2)],
$$
where $\Pi(\mathbb{Q}_1, \mathbb{Q}_2)$ denotes the set of all couplings of $\mathbb{Q}_1$ and $\mathbb{Q}_2$ on $\Xi \times \Xi$.\footnote{Here we adopt the minimization form not the infimum form to define the optimal transport cost. This choice is appropriate because, when the underlying space $\Xi$ is Polish, the existence of an optimal joint distribution is guaranteed by Theorem 4.1 of \cite{V08}.}
\end{definition}

An OT ball is defined as
$$
\mathbb{B}_{\delta}(\widehat{\mathbb{P}}) = \left\{ \mathbb{Q} \in \mathcal{P}(\Xi) : \mathcal{W}(\mathbb{Q}, \widehat{\mathbb{P}}) \leq \delta \right\},
$$
where $\widehat{\mathbb{P}}$ is the reference distribution and $\delta\ge 0$ is the ball radius.
It consists of all distributions within $\delta$-distance from $\widehat{\mathbb{P}}$ under the cost $\mathcal{W}$.

We begin with the \textit{Predict-then-Robustify} approach as a baseline, given its conceptual simplicity, followed by the more sophisticated formulations developed in \citet{N24} and \citet{E22}. For a random pair $(X,Y)$ on $\mathcal{X}\times\mathcal{Y}$ with joint distribution $\mathbb{P}$, we denote the conditional law of $Y$ given $X$ by ${\mathbb{P}}_{Y|X}$. To emphasize it as a distribution of $Y$ corresponding to a fixed covariate value $X=x$, we write ${\mathbb{P}}^{{x}}|_{Y}$. More generally, we use 
${\mathbb{P}}|_{Y}$ to denote the distribution of $Y$ under a conditioning event on $X$.

\paragraph{{\bf (1) Predict-then-Robustify.}}
This baseline approach can readily accommodate solving the conditional risk minimization (CRM) problem~\eqref{pro2} for general risk criteria, given its plug-and-play nature. Given $N$ samples $(\widehat{X},\widehat{Y})=\{(\widehat{x}_i, \widehat{y}_i)\}_{i=1}^N$ drawn from the joint distribution, one first applies a predictive model (e.g., kernel regression) to estimate the conditional distribution $\widehat{\mathbb{P}}_{Y|X}$, and then solves a standard DRO problem with a OT ball centered at the estimate $\widehat{\mathbb{P}}^{x_0}|_{Y}$:
\begin{align}
\min_{\alpha \in \mathcal{A}} \sup_{\mathbb{Q}|_{Y} \in \mathbb{B}_{\widehat{\delta}}(\widehat{\mathbb{P}}^{x_0}|_{Y})} \rho_{\mathbb{Q}|_{Y}}\left[\ell(Y, \alpha)\right], \label{kernel}
\end{align}
where $\mathbb{B}_{\widehat{\delta}}(\widehat{\mathbb{P}}^{x_0}|_{Y})$ is an OT ball of radius $\widehat{\delta}$ centered at $\widehat{\mathbb{P}}^{x_0}|_{Y}$. However, this setup inherits a modeling bias tied to the choice of predictive method (e.g., kernel type, bandwidth), which propagates through the decision process. Moreover, it is typically limited to the singleton case $\mathcal{N} = \{{x_0}\}$ in~\eqref{pro2}, since most predictive models are not designed to output conditional distributions over non-singleton covariate sets.

\paragraph{{\bf (2) Full Optimal Transport.}} 
The approach of \cite{N24} circumvents the need to estimate a conditional distribution by placing the ambiguity set over joint distributions $\mathbb{Q}$ supported on $\mathcal{X}\times\mathcal{Y}$. Let 
$\mathbb{D}_\mathcal{X} : \mathcal{X} \times \mathcal{X} \to \mathbb{R}_+$ denote a distance matric on $\mathcal{X}$. For $x_0\in\mathcal{X}$ and $\gamma\ge0$, define the covariate region $\mathcal{N}_\gamma(x_0) \triangleq \left\{
x \in \mathcal{X} : \mathbb{D}_{\mathcal{X}}(x, x_0) \leq \gamma
\right\}$. Given such a region ${\cal N}:=\mathcal{N}_\gamma(x_0)$, they propose the
following distributionally robust counterpart to problem~\eqref{pro1}:
\begin{align}\label{m01}
\min_{\alpha \in \mathcal{A}} \;
\sup_{\substack{\mathbb{Q} \in \mathbb{B}_{\delta_0}(\widehat{\mathbb{P}}) \\ \mathbb{Q}\left(X \in \mathcal{N}_\gamma(x_0)\right) \in \mathcal{N}_{\omega}}} 
\E_{\mathbb{Q}}\left[\ell(Y, \alpha) \mid X \in \mathcal{N}_\gamma(x_0)\right],
\end{align}
where
$$
\mathbb{B}_{\delta_0}(\widehat{\mathbb{P}}) = \left\{
\mathbb{Q} \in \mathcal{M}(\mathcal{X} \times \mathcal{Y}) : 
\mathcal{W}(\mathbb{Q}, \widehat{\mathbb{P}}) \leq \delta_0
\right\},
$$
{with $\delta_0 > 0$, $\widehat{\mathbb{P}}=\frac{1}{N}\sum_{i=1}^N\mathrm{I}_{(\widehat{x}_i, \widehat{y}_i)}$, and  $\mathcal{N}_{\omega} =[\omega,1]$ for $\omega>0$ or $\mathcal{N}_{\omega} =(\omega,1]$ for $\omega=0$}. This formulation evaluates the worst-case conditional expectation over distributions $\mathbb{Q}$ lying within a Wasserstein ball centered at $\widehat{\mathbb{P}}$. The constraint $\mathbb{Q}(X \in \mathcal{N}_\gamma(x_0)) \in \mathcal{N}_{\omega}$ specifies an admissible range for the probability mass placed on the conditioning region.

\paragraph{{\bf (3) Partial Mass Transport.}} The approach of \cite{E22} likewise circumvents the need to estimate a conditional distribution, but instead places the ambiguity set directly on joint distributions $\mathbb{Q}$ supported within a covariate region $\mathcal{N} := \mathcal{N}_\gamma(x_0)$. 
They construct this ambiguity set using the partial optimal transport distance, which measures the proximity between a candidate distribution $\mathbb{Q}$ and a trimmed version of the empirical distribution $\widehat{\mathbb{P}}$:
$$
\mathcal{W}\left(\mathbb{Q}, \mathcal{R}_{1-\beta}(\widehat{\mathbb{P}}) \right)
\triangleq 
\min_{\widetilde{{\mathbb{P}}} \in \mathcal{R}_{1-\beta}(\widehat{\mathbb{P}})}
\mathcal{W}\left(\mathbb{Q}, \widetilde{{\mathbb{P}}} \right),
$$
where the trimming set is defined as
$$
\mathcal{R}_{1-\beta}(\widehat{\mathbb{P}}) \triangleq \left\{
\mathbb{Q} \ll \widehat{\mathbb{P}} : \frac{d\mathbb{Q}}{d\widehat{\mathbb{P}}} \leq \frac{1}{\beta}
\right\}, \quad\beta\in (0, 1].
$$
The resulting distributionally robust counterpart to problem~\eqref{pro1} is:
\begin{align}\label{m02}
\min_{\alpha \in \mathcal{A}} \;
\sup_{\substack{\mathbb{Q} \in \mathbb{B}_{\delta_0}(\widehat{\mathbb{P}};\beta) \\ \mathbb{Q}\left(X \in \mathcal{N}_\gamma(x_0)\right) = 1}} 
\E_{\mathbb{Q}}\left[\ell(Y, \alpha)\right],
\end{align}
where $\beta\in(0,1]$ is a given constant and $\mathbb{B}_{\delta_0}(\widehat{\mathbb{P}};\beta)$ is the ambiguity set defined via partial mass transport:
$$
\mathbb{B}_{\delta_0}(\widehat{\mathbb{P}};\beta) \triangleq \left\{
\mathbb{Q} \in \mathcal{M}(\mathcal{X} \times \mathcal{Y}) : 
\mathcal{W}\left(\mathbb{Q}, \mathcal{R}_{1-\beta}(\widehat{\mathbb{P}}) \right) \leq \delta_0
\right\}.
$$

Despite their individual merits, it remains unclear how to choose among these three approaches, how to compare them on equal footing, and to what extent they can accommodate broader classes of risk minimization problems. Their differing structural assumptions and fixed modeling choices also raise concerns about hidden biases. These limitations motivate the need for a more universal formulation—one that systematically integrates these modeling elements within a single optimal transport framework.

\subsection{A Universal Framework via Union-Ball Formulation}
To introduce our framework, we begin by revisiting the Predict-then-Robustify paradigm and its vulnerability to modeling bias--stemming both from reliance on point estimates of conditional distributions and from fixed choices of the radius.  A natural way to mitigate this bias is to explicitly account for the uncertainty in both the reference distribution and the radius, by treating them as variables ranging over an admissible set.
Concretely, we consider discrete distributions on $\mathcal{Y}$ supported on empirical outcomes, i.e., distributions of the form $\sum_{i=1}^N p_i\mathrm{I}_{\widehat{y}_i}$ for some $p_i \geq 0$, and define $\mathcal{V}$ as a {convex} set of all admissible pairs $(p, \delta)$ of distribution weights and radius values. The resulting ambiguity set takes the form of a union:
$$\mathop{\bigcup}\limits_{(p, \delta) \in \mathcal{V}} \mathbb{B}_{\delta}\left(\sum_{i=1}^N p_i\mathrm{I}_{\widehat{y}_i}\right),$$
capturing the joint uncertainty over both the center and the size of the OT ball.

This leads to the following distributionally robust formulation:
\begin{align}
\min_{\alpha \in \mathcal{A}} \sup_{\mathbb{Q}|_{Y} \in \mathop{\bigcup}\limits_{(p, \delta) \in \mathcal{V}} \mathbb{B}_{\delta}\left(\sum_{i=1}^N p_i\mathrm{I}_{\widehat{y}_i}\right)} \rho_{\mathbb{Q}|_{Y}}\left[\ell(Y, \alpha)\right], \label{m0}
\end{align}
which serves as a robust counterpart to the Predict-then-Robustify formulation~\eqref{kernel}, with  $(\widehat{\mathbb{P}}^{x_0}|_{Y}, \widehat{\delta}) \in \mathcal{V}$. 

As we will show, this natural--yet previously unexplored--DRO formulation proves remarkably powerful: it unifies and generalizes existing approaches, reveals hidden tractability, uncovers deep structural connections, and offers full flexibility to encode modeling preferences and prior knowledge. Throughout the paper, we adopt the following definition of the optimal transport metric over the joint space $\Xi=\mathcal{X} \times \mathcal{Y}$, as used in \cite{N24}.

\begin{definition} [Separable Transport Cost]
We endow the joint space $\mathcal{X} \times \mathcal{Y}$ with a separable cost function of the form
$$
\mathbb{D}\left((x, y), (x', y')\right) = \mathbb{D}_{\mathcal{X}}(x, x') + \mathbb{D}_{\mathcal{Y}}(y, y'),
$$
where $\mathbb{D}_{\mathcal{X}}$ and $\mathbb{D}_{\mathcal{Y}}$ are symmetric, continuous, and nonnegative ground costs on $\mathcal{X} \times \mathcal{X}$ and $\mathcal{Y} \times \mathcal{Y}$, respectively, and satisfy $\mathbb{D}((x, y), (x', y')) = 0$ if and only if $(x, y) = (x', y')$.
\end{definition}

To best illustrate the expressive power of our formulation in \eqref{m0}, we begin with the following generalization of the full optimal transport model in \cite{N24}:
\begin{align}\label{m1}
\min_{\alpha \in \mathcal{A}} \;
\sup_{\substack{\mathbb{Q} \in \mathbb{B}_{\delta_0}(\widehat{\mathbb{P}}) \\ \mathbb{Q}\left(X \in \mathcal{N}_\gamma(x_0)\right) \in \mathcal{N}_{\omega}}} 
\rho_{\mathbb{Q}}\left[\ell(Y, \alpha) \mid X \in \mathcal{N}_\gamma(x_0)\right],
\end{align}
{with $\delta_0 > 0$, $\gamma \geq 0$, and a convex set $\mathcal{N}_{\omega} \subset (0, 1]$, which means that $\mathcal{N}_{\omega}$ is an interval contained in $(0,1]$. }This formulation extends \cite{N24} in two significant ways. First, it accommodates arbitrary decision criteria $\rho$, thus enabling a fully general conditional risk minimization framework. In contrast, \cite{N24} focus primarily on expectations or risk measures admitting a minimization form $\rho(Z) = \inf_{t \in \mathbb{R}^k} \mathbb{E}[\ell(Z, t)]$. Second, it allows flexible specification of the admissible mass range $\mathcal{N}_\omega$, offering greater control over the conditioning event when needed.

{For the remainder of this section, we assume without loss of generality that the admissible mass interval is $\mathcal{N}_\omega=[\omega_1,\omega_2]$ with $0<\omega_1\le\omega_2\le1$.} We now show that the general conditional risk minimization problem in \eqref{m1} can be written in the form of the union-ball formulation \eqref{m0}. { We first state the necessary assumptions for the neighborhood set $\mathcal{N}_\gamma(x_0)$.
\begin{assumption}\label{assumption}
(Regularity conditions). The following assumptions hold.
\begin{itemize}
\item [(i)] Projection: For any $i =1,\ldots,N$, there exists a projection $\widehat{x}_i^p \in \partial \mathcal{N}_\gamma\left(x_0\right)$ of $\widehat{x}_i$ onto $\partial\mathcal{N}_\gamma\left(x_0\right)$ under the cost $\mathbb{D}_{\mathcal{X}}$, that is,
$$
0 \leq \min _{x \in \partial\mathcal{N}_\gamma\left(x_0\right)} \mathbb{D}_{\mathcal{X}}\left(x, \widehat{x}_i\right)=\mathbb{D}_{\mathcal{X}}\left(\widehat{x}_i^p, \widehat{x}_i\right):=d_i.
$$
\item [(ii)] Vicinity: For any $x \in \partial \mathcal{N}_\gamma\left(x_0\right)$ and for any radius $r>0$, the neighborhood set $$\left\{x^{\prime} \in \mathcal{X} \backslash \mathcal{N}_\gamma\left(x_0\right)\right.: \left.\mathbb{D}_{\mathcal{X}}\left(x^{\prime}, x\right) \leq r\right\}$$ around the boundary point $x$ is non-empty.
\end{itemize}
\end{assumption}

Without loss of generality, assume that $(\widehat{x}_i, \widehat{y}_i)\notin \{(x,y):x\in\mathcal{N}_\gamma\left(x_0\right)\}$ for $i=1,...,m$, and $(\widehat{x}_i, \widehat{y}_i)\in \{(x,y):x\in\mathcal{N}_\gamma\left(x_0\right)\}$ for $i=m+1,...,N$. {As a first step, we establish the necessary and sufficient condition for feasibility of problem \eqref{m1}, presented in the following proposition.  
 \begin{proposition}\label{feasible_con}
    Let \begin{align*}
        \delta_{min}=\min_{\substack{v_i\in[0,1],i=1,...,N, \\ 1/N\sum_{i=1}^Nv_i\in\mathcal{N}_\omega}}\frac{1}{N}\left(\sum_{i=1}^mv_id_i+\sum_{i=m+1}^N(1-v_i)d_i\right),
    \end{align*}
    with $\mathcal{N}_\omega=[\omega_1,\omega_2]\subset(0,1]$. For the radius $\delta_0$ given in problem \eqref{m1}, we have
    \begin{itemize}
        \item[(i)] For $\frac{N-m}{N}>\omega_2$, problem \eqref{m1} is feasible if and only if $\delta_0>\delta_{min}$.
        \item[(ii)] For $\frac{N-m}{N}\le\omega_2$, problem \eqref{m1} is feasible if and only if $\delta_0\ge\delta_{min}$. Specifically, if $\frac{N-m}{N}\in [\omega_1,\omega_2]$, $\delta_{min}=0$.
    \end{itemize}
\end{proposition}
\begin{remark}
    If $\omega_2=1$, then $\delta_{min}$ coincides with the minimum radius given in Proposition 2.5 of \cite{N24}. For the more general case of $\mathcal{N}_\omega$, however, the radius in Proposition 2.5 of \cite{N24} no longer applies, whereas $\delta_{min}$ yields the correct value.
\end{remark}}

Before presenting the main result, we first formalize the admissible conditional distributions of $Y$, induced by joint laws $\mathbb{Q}\in\mathbb{B}_{\delta_0}(\widehat{\mathbb{P}})$ satisfying $\mathbb{Q}(X \in \mathcal{N}_\gamma(x_0)) \in \mathcal{N}_\omega$--as this structure will play a key role in what follows. Specifically,
\begin{align*}
\mathcal{K}({\mathcal{N}_\gamma\left(x_0\right),\mathcal{N}_{\omega},\mathbb{B}_{\delta_0}}(\widehat{\mathbb{P}}))=\Bigg\{\mathbb{Q}|_Y:&\mathbb{Q}|_Y(A)=\frac{\mathbb{Q}\left(\mathcal{N}_\gamma\left(x_0\right)\times A\right)}{\mathbb{Q}\left(\mathcal{N}_\gamma\left(x_0\right)\times \mathcal{Y}\right)}\text{ for any measurable set }A \subset \mathcal{Y},\\
&\text{ where } \mathbb{Q}\in\mathbb{B}_{\delta_0}(\widehat{\mathbb{P}})\text{ satisfies that }\mathbb{Q}\left(X \in \mathcal{N}_\gamma\left(x_0\right)\right) \in\mathcal{N}_{\omega}
\Bigg\}.
\end{align*}}

We impose the following mild assumptions on the risk measure $\rho$ and the loss function $\ell$. As shown in Lemma \ref{OT-continuity} in the next section, virtually all major risk functionals studied in the literature satisfy them.
\begin{assumption}\label{assum2}
\begin{itemize}
    \item[(i)] $\rho$ is law-invariant, that is, if $Y_1\overset{d}{=}Y_2$, $\rho(Y_1)=\rho(Y_2)$, where $Y_1 \overset{d}{=} Y_2$ means that $Y_1$ and $Y_2$ have the same distribution. Assuming the distribution of $Y$ is $\mathbb{Q}$, then we can denote $\rho(Y)$ as $\rho(\mathbb{Q})$.
    \item[(ii)] Denote $\rho\left[\ell(Y, \alpha) \right]=\rho\circ\ell(Y,\alpha)=\rho\circ\ell(\mathbb{Q},\alpha)$, where $\mathbb{Q}$ is the distribution of $Y$. $\rho\circ\ell$ is OT-continuity, that is, for any distribution $\mathbb{Q}$, if there exists a series of distribution $\{\mathbb{Q}_n\}$ such that $\mathcal{W}(\mathbb{Q}_n, \mathbb{Q})\rightarrow
    0$ as $n\rightarrow\infty$, then $\rho\circ\ell(\mathbb{Q}_n,\alpha)\rightarrow\rho\circ\ell(\mathbb{Q},\alpha)$ for any $\alpha\in\mathcal{A}$.
\end{itemize}
\end{assumption}


\begin{theorem} \label{thm1} Suppose the neighborhood set $\mathcal{N}_\gamma(x_0)$ satisfies Assumption~\ref{assumption}, the radius $\delta_0$ satisfies the feasibility condition of problem~\eqref{m1} in Proposition~\ref{feasible_con}, and the loss function $\ell$ and risk measure $\rho$ satisfy Assumption~\ref{assum2}. Then problem~\eqref{m1}, equivalently
\begin{align}\label{innerproblem*}\min_{\alpha\in\mathcal{A}}\sup _{\mathbb{Q}|_{{Y}} \in\mathcal{K}({\mathcal{N}_\gamma\left(x_0\right),\mathcal{N}_{\omega},\mathbb{B}_{\delta_0}}(\widehat{\mathbb{P}})) } \rho_{\mathbb{Q}|_{{Y}}}\left[\ell(Y, \alpha) \right],\end{align}
admits the union-ball reformulation~\eqref{m0} with
	$${\cal V}:=\left\{(p,\delta):\exists \epsilon \text{~such that }(p,\delta,\epsilon)\in \mathcal{V}_0 \right\},$$
 where 
 \begin{align*}
 	{\cal V}_0:=\Bigg\{(p,\delta,\epsilon):&\frac{1}{\epsilon}\in \mathcal{N}_{\omega}, p_i\in\left[0,\frac{\epsilon}{N}\right], i=1,..., N, \sum_{i=1}^Np_i=1,\\
 	&\delta = \epsilon\left(\delta_0-\frac{1}{N}\sum_{i=m+1}^Nd_i\right)-\sum_{i=1}^{m}p_id_i+\sum_{i=m+1}^Np_id_i\Bigg\}.
 \end{align*}
\end{theorem}


We next show that the partial optimal transport model of \citet{E22}—which relies on the partial optimal transport distance rather than the standard Wasserstein formulation—can likewise be recast in the form of~\eqref{m0}. This is established through the following generalization of their model:
\begin{align}\label{m2}
\min_{\alpha \in \mathcal{A}} \;
\sup_{\substack{\mathbb{Q} \in \mathbb{B}_{\delta_0}(\widehat{\mathbb{P}};\mathcal{N}_{\omega}) \\ \mathbb{Q}\left(X \in \mathcal{N}_\gamma(x_0)\right) = 1}} 
\rho_{\mathbb{Q}}\left[\ell(Y, \alpha)\right],
\end{align}
where $\mathbb{B}_{\delta_0}(\widehat{\mathbb{P}};\mathcal{N}_{\omega})=\left\{
\mathbb{Q} \in \mathcal{M}(\mathcal{X} \times \mathcal{Y}) : \exists \beta \in \mathcal{N}_{\omega}, 
\mathcal{W}\left(\mathbb{Q}, \mathcal{R}_{1-\beta}(\widehat{\mathbb{P}}) \right) \leq \delta
\right\}$ with a convex set $\mathcal{N}_{\omega} \subset (0, 1]$. This formulation extends \citet{E22} in two ways: it accommodates arbitrary decision criteria $\rho$, and it allows flexible specification of the admissible mass range ${\mathcal N}_\omega$. {Based on the definition of trimming set, one can verify that the constraint $\exists \beta \in \mathcal{N}_{\omega}, 
\mathcal{W}\left(\mathbb{Q}, \mathcal{R}_{1-\beta}(\widehat{\mathbb{P}}) \right) \leq \delta_0$ is equivalent to $\mathcal{W}\left(\mathbb{Q}, \mathcal{R}_{1-\epsilon}(\widehat{\mathbb{P}}) \right) \leq \delta_0$, where $\epsilon:=\min_{\beta\in\mathcal{N}_{\omega}}\beta$.} As a first step, we establish the necessary and sufficient condition for feasibility of problem \eqref{m2}, presented in the following proposition.  
 \begin{proposition}\label{feasible_con_2}
    Let \begin{align*}
        \delta_{min}=\min_{\substack{v_i\in[0,1],i=1,...,N, \\ \sum_{i=1}^Nv_i=N\epsilon}}\frac{1}{N\epsilon}\sum_{i=1}^mv_id_i,
    \end{align*}
    with $\epsilon:=\min_{\beta\in\mathcal{N}_{\omega}}\beta$. For the radius $\delta_0$ given in problem \eqref{m2}, we have problem \eqref{m2} is feasible if and only if $\delta_0\ge\delta_{min}$.
\end{proposition}
\begin{remark}
    In Definition 2 of \citet{E22}, they also give the feasible condition to problem \eqref{m02}. That is, problem \eqref{m02} is feasible if and only if $$\delta_0\ge\min_{\substack{v_i\in[0,1],i=1,...,N, \\\sum_{i=1}^Nv_i=N\beta}}\frac{1}{N\beta}\sum_{i=1}^mv_id_i,$$
    which is consistent with the result of Proposition \ref{feasible_con_2} when $\mathcal{N}_\omega=\{\beta\}$ with  $\beta\in(0,1]$. 
\end{remark}
Next, we again formalize the admissible conditional distributions of $Y$, induced by feasible joint distributions  
$\mathbb{Q}$ satisfying $\mathbb{Q} \in \mathbb{B}_{\delta_0}(\widehat{\mathbb{P}};\mathcal{N}_{\omega})$ and $\mathbb{Q}\!\left(X \in \mathcal{N}_\gamma(x_0)\right)=1$. Specifically, 
\begin{align*}
		\mathcal{K}(\mathcal{N}_\gamma(x_0),\mathbb{B}_{\delta_0}(\widehat{\mathbb{P}};\mathcal{N}_{\omega}))=\Bigg\{\mathbb{Q}|_Y:&\mathbb{Q}|_Y\text{ is the marginal distribution of }\mathbb{Q}\text{ on }\mathcal{Y},\\
		&\text{ where } \mathbb{Q}\in\mathbb{B}_{\delta_0}(\widehat{\mathbb{P}};\mathcal{N}_{\omega})\text{ satisfies that }\mathbb{Q}\left(X \in \mathcal{N}_\gamma(x_0)\right) =1
		\Bigg\}.
\end{align*} 


\begin{theorem} \label{thm2} The problem  \eqref{m2}, equivalently 
\begin{align}\label{conditional-version}\min_{\alpha\in\mathcal{A}}\sup _{\mathbb{Q}|_{{Y}} \in\mathcal{K}({\mathcal{N}_\gamma(x_0),\mathbb{B}_{\delta_0}(\widehat{\mathbb{P}};\mathcal{N}_{\omega})}) } \rho_{\mathbb{Q}|_{{Y}}}\left[\ell(Y, \alpha) \right],\end{align}
admits the union-ball reformulation~\eqref{m0} with
 $${\cal V}:=\left\{(p,\delta):{\epsilon}=\min_{\beta\in\mathcal{N}_{\omega}}\beta, p_i\in[0,\frac{1}{N\epsilon}], i=1,..., N, \sum_{i=1}^Np_i=1,\delta= \delta_0-\sum_{i=1}^mp_id_i\right\}.$$
\end{theorem}

Juxtaposing the admissible sets in Theorems~\ref{thm1} and \ref{thm2} reveals a shared structural backbone between models \eqref{m1} and \eqref{m2}--despite their distinct premises--while also pinpointing the subtle yet consequential ways they diverge. The key difference lies in how they handle transport costs across the boundary of ${\cal N}_\gamma(x_0)$. Model \eqref{m1} permits bidirectional transport, accounting for the movement of both in-region and out-of-region samples--i.e., the points $(\widehat{x}_i, \widehat{y}_i)$ for $i = 1, \dots, m$ and $i = m+1, \dots, N$--to determine the radius:
$$\delta = \epsilon\left(\delta_0-\frac{1}{N}\sum_{i=m+1}^Nd_i\right)-\sum_{i=1}^{m}p_id_i+\sum_{i=m+1}^Np_id_i.$$
By contrast, model \eqref{m2} restricts the budget to inward transport only--considering just $(\widehat{x}_i, \widehat{y}_i)$ for $i = 1, \dots, m$--and ignores any offsetting movement out of the region. This structural asymmetry makes the partial model more conservative.  Notably, the two models become exactly equivalent in two common cases: (1) when the covariate neighborhood ${\cal N}_\gamma(x_0)$ reduces to a singleton\footnote{{According to the proof of Theorem \ref{thm1}, this equivalence holds if the risk measure $\rho$ and the loss function $\ell$ satisfy Assumption \ref{assum2}}.}, and (2) when the empirical distribution $\widehat{\mathbb{P}}$ places no mass within ${\cal N}_\gamma(x_0)$. In either case, the transport cost across the boundary vanishes, and the models coincide. This equivalence, previously unnoticed, is practically important: it shows that either model can be applied interchangeably in these cases, which, to the best of our knowledge, had not been recognized before.

The union-ball characterization also enables direct and explicit comparison between the joint-distribution models \eqref{m1} and \eqref{m2} and the Predict-then-Robustify model \eqref{kernel}. In particular, the joint formulations can be viewed as robust counterparts of \eqref{kernel} whenever the estimated pair $(\widehat{\mathbb{P}}^{x_0}\!\mid_Y,\widehat{\delta})$ from \eqref{kernel} lies within the admissible set $\mathcal{V}$ characterized in Theorems~\ref{thm1} and \ref{thm2}. Verifying this condition reduces to checking the linear (in)equalities that define $\mathcal{V}$--a straightforward task. Moreover, as long as the conditional distribution $\widehat{\mathbb{P}}^{x_0}\!\mid_Y$ satisfies the upper-bound constraints encoded in $\mathcal{V}$, one can always calibrate the joint model's transport budget $\delta_0$ to ensure it is at least as robust as \eqref{kernel}. For instance, under the setup of Theorem~\ref{thm2}, setting
{$$
\delta_0 \;=\; \widehat{\delta}\;+\;\sum_{i=1}^{m}\widehat{p}_i\,d_i
$$}guarantees that $(\widehat{\mathbb{P}}^{x_0}\!\mid_Y,\widehat{\delta}) \in \mathcal{V}$, and thus ensures that the joint model dominates in terms of conservativeness. These comparisons, opaque in the original formulations, become transparent, verifiable, and practically calibratable through the lens of the union-ball characterization.

\subsection*{A New Robust Counterpart to Predict-then-Robustify}
Having shown that models \eqref{m1} and \eqref{m2} can be interpreted as robust counterparts to the Predict-then-Robustify scheme \eqref{kernel}, we conclude by illustrating how the framework can also be used to design a new formulation via direct specification of the admissible set ${\cal V}$. Given an estimated conditional probability distribution $\widehat{\mathbb{P}}_{Y|X}$ with a fixed covariate value $X=x_0$, it is natural to account for the uncertainty in using it as a reference by first constructing an uncertainty set around it:
$$
{\cal C} := \left\{p\in[0,1]^N : \sum_{i=1}^N p_i=1,~d(p,\widehat{\mathbb{P}}^{x_0}\!\mid_Y) \leq \gamma \right\},
$$
where $d$ denotes a “distance” metric between finite-dimensional probability vectors, such as a $\phi$-divergence, and $\gamma \geq 0$ is the radius. To avoid the over-conservatism of a fixed-radius union ball—and inspired by Theorems~\ref{thm1} and \ref{thm2}—we allow the radius $\delta$ to vary with the choice of reference distribution $p$: the farther $p$ is from the estimate $\widehat{\mathbb{P}}_{Y|X}$, the smaller the radius. This trade-off can be modeled through the admissible set
$$
{\cal V} := \left\{(p, \delta): p \in {\cal C},\ 0 \leq \delta \leq \bar{h}\left(d(p,\widehat{\mathbb{P}}^{x_0}\!\mid_Y)\right)\right\},
$$
where $\bar{h}:\mathbb{R}^+ \rightarrow \mathbb{R}^+$ is assumed to be concave and decreasing, with $\bar{h}(0) = \widehat{\delta}$ and $\bar{h}(\gamma) = 0$. This construction ensures that the union ball includes, as a member, the ball underlying problem~\eqref{kernel}: when $d(p,\widehat{\mathbb{P}}^{x_0}\!\mid_Y) = 0$, the bound $\delta \leq \widehat{\delta}$ holds, consistent with the Predict-then-Robustify scheme.

\section{Unified Tractability for Distributionally Robust CRM}\label{tractability}
A key strength of the union-ball characterization lies not only in its versatility, but in its ability to systematically yield exact and tractable reformulations across a wide spectrum of distributionally robust CRM problems. As demonstrated throughout this section, it accommodates a broad range of risk functionals and modeling choices. Notably, for cases previously treated by \citet{N24} via optimal transport--where only conservative approximations were available--our framework provides exact and significantly simpler reformulations.

Let $\mathcal{K}$ denote the set of conditional distributions ${\mathbb{Q}}|_{Y}$ arising in a distributionally robust CRM problem (e.g. $\mathcal{K}({\mathcal{N}_\gamma\left(x_0\right),\mathcal{N}_{\omega},\mathbb{B}_{\delta_0}}(\widehat{\mathbb{P}}))$ in Theorem \ref{thm1} and $\mathcal{K}({\mathcal{N}_\gamma(x_0),\mathbb{B}_{\delta_0}(\widehat{\mathbb{P}};\mathcal{N}_{\omega})})$ in Theorem \ref{thm2}). The union-ball characterization reveals that the inner worst-case conditional risk reduces to a maximization over standard optimal transport-based problems. This yields the following reformulation of the distributionally robust CRM problem:
\begin{align}
 & \min_{\alpha \in \mathcal{A}} \sup_{{\mathbb{Q}}|_{Y} \in \mathcal{K}}
	\rho_{{\mathbb{Q}}|_{Y}}\left[\ell(Y, \alpha) \right] \label{bb}\\
= & \min_{\alpha \in \mathcal{A}} \sup_{{\mathbb{Q}}|_{Y} \in \bigcup\limits_{(p,\delta) \in \mathcal{V}} 
		\mathbb{B}_{\delta}\left(\sum_{i=1}^N p_i\mathrm{I}_{\widehat{y}_i}\right)} 
	\rho_{{\mathbb{Q}}|_{Y}}\left[\ell(Y, \alpha) \right]\label{union-version} \\
= & 	\min _{\alpha \in \mathcal{A}} \sup_{(p,\delta) \in {\cal V}} g(\alpha,(p,\delta)), \label{r11}
\end{align}
where 
	\begin{equation} \label{r2}
		g(\alpha,(p,\delta)):=
		\left\{\sup _{{\mathbb{Q}}|_{Y} \in \mathbb{B}_{\delta}(\sum_{i=1}^Np_i\mathrm{I}_{\widehat{y}_i})} \rho_{{\mathbb{Q}}|_{Y}}\left[\ell(Y, \alpha) \right]\right\}.
	\end{equation}
This decomposition offers two key advantages for tackling problem \eqref{bb}:
\begin{enumerate}
	\item For each fixed $(p,\delta)$, the inner subproblem \eqref{r2} becomes a well-structured (unconditional) Optimal Transport DRO centered at a finitely supported distribution. This facilitates tractable reformulations by leveraging structural properties inherited from Wasserstein DRO formulations-namely, with $\mathbb{D}_{\mathcal{Y}}(y_1,y_2)=||y_1-y_2||^q_{\mathcal{Y}}$ for any $y_1,y_2\in\mathcal{Y}$ and $q\in[1,\infty)$.
	\item  The outer maximization $\sup_{(p,\delta)\in\mathcal{V}} g(\alpha,(p,\delta))$ broadly resembles a robust optimization problem with uncertain probability weights $p$, but with additional complexity arising from the uncertain radius $\delta$. As will be shown later in this section, it nevertheless admits tractable reformulations.
\end{enumerate}

In particular, to highlight the tractability of \eqref{bb} through the min–max reformulation \eqref{r11}, we first state the following structural observation.
\begin{proposition} \label{gen}
Suppose $\rho_{\mathbb{Q}|_Y}$ is a law-invariant risk functional that is concave in the distribution, and $\rho_{\mathbb{Q}|_Y} \circ \ell$ is convex in $\alpha$. Then $g(\alpha, (p, \delta))$ is jointly concave in $(p, \delta)$ and convex in $\alpha$.
\end{proposition}

Observe that $\min _{\alpha \in \mathcal{A}} \sup_{(p,\delta) \in {\cal V}} g(\alpha,(p,\delta))=\min _{\alpha \in \mathcal{A}} \sup_{(p,\delta) \in {\cal V},\delta\ge 0} g(\alpha,(p,\delta))$, since $g(\alpha,(p,\delta))=-\infty$ whenever $\mathbb{B}_{\delta}\!\left(\sum_{i=1}^N p_i \mathrm{I}_{\widehat{y}_i}\right)$ is empty, i.e., when $\delta<0$. For brevity, define
$$\mathcal{V}_+=\left\{(p,\delta):(p,\delta)\in\mathcal{V}, \delta\ge 0\right\},$$ which will be used hereafter. In the remainder of this section, we focus on affine decision rules, where $\ell(Y,\alpha)=\ell(Y^\top\alpha)$—an assumption already required in unconditional Wasserstein DRO settings to ensure tractability. Leveraging the decomposition insight and structural properties identified above, we show how a wide class of distributionally robust CRM problems can be reformulated as tractable convex programs. The following definitions of convex and concave conjugate functions will be instrumental in this development.
	\begin{definition}
    Assume a convex set $\mathcal{D}\subset\mathbb{R}^n$.
		\begin{itemize}
			\item[(i)] The concave conjugate $f_*(\cdot)$ of a function $f: \mathcal{D} \rightarrow \mathbb{R}$ is defined as a function $f_*: \mathbb{R}^n \rightarrow \mathbb{R} \cup\{-\infty\}$ :
			$$
			f_*(v)=\inf _{p \in \mathcal{D}}\left\{v^\top p-f(p)\right\}.
			$$
			\item[(ii)] The convex conjugate $g^*(\cdot)$ of a function $g: \mathcal{D} \rightarrow \mathbb{R}$ is defined as a function $g^*: \mathbb{R}^n \rightarrow \mathbb{R} \cup\{+\infty\}:$
			$$
			g^*(v)=\sup _{p \in \mathcal{D}}\left\{v^\top p-g(p)\right\}.
			$$
		\end{itemize}
	\end{definition}
	\begin{remark}
		Let $f(\cdot)$ and $g(\cdot)$ be $-\infty$ and $\infty$ outside $\mathcal{D}$. Then the aforementioned definition fits the standard definition of convex (concave) conjugate.
	\end{remark}
    
To facilitate navigation of the tractable reformulations developed in this section, Table \ref{summarize} provides a concise summary of the results.
\begin{table}[ht]
\centering
\caption{Summary of tractable results by risk functional, norm choice, and loss function $\ell$}
\label{summarize}
\begin{tabular}{c|c|c|c}
\hline
\multirow{2}{*}{\textbf{Risk functional}} 
  & \multirow{2}{*}{\makecell{\textbf{Order $1$ norm} \\ ($\|\cdot\|_\mathcal{Y}$) }} 
  & \multicolumn{2}{c}{\textbf{Order $q$ norm} ($\|\cdot\|^q_\mathcal{Y}$)} \\ \cline{3-4}
  &   & \textbf{Specific loss function} & \textbf{General loss function} \\ \hline
{Expectation} ($\E\left[\ell(Z)\right]$) & Proposition \ref{regular_exp} (i) & Proposition \ref{regular_exp} (ii) & Proposition \ref{general_exp}  \\ \hline {$\inf _{t \in \mathbb{R}^k} \mathbb{E}\left[\ell(Z, t)\right]$}
& Proposition \ref{general_min_exp} (i) & Proposition \ref{general_min_exp} (ii) & Proposition \ref{pro2*} \\ \hline
$\inf _{t \in \mathbb{R}}\left\{t+\left(\mathbb{E}\left[\ell^q(Z, t)\right]\right)^{1 /q}\right\}$ & Proposition \ref{general_min_exp_1} (i) & Proposition \ref{general_min_exp_1} (i) & Proposition \ref{pro_rho2_general} \\ \hline
Shortfall risk measure & Proposition \ref{pro:shortfall_1} & Remark \ref{q_order_utility_shortfall} & Proposition \ref{pro:shortfall_2} \\ \hline
Distortion functional & Proposition \ref{pro4*} (i) & Proposition \ref{pro4*} (ii) & --- \\ \hline
\end{tabular}
\end{table}
Moreover, the risk functionals presented in Table \ref{summarize} satisfy Assumption \ref{assum2}, that is, the risk functionals is OT-continuity under that the cost function $\mathbb{D}_{\mathcal{Y}}(y_1,y_2)=||y_1-y_2||^q_{\mathcal{Y}}$ for any $y_1,y_2\in\mathcal{Y}$ with $q\in[1,\infty)$. We present the result in the following lemma.
\begin{lemma}\label{OT-continuity}
    The risk functionals presented in Table \ref{summarize} is OT-continuity under that the cost function $\mathbb{D}_{\mathcal{Y}}(y_1,y_2)=||y_1-y_2||^q_{\mathcal{Y}}$ for any $y_1,y_2\in\mathcal{Y}$ with $q\in[1,\infty)$.
\end{lemma}

\subsection{Conditional Expectation}\label{CE}
We begin with the base case of the conditional expectation problem under its union-ball formulation:
\begin{align}\label{pro1:exp1-1}
		\min_{\alpha \in \mathcal{A}} \sup_{{\mathbb{Q}}|_{Y} \in \bigcup\limits_{(p,\delta) \in \mathcal{V}} 
			\mathbb{B}_{\delta}\left(\sum_{i=1}^N p_i\mathrm{I}_{\widehat{y}_i}\right)} 
		\E_{{\mathbb{Q}}|_{Y}}\left[\ell(Y^\top\alpha) \right] = \min_{\alpha \in \mathcal{A}} \sup_{(p,\delta) \in \mathcal{V}} \left\{ \sup_{{\mathbb{Q}}|_{Y} \in \mathbb{B}_{\delta}\left(\sum_{i=1}^N p_i\mathrm{I}_{\widehat{y}_i}\right)} \E_{{\mathbb{Q}}|_{Y}}\left[\ell(Y^\top\alpha) \right] \right\}.
\end{align}

Invoking the regularization reformulation of the inner worst-case subproblem (\citet{W22}, see Section~EC.2), we obtain a family of robust optimization problems. As we show below, these problems admit convex reformulations.
\noindent
\begin{proposition}\label{regular_exp}
\begin{itemize}
	\item [(i)]Let $\mathbb{D}_{\mathcal{Y}}(y_1,y_2)=||y_1-y_2||_{\mathcal{Y}}$ for any $y_1,y_2\in\mathcal{Y}$ and $\ell: \mathbb{R} \rightarrow \mathbb{R}$ be a convex and Lipschitz continuous function. 
	The problem \eqref{pro1:exp1-1} is equivalent to the robust optimization problem
$$\min_{\alpha \in\mathcal{A}} \sup_{(p,\delta) \in \mathcal{V}_+ } \sum_{i=1}^N p_i \ell(\widehat{y}_i^\top\alpha)  + \operatorname{Lip}(\ell)\delta ||\alpha||_*,$$	
and can be solved by the convex program
	\begin{align*}
	 &\inf_{\alpha \in\mathcal{A},v}~\sigma^*(v\mid\mathcal{V}_+)\\
		&~~~\text{s.t.}~ v_i\ge \ell(\widehat{y}_i^\top\alpha), ~i=1,...,N, \\
		&~~~~~~~~~v_{N+1}\ge \operatorname{Lip}(\ell)||\alpha||_*.
	\end{align*}
	\item[(ii)]Let $\mathbb{D}_{\mathcal{Y}}(y_1,y_2)=||y_1-y_2||_{\mathcal{Y}}^q$ for any $y_1,y_2\in\mathcal{Y}$ and $q\in(1,\infty)$. 
\begin{itemize}
	\item [(1)]Assume the loss function $\ell$ takes one of the following two forms, multiplied by $C>0$ :
	\begin{itemize}
		\item [(a)]$\ell_1(x)=x+b$ or $\ell_1(x)=-x+b$ with some $b \in \mathbb{R}$;
		\item[(b)]  $\ell_2(x)=\left|x-b_1\right|+b_2$ with some $b_1, b_2 \in \mathbb{R}$.
	\end{itemize}The problem \eqref{pro1:exp1-1} is equivalent to the robust optimization problem
$$\min_{\alpha \in\mathcal{A}} \sup_{(p,\delta) \in \mathcal{V}_+ } \sum_{i=1}^N p_i \ell(\widehat{y}_i^\top\alpha) +C\delta^\frac{1}{q} ||\alpha||_*,$$
and can be solved by the convex program
	\begin{align*}
		&\inf_{\alpha  \in\mathcal{A},v}~\sigma^*(v|\mathcal{V}_+)+C_qv_{N+1}\left\|\frac{\alpha}{v_{N+1}}\right\|_*^{\frac{q}{q-1}}\\
		&~~~\text{s.t.}~v_i\ge \ell(\widehat{y}_i^\top\alpha), ~i=1,...,N,\\
		&~~~~~~~~~v_{N+1}\ge 0,
	\end{align*}
	where $C_q=C^{\frac{q}{q-1}}\left(q^{\frac{1}{1-q}}-q^{\frac{q}{1-q}}\right)$.
	\item[(2)] Assume the loss function $\ell$ is of the form $\ell(x) = \left(C \cdot \ell_i(x)\right)^q$, where $C > 0$, $q \in (1, \infty)$, and $\ell_i$ is one of the following:	
	\begin{itemize}
		\item [(a)]$\ell_1(x)=(x-b)_{+}$with some $b \in \mathbb{R}$;
		\item[(b)] $\ell_2(x)=(x-b)_{-}$with some $b \in \mathbb{R}$;
		\item[(c)] $\ell_3(x)=\left(\left|x-b_1\right|-b_2\right)_{+}$with some $b_1 \in \mathbb{R}$ and $b_2 \geqslant 0$;
		\item[(d)] $\ell_4(x)=\left|x-b_1\right|+b_2$ with some $b_1 \in \mathbb{R}$ and $b_2>0$.
	\end{itemize} 
	The problem \eqref{pro1:exp1-1} is equivalent to the robust optimization problem
	
$$ \min_{\alpha \in\mathcal{A}} \sup_{(p,\delta) \in \mathcal{V}_+ } \left[\left(\sum_{i=1}^N p_i \ell(\widehat{y}_i^\top\alpha) \right)^{1/q} + C\delta^\frac{1}{q} ||\alpha||_*\right]^{q},$$	
and can be solved by the convex program	
	\begin{align*}
		&\inf_{\alpha\in\mathcal{A},v,y\geq 0}\left[\sigma^*(v|\mathcal{V}_+)+C_q^1y+C_qv_{N+1}\left\|\frac{\alpha}{v_{N+1}}\right\|_*^{\frac{q}{q-1}}\right]^q\\
		&~~~~~~\text{s.t.}~v_i\ge \frac{\ell(\widehat{y}_i^\top\alpha)}{y^{q-1}}, ~i=1,...,N,\\
		&~~~~~~~~~~~~v_{N+1}\ge 0.
	\end{align*}
	where $C_q^1=q^{\frac{1}{1-q}}-q^{\frac{q}{1-q}}$ and $C_q=C^{\frac{q}{q-1}}\left(q^{\frac{1}{1-q}}-q^{\frac{q}{1-q}}\right)$.  
\end{itemize}
\end{itemize}
\end{proposition}

More generally, even without regularization equivalents, the decomposition remains powerful: by isolating each Wasserstein DRO subproblem, it enables us to exploit their projection property \citep{W22}:
\begin{align*}
\sup_{{\mathbb{Q}}|_{Y} \in \mathbb{B}_{\delta}\left(\sum_{i=1}^N p_i \mathrm{I}_{\widehat{y}_i}\right)} 
\E_{{\mathbb{Q}}|_{Y}} \left[ \ell(Y^\top \alpha) \right] 
= \sup_{{\mathbb{Q}}|_{Z} \in \mathbb{B}_{\delta \|\alpha\|_*} \left( \sum_{i=1}^N p_i \mathrm{I}_{\widehat{y}_i^\top \alpha} \right)} 
\E_{{\mathbb{Q}}|_{Z}} \left[ \ell(Z) \right],
\end{align*}
which yields a one-dimensional reformulation that is considerably easier to handle. Leveraging this insight, we derive a final convex program reformulation of the conditional expectation problem \eqref{pro1:exp1-1}, applicable to any suitably regular convex loss function~$\ell$.

\begin{proposition}\label{general_exp}
Let $\mathbb{D}_{\mathcal{Y}}(y_1,y_2)=||y_1-y_2||_{\mathcal{Y}}^q$ for any $y_1,y_2\in\mathcal{Y}$ and $q\in(1,\infty)$. {Assuming the loss function $\ell$ is convex and satisfies $\ell(x)-\ell(x_0) \leqslant L|x-x_0|^q+M, x \in \mathbb{R}$ for some $L,M>0$, and some $x_0\in\mathbb{R}$, the problem} \eqref{pro1:exp1-1} can be solved by the convex program
	\begin{align}
		&\inf_{\alpha\in\mathcal{A},v,\eta \geq 0}~\sigma^*(v|\mathcal{V}_+)\notag\\
		&~~~~~~\text{s.t.}~v_i\ge \sup _{z \in \mathbb{R}}\left\{z \cdot \widehat{y}_i^{\top} \alpha-\ell^*(z)+\eta \cdot \frac{|z|^{q^*}}{q^* q^{q^*-1}}\right\}, ~i=1,...,N, \label{ex1}\\
		&~~~~~~~~~~~~v_{N+1}\ge ||\alpha||_*^q/\eta^{q-1}.\notag
	\end{align}
\end{proposition}

Constraints like \eqref{ex1}, where the right-hand side involves an optimization problem, commonly arise in dual reformulations of Wasserstein DRO. What sets \eqref{ex1} apart is that the optimization is univariate, which not only simplifies analysis and computation but also enables further problem-specific reductions, as we demonstrate below.

\begin{example}
    \begin{itemize}
        \item[(i)](Piecewise linear convex loss function) Let $\ell(x)=\sup_{j=1,...,k}a_jx+b_j$. The convex conjugate function is $$
\ell^*(z)= \begin{cases}-\max _{j: a_j=z} b_j & \text { if } z \in\left\{a_1, \ldots, a_k\right\}, \\ +\infty & \text { otherwise }.\end{cases}
$$ Without loss of generality, assume that $a_{j_1}\neq a_{j_2}$ for ${j_1}\neq {j_2}$. Thus, for $q>1$, we have 
\begin{align*}
    \sup _{z \in \mathbb{R}}\left\{zt-\ell^*(z)+\eta \cdot \frac{|z|^{q^*}}{q^* q^{q^*-1}}\right\}=\max _{j=1, \ldots, k}\left\{a_jt+b_j+\eta \cdot \frac{|a_j|^{q^*}}{q^* q^{q^*-1}}\right\} .
\end{align*}
        \item[(ii)](Piecewise quadratic convex loss function) Let $\ell(x)=\sup_{j=1,...,k}a_jx^2+b_jx+c_j$ and $a_{j}>0$ for $j=1,...,k$. The convex conjugate function is 
        $$
\ell^*(z)=\min _{j=1, \ldots, k}\left\{\frac{\left(z-b_j\right)^2}{4 a_j}-c_j\right\} .
$$
Thus, for $q=2$, we have 
\begin{align*}
    \sup _{z \in \mathbb{R}}\left\{zt-\ell^*(z)+\eta \cdot \frac{|z|^{q^*}}{q^* q^{q^*-1}}\right\}=\max _{j=1, \ldots, k}\left\{\frac{a_j t^2+b_j t}{1-a_j \eta}+c_j+\frac{a_j b_j^2\eta}{4\left(1-a_j \eta\right)}\right\},
\end{align*}
and $\eta<\min_{j=1,...,k}\frac{1}{a_j}$.
\item[(iii)](Power loss function) Let $\ell(x)=\frac{|x|^q}{q}$ for $q>1$. The convex conjugate function is
    $$
    \ell^*(z)=\frac{|z|^{q^*}}{q^*}.
    $$
    Thus, we have 
    \begin{align*}
    \sup _{z \in \mathbb{R}}\left\{zt-\ell^*(z)+\eta \cdot \frac{|z|^{q^*}}{q^* q^{q^*-1}}\right\}=\frac{|t|^q}{q}\left(1-\frac{\eta}{q^{q^*-1}}\right)^{-(q-1)},
\end{align*}
and $\eta<q^{q^*-1}$.
    \end{itemize}
\end{example}

Thus far, our tractable reformulations have been kept general, applicable to any distributionally robust CRM that admits a union-ball representation with a convex admissible set~$\mathcal{V}$. To illustrate how the support function~$\sigma^*(v|\mathcal{V}_+)$ is derived, we now apply the admissible sets from Theorems~\ref{thm1} and~\ref{thm2}. These examples also enable direct comparison between the final formulations from our union-ball approach and those in existing OT-based methods. Recall that $\mathcal{N}_{\omega}=[\omega_1,\omega_2]\subset(0,1]$.

\begin{example}\label{full-conjugate}
Denote $p'=(p,\delta,\epsilon)$.  Following Theorem~\ref{thm1}, the conjugate of the indicator function of $\mathcal{V}_+$ is given by 
    \begin{align*}
\sigma^*(v\mid \mathcal{V}_+)
= \sup_{(p,\delta)\in\mathcal{V}_+} v^\top (p,\delta)
= \begin{cases}\sup v^\top (p,\delta)+0\cdot\epsilon\\
\text{s.t. } Ap' \le b,\; p' \ge 0
\end{cases} = \begin{cases}\inf z^\top b\\
\text{s.t. } A^\top z \ge (v,0),\; z \ge 0,
\end{cases}
\end{align*}
where $A\in\mathbb{R}^{(N+6)\times(N+2)}$,and  $b\in\mathbb{R}^{N+6}$ are defined by 
 \begin{align*}
         A = \begin{pmatrix}
		1& \cdots & 1 & 0 &0 \\
		-1& \cdots & -1 & 0&0\\
		0& \cdots & 0 & 0&1\\
		0& \cdots & 0 & 0&-1\\
		1 & \cdots & 0 & 0&-\frac{1}{N}\\
		\vdots & \ddots & \vdots&\vdots&\vdots \\
		0 & \cdots & 1&0&-\frac{1}{N}\\
		-a_1 & \cdots & -a_{N}&-a_{N+1}&-a_{N+2}\\
		a_1 & \cdots & a_{N}&a_{N+1}& a_{N+2}
	\end{pmatrix},\quad\text{ with }
  a = \begin{pmatrix}-d_1 \\ \vdots \\ -d_m \\ d_{m+1} \\ \vdots \\ d_N \\ -1 \\[.5ex]\displaystyle \delta_0 - \frac1N\sum_{i=m+1}^N d_i\end{pmatrix},\quad\text{ and }
  b = 
    \begin{pmatrix}
      1 \\ -1 \\[3pt]
      \tfrac1{\omega_1} \\ -\tfrac1{\omega_2} \\[3pt]
      0 \\ \vdots \\ 0
    \end{pmatrix}.
    \end{align*}
\end{example}

\begin{example}\label{partial-conjugate}
Note that $\mathcal{N}_{\omega}=[\omega_1,\omega_2]\subset(0,1]$. Then ${\epsilon}=\min_{\beta\in\mathcal{N}_{\omega}}\beta=\omega_1$. Following Theorem~\ref{thm2}, the conjugate of the indicator function of $\mathcal{V}_+$ is given by 
 \begin{align*}
		\sigma^*(v\mid \mathcal{V}_+)=\sup_{p'\in\mathcal{V}_+}v^\top p'
		=\begin{cases}\sup v^\top p'\\
		\text{s.t.}~A'p'\le b',~p'\ge 0\end{cases}
        =\begin{cases}\inf z^\top b'\\
		\text{s.t.}~(A')^\top z\ge v,~z\ge 0,\end{cases}
	\end{align*}
where $A'\in\mathbb{R}^{(N+4)\times(N+1)}$, and  $b'\in\mathbb{R}^{N+4}$ are defined by 
	 \begin{align*}
         A'=\begin{pmatrix}
		1& \cdots & 1 & 0  \\
		-1& \cdots & -1 & 0\\
		1 & \cdots & 0 & 0\\
		\vdots & \ddots & \vdots&\vdots \\
		0 & \cdots & 1&0\\
		-a'_1 & \cdots & -a'_{N}&-a'_{N+1}\\
		a'_1 & \cdots & a'_{N}&a'_{N+1}
	\end{pmatrix},\quad\text{ with }
  a' = \begin{pmatrix}d_1 \\ \vdots \\ d_m \\ 0 \\ \vdots \\ 0 \\ 1 \end{pmatrix},\quad\text{ and }
  b' = 
    \begin{pmatrix}
      1 \\ -1 \\[3pt]
      \frac{1}{N\omega_1}\\ \vdots \\ \frac{1}{N\omega_1} \\
      -\delta_0 \\ \delta_0
    \end{pmatrix}.
    \end{align*}
\end{example}
{\begin{remark}
Building on the conjugates in Examples \ref{full-conjugate} and \ref{partial-conjugate}, we provide full reformulations of the problems in Appendix EC.3. There, we also discuss how our union-ball reformulation substantially simplifies the final form, in comparison with the reformulation given in \citet{N24}, as highlighted in Remark \ref{comparison_full}.
\end{remark}}

\subsection{Conditional Risk Functionals in Expectation Form}\label{Expectation Form}
The decomposition framework naturally extends to a broader class of risk functionals that admit expectation-based representations. {In this section, we focus on solving the following problem 
\begin{align}\label{general_rho_union_version}
    \min_{\alpha \in \mathcal{A}} \sup_{{\mathbb{Q}}|_{Y} \in \bigcup\limits_{(p,\delta) \in \mathcal{V}} 
		\mathbb{B}_{\delta}\left(\sum_{i=1}^N p_i\mathrm{I}_{\widehat{y}_i}\right)} 
	\rho_{{\mathbb{Q}}|_{Y}}\left(Y^\top \alpha \right)=\min_{\alpha \in \mathcal{A}} \sup_{(p,\delta) \in \mathcal{V}} \left\{ \sup_{{\mathbb{Q}}|_{Y} \in \mathbb{B}_{\delta}\left(\sum_{i=1}^N p_i\mathrm{I}_{\widehat{y}_i}\right)} \rho_{{\mathbb{Q}}|_{Y}}\left(Y^\top \alpha \right) \right\}.
\end{align}} We begin with risk measures of the form
\begin{align}\label{rho_1}
\rho^{(1)}(Z)=\inf _{t \in \mathbb{R}^k} \mathbb{E}\left[\ell(Z, t)\right], 
\end{align}
as considered in \citet{N24}, where only conservative reformulations were available. In contrast, our approach yields reformulations that are not only provably exact, but also substantially simpler--requiring fewer variables, fewer and simpler constraints, and exhibiting a modular structure that facilitates both analysis and computation. {The following property of a function is necessary for the propositions in this section.
\begin{definition}(Level-bounded)    
A function $\ell: \mathbb{R}^k \rightarrow\mathbb{R}$ is level bounded if for each $A \in \mathbb{R}$, the sublevel set $\left\{t \in \mathbb{R}^k: \ell(t) \leq A\right\}$ is bounded (possibly empty).
\end{definition}} As in the case of conditional expectation, we first exploit regularization equivalents to obtain robust optimization reformulations of \eqref{general_rho_union_version}, and then show that these admit convex minimization reformulations.
{\begin{proposition}\label{general_min_exp}
\begin{itemize}
\item[(i)]  Let $\mathbb{D}_{\mathcal{Y}}(y_1,y_2)=||y_1-y_2||_{\mathcal{Y}}$ for any $y_1,y_2\in\mathcal{Y}$. Suppose that $\ell(z,t): \mathbb{R}\times\mathbb{R}^k\rightarrow\mathbb{R}$ is jointly convex in $(z,t)$, level-bounded in $t$ for all $z \in \mathbb{R}$, and Lipschitz continuous in $z$ for all $t \in \mathbb{R}^k$ with a uniform Lipschitz constant $\operatorname{Lip}(\ell)$. The problem \eqref{general_rho_union_version} with $\rho=\rho^{(1)}$ is equivalent to the robust optimization problem
    \begin{align*}
        \inf_{\alpha\in\mathcal{A},t\in\mathbb{R}^k}\sup_{(p,\delta) \in \mathcal{V}_+ } \sum_{i=1}^N p_i \ell(\widehat{y}_i^\top\alpha,t)  + \operatorname{Lip}(\ell)\delta ||\alpha||_*,
    \end{align*}
    and can be solved by the convex program\begin{align*}
		&\inf_{\alpha\in\mathcal{A},t\in\mathbb{R}^k,v}\sigma^*(v|\mathcal{V}_+)\\
		&~~~~~~\text{s.t.}~v_i\ge\ell(\widehat{y}_i^\top\alpha,t),i=1,...,N,\\
        &~~~~~~~~~~~~v_{N+1}\ge \operatorname{Lip}(\ell)\|\alpha\|_*.
	\end{align*}
    \item[(ii)] Let $\mathbb{D}_{\mathcal{Y}}(y_1,y_2)=||y_1-y_2||_{\mathcal{Y}}^q$ for any $y_1,y_2\in\mathcal{Y}$, $q\in(1,\infty)$ and $k=1$. Assume the loss function $\ell$ is the form $\ell(z, t):=\left(C\bar{\ell}_i(z-t)\right)^q$, where $C>0$, and $\bar{\ell}_i$ is one of the following:
    \begin{itemize}
	\item [(a)] $\bar{\ell}_1(x)=\left(\left|x-b_1\right|-b_2\right)_{+}$with some $b_1 \in \mathbb{R}$ and $b_2 \geqslant 0$;
	\item[(b)] $\bar{\ell}_2(x)=\left|x-b_1\right|+b_2$ with some $b_1 \in \mathbb{R}$ and $b_2>0$.
\end{itemize}
The problem \eqref{general_rho_union_version} with $\rho=\rho^{(1)}$ is equivalent to the robust optimization problem
    $$ \inf_{\alpha \in\mathcal{A},t\in\mathbb{R}} \sup_{(p,\delta) \in \mathcal{V}_+ } \left[\left(\sum_{i=1}^N p_i \ell(\widehat{y}_i^\top\alpha,t) \right)^{1/q} + C\delta^\frac{1}{q} ||\alpha||_*\right]^{q},$$	
and can be solved by the convex program	
	\begin{align*}
		&\inf_{\alpha\in\mathcal{A},t\in\mathbb{R},y\geq 0,v}\left[\sigma^*(v|\mathcal{V}_+)+C_q^1y+C_qv_{N+1}\left\|\frac{\alpha}{v_{N+1}}\right\|_*^{\frac{q}{q-1}}\right]^q\\
		&~~~~~~~~\text{s.t.}~v_i\ge \frac{\ell(\widehat{y}_i^\top\alpha,t)}{y^{q-1}}, ~i=1,...,N,\\
		&~~~~~~~~~~~~~~v_{N+1}\ge 0,
	\end{align*}
	where $C_q^1=q^{\frac{1}{1-q}}-q^{\frac{q}{1-q}}$ and $C_q=C^{\frac{q}{q-1}}\left(q^{\frac{1}{1-q}}-q^{\frac{q}{1-q}}\right)$.  
\end{itemize}   
\end{proposition}}

For general convex loss functions $\ell$, we obtain the following convex reformulation by leveraging two key ingredients inherited from the unconditional Wasserstein DRO setting: (i) the projection property, and (ii) the validity of a min–max switching argument, established in the appendix (see Lemma~\ref{high_dimen_min_max}).

\begin{proposition} \label{pro2*}
Let $\mathbb{D}_{\mathcal{Y}}(y_1,y_2)=||y_1-y_2||_{\mathcal{Y}}^q$ for any $y_1,y_2\in\mathcal{Y}$ and $q\in(1,\infty)$. Assume the loss function $\ell$ is the form $\ell(z,t)=\sum_{q_1\in\mathcal{I}}C_{q_1}\left(\bar{\ell}_{q_1}(z,t)\right)^{q_1}$, with a finite index set $\mathcal{I}\subset[1,q]$  and $C_{q_1}\ge 0$, where $\bar{\ell}_{q_1}(z,t):\mathbb{R}\times\mathbb{R}^k\rightarrow\mathbb{R}$ is nonnegative, jointly convex in $(z,t)\in \mathbb{R}^{k+1}$, level-bounded in $t\in\mathbb{R}^k$ for all $z \in \mathbb{R}$, and Lipschitz continuous in $z$ for  all $t \in \mathbb{R}^k$ with a uniform Lipschitz constant $\operatorname{Lip}(\bar{\ell}_{q_1})$, for $q_1\in\mathcal{I}$. The problem \eqref{general_rho_union_version} with $\rho=\rho^{(1)}$ can be solved by the convex program
	\begin{align*}
		&\inf_{\alpha\in\mathcal{A},t\in\mathbb{R}^k,\eta\ge 0,v}~\sigma^*(v|\mathcal{V}_+)\\
		&~~~~~~~~~\text{s.t.}~v_i\ge \sup _{z \in \mathbb{R}}\left\{z \cdot \widehat{y}_i^{\top} \alpha-(\ell(\cdot,t))^*(z)+\eta \cdot \frac{|z|^{q^*}}{q^* q^{q^*-1}}\right\}, ~i=1,...,N,\\
		&~~~~~~~~~~~~~~~v_{N+1}\ge ||\alpha||_*^q/\eta^{q-1}.
	\end{align*} 
\end{proposition}
\begin{remark}
    Note that if $\bar{\ell}_{q_1}(z,t)=\bar{\ell}_{q_1}(z)$, $C_{q_1}$ can relax the nonnegative constraint, that is, $C_{q_1}\in\mathbb{R}$, and $\bar{\ell}_{q_1}(z)$ also can relax the level-bounded and convex constraint.
\end{remark}

As a demonstrative example, we apply our framework to the mean-variance portfolio allocation problem studied in \citet{N24}, showcasing its ability to yield not only provably exact but also markedly simpler reformulations.

\begin{example}\label{re:variance}
The robustified conditional  mean-variance portfolio allocation problem in \cite{N24}  is as follows: \begin{align}\label{variance}
\min _{\alpha \in \mathcal{A}} \sup _{\substack{\mathbb{Q} \in \mathbb{B}_{\delta_0}(\widehat{\mathbb{P}}) \\ \mathbb{Q}\left(X \in \mathcal{N}_\gamma(x_0)\right) \in \mathcal{N}_{\omega}}} \operatorname{Variance}_{\mathbb{Q}}\left[Y^{\top} \alpha \mid X \in \mathcal{N}_\gamma\left(x_0\right)\right]-\theta \cdot \mathbb{E}_{\mathbb{Q}}\left[Y^{\top} \alpha \mid X \in \mathcal{N}_\gamma\left(x_0\right)\right],    
\end{align}
where $\mathcal{N}_{\omega}=[\epsilon_0,1]$, $\epsilon_0\in(0,1]$ and $\theta\ge 0$. Suppose in addition that $\mathcal{X}=\mathbb{R}^{n_1}, \mathcal{Y}=\mathbb{R}^{n_2}, \mathbb{D}_{\mathcal{X}}(x, \widehat{x})=\|x-\widehat{x}\|^2$ and $\mathbb{D}_{\mathcal{Y}}(y, \widehat{y})=\|y-\widehat{y}\|_2^2$. The problem \eqref{variance} can be solved by the convex program
\begin{align*}
&\inf_{\alpha\in\mathcal{A},~0\le\eta\le1,~t\in\mathbb{R},z,s}~z^\top b-\frac{\theta^2}{4}-t\theta\\
&~~~~\text{s.t.}~z\ge 0,~A^\top z\ge s\\
&~~~~~~~~~s_i\ge \frac{\left(\widehat{y}_i^{\top} \alpha-\theta/2-t\right)^2}{1-\eta}, ~i=1,...,N,\\
&~~~~~~~~~s_{N+1}\ge ||\alpha||_*^2/\eta,~s_{N+2}\ge 0,
\end{align*}
where $A\in\mathbb{R}^{(N+6)\times(N+2)}$,and  $b\in\mathbb{R}^{N+6}$ are defined by 
 \begin{align*}
         A = \begin{pmatrix}
		1& \cdots & 1 & 0 &0 \\
		-1& \cdots & -1 & 0&0\\
		0& \cdots & 0 & 0&1\\
		0& \cdots & 0 & 0&-1\\
		1 & \cdots & 0 & 0&-\frac{1}{N}\\
		\vdots & \ddots & \vdots&\vdots&\vdots \\
		0 & \cdots & 1&0&-\frac{1}{N}\\
		-a_1 & \cdots & -a_{N}&-a_{N+1}&-a_{N+2}\\
		a_1 & \cdots & a_{N}&a_{N+1}& a_{N+2}
	\end{pmatrix},\quad\text{ with }
  a = \begin{pmatrix}-d_1 \\ \vdots \\ -d_m \\ d_{m+1} \\ \vdots \\ d_N \\ -1 \\[.5ex]\displaystyle \delta_0 - \frac1N\sum_{i=m+1}^N d_i\end{pmatrix},\quad\text{ and }
  b = 
    \begin{pmatrix}
      1 \\ -1 \\[3pt]
      \tfrac1{\epsilon_0} \\ -1 \\[3pt]
      0 \\ \vdots \\ 0
    \end{pmatrix}.
    \end{align*}
\end{example}

\begin{remark}\label{comparison_mean_variance}
One can verify that the constraint $s_i\ge \frac{\left(\widehat{y}_i^{\top} \alpha-\theta/2-t\right)^2}{1-\eta}$, $i=1,...,N$, can be reformulated as the following second order cone constraints
\begin{align*}
    \left\|\left[\begin{array}{c}2 \widehat{y}_i^{\top} \alpha-2t-\theta \\s_i-(1-\eta)\end{array}\right]\right\|_2 \leq s_i+1+\eta, ~i=1,...,N.
\end{align*}
and that the constraint $s_{N+1}\ge ||\alpha||_*^2/\eta$ can be reformulated as the following second order cone constraint
\begin{align*}
    \left\|\left[\begin{array}{c}2 \alpha \\\eta-s_{N+1}\end{array}\right]\right\|_2 \leq \eta-s_{N+1}.
\end{align*}
Thus, our reformulation is a second-order cone problem (SOCP), while the reformulation of Proposition 4.5 in \cite{N24} is a semi-definite problem (SDP). From this perspective, our formulation has a simpler form and is easy to apply.
\end{remark}

Our framework also accommodates other prominent risk measures with expectation-based structure, including those of the form
\begin{align}\label{rho_2}
\rho^{(2)}(Z)=\inf _{t \in \mathbb{R}}\left\{t+\left(\mathbb{E}\left[\ell^q(Z, t)\right]\right)^{1 /q}\right\},   
\end{align}
for $q \in [1, \infty)$, which lie beyond the scope of both \citet{N24} and \citet{E22}. The derivation of tractable reformulations for this class proceeds analogously, building on the same decomposition principles introduced earlier.

\begin{proposition}\label{general_min_exp_1}
\begin{itemize}
\item[(i)] Let $\mathbb{D}_{\mathcal{Y}}(y_1,y_2)=||y_1-y_2||_{\mathcal{Y}}$ for any $y_1,y_2\in\mathcal{Y}$ and $q=1$. Suppose that $\ell(z,t): \mathbb{R}\times\mathbb{R}\rightarrow\mathbb{R}$ is nonnegative, jointly convex in $(z,t)\in\mathbb{R}^2$, and Lipschitz continuous in $z$ for  all $t \in \mathbb{R}$ with a uniform Lipschitz constant $\operatorname{Lip}(\ell)$. Moreover, assume $t+\ell(z,t)$ is level bounded in $t$ for all $z\in\mathbb{R}$. The problem \eqref{general_rho_union_version} with $\rho=\rho^{(2)}$ is equivalent to the robust optimization problem
    \begin{align*}
        \inf_{\alpha\in\mathcal{A},t\in\mathbb{R}}\sup_{(p,\delta) \in \mathcal{V}_+ } t+\sum_{i=1}^N p_i \left[\ell(\widehat{y}_i^\top\alpha,t) \right] + \operatorname{Lip}(\ell)\delta ||\alpha||_*,
    \end{align*}
    and can be solved by the convex program\begin{align*}
		&\inf_{\alpha\in\mathcal{A},t\in\mathbb{R},v}t+\sigma^*(v|\mathcal{V}_+)\\
		&~~~~~~\text{s.t.}~v_i\ge\ell(\widehat{y}_i^\top\alpha,t),i=1,...,N,\\
        &~~~~~~~~~~~~v_{N+1}\ge \operatorname{Lip}(\ell)\|\alpha\|_*.
	\end{align*}
    \item[(ii)] Let $\mathbb{D}_{\mathcal{Y}}(y_1,y_2)=||y_1-y_2||_{\mathcal{Y}}^q$ for any $y_1,y_2\in\mathcal{Y}$ and $q\in(1,\infty)$. Assume the loss function $\ell$ is the form $\ell(z, t):=C\bar{\ell}_i(z-t)$, where $C>1$, and $\bar{\ell}_i$ is one of the following:
   \begin{itemize}
	\item[(a)]  $\bar{\ell}_1(x)=(x-b)_{+}$with some $b \in \mathbb{R}$;
	\item [(b)] $\bar{\ell}_2(x)=\left(\left|x-b_1\right|-b_2\right)_{+}$with some $b_1 \in \mathbb{R}$ and $b_2 \geqslant 0$;
	\item[(c)] $\bar{\ell}_3(x)=\left|x-b_1\right|+b_2$ with some $b_1 \in \mathbb{R}$ and $b_2>0$,
\end{itemize}
or $\ell(z, t)=C(|z|-t)_{+}$. The problem \eqref{general_rho_union_version} with $\rho=\rho^{(2)}$ is equivalent to the robust optimization problem
    $$ \inf_{\alpha \in\mathcal{A},t\in\mathbb{R}} \sup_{(p,\delta) \in \mathcal{V}_+ } t+\left(\sum_{i=1}^N p_i \left[\ell^q(\widehat{y}_i^\top\alpha,t) \right]\right)^{1/q} + C\delta^\frac{1}{q} ||\alpha||_*,$$	
and can be solved by the convex program	
\begin{align*}
	&\inf_{\alpha\in\mathcal{A},t\in\mathbb{R},y\ge0,v}t+\sigma^*(v|\mathcal{V}_+)+C_q^1y+C_qv_{N+1}\left\|\frac{\alpha}{v_{N+1}}\right\|_*^{\frac{q}{q-1}}\\
	&~~~~~~~~\text{s.t.}~v_i\ge \ell^q(\widehat{y}_i^\top\alpha,t)/y^{q-1},i=1,...,N,\\
&~~~~~~~~~~~~~~v_{N+1}\ge 0,
\end{align*}
	where $C_q^1=q^{\frac{1}{1-q}}-q^{\frac{q}{1-q}}$ and $C_q=C^{\frac{q}{q-1}}\left(q^{\frac{1}{1-q}}-q^{\frac{q}{1-q}}\right)$.  
\end{itemize}   
\end{proposition}

For the general case, we again rely on the two key ingredients—the projection property and the min–max switching argument (see Lemma~\ref{high_dimen_min_max})—to derive tractable convex reformulations.

\begin{proposition} \label{pro_rho2_general}
Let $\mathbb{D}_{\mathcal{Y}}(y_1,y_2)=||y_1-y_2||_{\mathcal{Y}}^q$ for any $y_1,y_2\in\mathcal{Y}$ and $q\in(1,\infty)$. Assume that $\ell(z,t)$ is nonnegative, jointly  convex in $(z,t)\in\mathbb{R^2}$, and Lipschitz continuous in $z$ for  all $t \in \mathbb{R}$ with a uniform Lipschitz constant $\operatorname{Lip}(\ell)$. Moreover, assume $t+\ell(z,t)$ is level bounded in $t$ for all $z\in\mathbb{R}$. The problem \eqref{general_rho_union_version} with $\rho=\rho^{(2)}$ can be solved by the convex program
	\begin{align*}
		&\inf_{\alpha\in\mathcal{A},\eta\ge 0,t\in\mathbb{R},v,w\ge 0}~t+\sigma^*(v|\mathcal{V}_+)+C_q^1w\\
		&~~~~~~~~~~~\text{s.t.}~v_i\ge \sup _{z \in \mathbb{R}}\left\{z \cdot \widehat{y}_i^{\top} \alpha-\left(\frac{\ell^q(\cdot,t)}{w^{q-1}}\right)^*(z)+\eta \cdot \frac{|z|^{q^*}}{q^* q^{q^*-1}}\right\}, ~i=1,...,N,\\
		&~~~~~~~~~~~~~~~~~v_{N+1}\ge ||\alpha||_*^q/\eta^{q-1}.
	\end{align*} 
\end{proposition}

Our methodology also applies to risk measures defined implicitly through expectation constraints, such as the utility-based shortfall risk measure, defined by
    \begin{align*}
\rho^{(3)}(Z)=\inf \left\{\kappa \in \mathbb{R}: \mathbb{E}\left[u(-Z-\kappa)\right] \le l\right\},
\end{align*}
where $u$ is a convex, increasing function and $l$ is a fixed constant in the interior of the range of $u$. The same reformulation principles apply.

\begin{proposition}\label{pro:shortfall_1}
Let $\mathbb{D}_{\mathcal{Y}}(y_1,y_2)=||y_1-y_2||_{\mathcal{Y}}$ for any $y_1,y_2\in\mathcal{Y}$ and $u: \mathbb{R} \rightarrow \mathbb{R}$ be a convex, increasing and Lipschitz continuous function with Lipschitz constant $\operatorname{Lip}(u)$.  The problem \eqref{general_rho_union_version} with $\rho=\rho^{(3)}$
    is equivalent to the robust optimization problem 
    \begin{align*}
    &\min_{\alpha \in \mathcal{A}}\sup_{(p,\delta)\in\mathcal{V}_+}\inf_{\kappa \in \mathbb{R}}\kappa\\
			&~~~~~~~~~~~~~~~~~~~~~\text{s.t.} \sum_{i=1}^N p_i u(-\widehat{y}_i^\top\alpha-\kappa) + \operatorname{Lip}(u)\delta ||\alpha||_*\le l\\
        =&\inf_{\alpha \in \mathcal{A},\kappa \in \mathbb{R}}\kappa\\
			&~~~~~\text{s.t.} \sum_{i=1}^N p_i u(-\widehat{y}_i^\top\alpha-\kappa) + \operatorname{Lip}(u)\delta ||\alpha||_*\le l~~~ \forall (p,\delta)\in\mathcal{V}_+,
    \end{align*}
    and can be solved by the convex program
\begin{align*}
 	&\inf_{\alpha \in \mathcal{A},\kappa \in \mathbb{R},v}\kappa\\
 	&~~~~~~\text{s.t.}~\sigma^*(v|\mathcal{V}_+)\le l,\\
    &~~~~~~~~~~~~v_i\ge u(-\widehat{y}_i^\top\alpha-\kappa),i=1,...,N,\\
    &~~~~~~~~~~~~v_{N+1}\ge \operatorname{Lip}(u)||\alpha||_*.
 \end{align*} 
\end{proposition}
\begin{remark}\label{q_order_utility_shortfall}
   Note that when $\mathbb{D}_{\mathcal{Y}}(y_1,y_2)=\|y_1-y_2\|_{\mathcal{Y}}^q$ for some $q\in(1,\infty)$, a convex program reformulation is still possible, but—unlike the earlier cases where several forms of $u$ were admissible—it applies only to the linear specification $u(x)=Cx+b$ with $C>0$, as required by monotonicity. In this setting, problem~\eqref{general_rho_union_version} admits the equivalent convex program:
\begin{align*}
    &\inf_{\alpha \in \mathcal{A},\,\kappa \in \mathbb{R},\,v}\;\kappa\\
    &~~~~~~\text{s.t.}~~\sigma^*(v \mid \mathcal{V}_+) + C_q v_{N+1}\left\|\tfrac{\alpha}{v_{N+1}}\right\|_*^{\tfrac{q}{q-1}} \le l,\\
    &~~~~~~~~~~~~~v_i \ge u(-\widehat{y}_i^\top \alpha - \kappa), \quad i=1,\dots,N,\\
    &~~~~~~~~~~~~~v_{N+1} \ge 0,
\end{align*}
where $C_q = C^{\tfrac{q}{q-1}}\!\left(q^{\tfrac{1}{1-q}} - q^{\tfrac{q}{1-q}}\right)$.

\end{remark}

Finally, we arrive at the following convex reformulation for the general case of a convex utility function $u$ for higher order $q$.
\begin{proposition}\label{pro:shortfall_2}
Let $\mathbb{D}_{\mathcal{Y}}(y_1,y_2)=||y_1-y_2||_{\mathcal{Y}}^q$ for any $y_1,y_2\in\mathcal{Y}$ and $q\in(1,\infty)$. Assuming the function $u$ is convex and increasing and satisfies $u(x)-u(x_0) \leqslant L|x-x_0|^q+M, x \in \mathbb{R}$ for some $L,M>0$, and some $x_0\in\mathbb{R}$, the problem \eqref{general_rho_union_version} with $\rho=\rho^{(3)}$ can be solved by the convex program
\begin{align*}
		&\inf_{\alpha \in \mathcal{A},\kappa \in \mathbb{R},\eta\ge0,v}\kappa\\
		&~~~~~~~~~\text{s.t.}~\sigma^*(v|\mathcal{V}_+)\le l,\\
		&~~~~~~~~~~~~~~~v_i\ge \sup _{z \in \mathbb{R}}\left\{z \cdot \widehat{y}_i^{\top} \alpha-\left(u(\cdot,\kappa)\right)^*(z)+\eta \cdot \frac{|z|^{q^*}}{q^* q^{q^*-1}}\right\}, ~i=1,...,N,\\
		&~~~~~~~~~~~~~~~v_{N+1}\ge ||\alpha||_*^q/\eta^{q-1}.
	\end{align*}
\end{proposition}

\subsection{Conditional Risk Functionals in Distortion Form}\label{distortion_form}
We now turn to risk functionals defined through distortion functions, which play a central role in both risk measurement--as distortion risk measures--and decision theory, where they appear as rank-dependent utility functionals. For general distortion functions, these risk measures lack finite-dimensional expectation representations, i.e., \eqref{rho_1}, rendering them fully intractable under prior approaches such as \citet{N24} and \citet{E22}. A key highlight of our union-ball framework is that, even in this challenging setting, it yields tractable reformulations via regularization equivalents--offering, to our knowledge, the first viable solution framework for distributionally robust conditional distortion problems. We first give the definition of distortion risk measure.
\begin{definition}(distortion risk measure)
    Given a random variable $Z$ and a distortion function $h:[0,1] \rightarrow[0,1]$, that is, $h$ is an increasing function satisfies that $h(0)=0$, and $h(1)=1$, the distortion risk measure is defined as $\rho_h(Z)=\int_{\mathbb{R}}z \mathrm{~d} h\left(F_Z(z)\right)$, where $F_Z$ is the distribution of $Z$. Specifically, if $h$ is a convex function, $\rho_h$ is a convex, thus coherent, risk measure.
\end{definition}

To present our reformulation results, we first introduce the admissible set
\begin{align*}
			\mathcal{V}_+^h=\left\{\left(p,\delta,\bar{p}\right):(p,\delta)\in\mathcal{V}_+,\bar{p}\in\mathcal{V}_p^h\right\},
		\end{align*}
        where
		$$\mathcal{V}_p^h=\left\{\bar{p}:\sum_{i=1}^N \bar{p}_i=1,\bar{p}_i \geq 0, i=1,...,N, h\left(\sum_{i \in J} p_i\right) \leq \sum_{i \in J} \bar{p}_i, \forall J \subset\{1,...,N\}  \right\},$$which is the risk envelope associated with coherent distortion risk measures. Leveraging this set characterization, together with the decomposition principle and regularization equivalents, we obtain the following convex reformulation.
\begin{proposition} \label{pro4*}Assume $\rho$ is a convex distortion risk measure with convex dstortion function $h$. Denote the left-derivative of $h$ on $(0,1]$ as $h'_-$. Then we have following results:
\begin{itemize}
\item [(i)]Let $\mathbb{D}_{\mathcal{Y}}(y_1,y_2)=||y_1-y_2||_{\mathcal{Y}}$ for any $y_1,y_2\in\mathcal{Y}$ and $\ell: \mathbb{R} \rightarrow \mathbb{R}$ be a convex and Lipschitz continuous function.  
The problem \eqref{union-version} with $\rho=\rho_h$ is equivalent to the robust optimization problem 
    \begin{align}
			\min _{\alpha \in \mathcal{A}}\sup_{(p,\delta) \in \mathcal{V}_+}\rho_{\sum_{i=1}^N p_i\mathrm{I}_{\widehat{y}_i}}^h\left[\ell( Y^\top\alpha)\right]+\operatorname{Lip}(\ell)\left\|h_{-}^{\prime}\right\|_{\infty} \delta\|{\alpha}\|_*.\label{pro4*-1}
    \end{align}
    and can be solved by the convex program
\begin{align*}
	&\inf_{\alpha\in\mathcal{A}, v}~\sigma^*(v|\mathcal{V}_+^h)\\
	&~~~\text{s.t.}~v_i\ge 0, i=1,...,N,~v_{N+1}\ge \operatorname{Lip}(\ell)\left\|h_{-}^{\prime}\right\|_{\infty}\|\alpha\|_*,\\
    &~~~~~~~~~v_{i+N+1}\ge \ell(\widehat{y}_i^\top\alpha), i=1,...,N.
\end{align*}
\item[(ii)] 
Let $\mathbb{D}_{\mathcal{Y}}(y_1,y_2)=||y_1-y_2||_{\mathcal{Y}}^q$ for any $y_1,y_2\in\mathcal{Y}$ and $q\in(1,\infty)$, and $\frac{1}{q}+\frac{1}{q^*}=1$. Assume the function $\ell$ takes one of the following two forms, multiplied by a constant $C>0$:
\begin{itemize}
	\item [(a)] $\ell_1(x)=x+b$ or $\ell_1(x)=-x+b$ with some $b \in \mathbb{R}$;
	\item [(b)] $\ell_2(x)=\left|x-b_1\right|+b_2$ with some $b_1, b_2 \in \mathbb{R}$.
\end{itemize}
 The problem \eqref{union-version} with $\rho=\rho_h$ is equivalent to the robust optimization problem 
    \begin{align*}
			\min _{\alpha \in \mathcal{A}}\sup_{(p,\delta) \in \mathcal{V}_+}\rho_{\sum_{i=1}^N p_i\mathrm{I}_{\widehat{y}_i}}^h\left[\ell( Y^\top\alpha)\right]+C\left\|h_{-}^{\prime}\right\|_{q^*} \delta^{\frac{1}{q}}\|{\alpha}\|_*.
    \end{align*}
    and can be solved by the convex program
\begin{align*}
&\inf_{\alpha\in\mathcal{A},v}\sigma^*(v|\mathcal{V}_+^h)+C_{h,q}v_{N+1}\left\|\frac{\alpha}{v_{N+1}}\right\|_*^{\frac{q}{q-1}}\\
	&~~~\text{s.t.}~v_i\ge 0, i=1,...,N+1,\\
    &~~~~~~~~~v_{i+N+1}\ge \ell(\widehat{y}_{i}^\top\alpha), i=1,...,N,
\end{align*}
where $C_{h,q}=\left(q^{\frac{1}{1-q}}-q^{\frac{q}{1-q}}\right)\left(C\left\|h_{-}^{\prime}\right\|_{q^*}\right)^{\frac{q}{q-1}}$.
\end{itemize}
\end{proposition}

{Similar to Examples \ref{full-conjugate} and \ref{partial-conjugate}, we now apply the admissible sets from Theorems~\ref{thm1} and~\ref{thm2} to illustrate how the conjugate of the support function~$\sigma^*(v|\mathcal{V}^h_+)$ is derived. Recall that $\mathcal{N}_{\omega}=[\omega_1,\omega_2]\subset(0,1]$, and ${\epsilon}=\min_{\beta\in\mathcal{N}_{\omega}}\beta=\omega_1$.}
\begin{example}\label{distortion_conjugate_full}
Following the definition of $\mathcal V$ in Theorem \ref{thm1}, the conjugate function $\sigma^*(v|\mathcal{V}_+^h)$  is given by 
\begin{align*}
&\inf-\frac{\eta_1}{\omega_2}+\frac{\eta_2}{\omega_1}-\lambda-\mu+\sum_{j=1}^{2^N-2}\zeta_jh^*(\frac{-\phi_j}{\zeta_j})\\
&~\text{s.t.}~\lambda, \mu,\xi_2\in\mathbb{R},\left\{\phi_j\right\}_{j=1}^{2^N-2}\in\mathbb{R}^{2^N-2},\left\{\zeta_j\right\}_{j=1}^{2^N-2},\left\{\tau_i^1\right\}_{i=1}^N,\left\{\tau_i^2\right\}_{i=1}^N,\left\{\gamma_i\right\}_{i=1}^N,  \eta_1, \eta_2,\xi_1\ge 0\\
&~~~~~~-\xi_2(\delta_0-\frac{1}{N} \sum_{i=m+1}^Nd_i)+\frac{1}{N} \sum_{i=1}^N \tau_i^2+\eta_1-\eta_2\le 0,\\
&~~~~~~~v_i+\lambda+\tau_i^1-\tau_i^2+c_i+\sum_{j:i\in I_j}\phi_j\le 0,~  i=1,...,N,\\
&~~~~~~~c_i=\xi_2 d_i,  i=1, \ldots, m,~c_i= -\xi_2 d_i,  i=m+1, \ldots, N,\\
&~~~~~~~v_{N+1+i}+\mu+{\sum_{j:i\in I_j} \zeta_j}+\gamma_i \le 0,~i=1,...,N,\\
&~~~~~~~v_{N+1}+\xi_1+\xi_2 \le 0.
\end{align*}

\end{example}

\begin{example}\label{distortion_conjugate_partial}
Following the definition of $\mathcal V$ in Theorem \ref{thm2}, the conjugate function $\sigma^*(v|\mathcal{V}_+^h)$  is given by  
\begin{align*}
&\inf-\lambda-\mu+\frac{1}{N\epsilon}\sum_{i=1}^N\tau_i^2-\xi_2\delta_0+\sum_{j=1}^{2^N-2}\zeta_jh^*(\frac{-\phi_j}{\zeta_j})\\
&~\text{s.t.}~\lambda, \mu,\xi_2\in\mathbb{R},\left\{\phi_j\right\}_{j=1}^{2^N-2}\in\mathbb{R}^{2^N-2},\left\{\zeta_j\right\}_{j=1}^{2^N-2},\left\{\tau_i^1\right\}_{i=1}^N,\left\{\tau_i^2\right\}_{i=1}^N,\left\{\gamma_i\right\}_{i=1}^N, \xi_1\ge 0,\\
&~~~~~~~v_i+\lambda+\tau_i^1-\tau_i^2+c_i'+\sum_{j:i\in I_j}\phi_j\le 0,~  i=1,...,N,\\
&~~~~~~~c_i'=\xi_2 d_i,  i=1, \ldots, m,~c_i= 0,  i=m+1, \ldots, N,\\
&~~~~~~~v_{N+1+i}+\mu+{\sum_{j:i\in I_j} \zeta_j}+\gamma_i \le 0,~i=1,...,N,\\
&~~~~~~~v_{N+1}+\xi_1+\xi_2 \le 0.
\end{align*}
\end{example}

The mean-CVaR risk measure considered in the example of \citet{N24} is a convex distortion risk measure. While their approach yields only conservative approximations to the associated portfolio allocation problem, we demonstrate below how our framework provides reformulations that are not only provably exact but also markedly simpler, featuring fewer variables, fewer constraints, and a modular structure that streamlines computation.
\begin{example}\label{re:CVaR}
The robustified conditional  mean-$\CVaR$ portfolio allocation problem in \cite{N24}  is as follows: \begin{align}\label{CVaR}
\min _{\alpha \in \mathcal{A}} \sup _{\substack{\mathbb{Q} \in \mathbb{B}_{\delta_0}(\widehat{\mathbb{P}}) \\ \mathbb{Q}\left(X \in \mathcal{N}_\gamma(x_0)\right) \in \mathcal{N}_{\omega}}} \operatorname{CVaR}_{\mathbb{Q}}^{1-\kappa}\left[Y^{\top} \alpha \mid X \in \mathcal{N}_\gamma\left(x_0\right)\right]-\theta \cdot \mathbb{E}_{\mathbb{Q}}\left[Y^{\top} \alpha \mid X \in \mathcal{N}_\gamma\left(x_0\right)\right],    
\end{align}
where $\mathcal{N}_{\omega}=[\epsilon_0,1]$, $\epsilon_0\in(0,1]$, $\theta\ge 0$,  $\kappa\in(0,1]$, $$\operatorname{CVaR}_{\mathbb{Q}}^{1-\kappa}(Z)=\frac{1}{\kappa}\int_{1-\kappa}^1\operatorname{VaR}_{\mathbb{Q}}^{\phi}(-Z)\operatorname{d}\phi,$$
and $\operatorname{VaR}_{\mathbb{Q}}^{\phi}(Z)= \inf \left\{z \in \mathbb{R}: \mathbb{Q}(Z\le z) \geqslant \phi\right\}$. Suppose in addition that $\mathcal{X}=\mathbb{R}^{n_1}, \mathcal{Y}=\mathbb{R}^{n_2}, \mathbb{D}_{\mathcal{X}}(x, \widehat{x})=\|x-\widehat{x}\|^2$ and $\mathbb{D}_{\mathcal{Y}}(y, \widehat{y})=\|y-\widehat{y}\|_2^2$. The problem \eqref{CVaR} can be solved by the convex program 
\begin{align*}
&\inf_{\alpha\in\mathcal{A},t\in\mathbb{R},z\ge 0,v_{N+1},s}~z^\top b+s\\
&~~~~~~~~~~~~\text{s.t.}~v_{N+1}\ge0,~A^\top z\ge (\ell(\widehat{y}_1^\top\alpha,t),\cdots,\ell(\widehat{y}_N^\top\alpha,t),v_{N+1},0)^\top,\\
&~~~~~~~~~~~~~~~~~~s\ge C_{\theta,\kappa}v_{N+1}\left\|\frac{\alpha}{v_{N+1}}\right\|_2^{2},
\end{align*}
where $\ell(z,t)=t+\frac{1}{\kappa}(-z-t)_+-\theta z$,  $A\in\mathbb{R}^{(N+6)\times(N+2)}$,and  $b\in\mathbb{R}^{N+6}$ are defined by 
 \begin{align*}
         A = \begin{pmatrix}
		1& \cdots & 1 & 0 &0 \\
		-1& \cdots & -1 & 0&0\\
		0& \cdots & 0 & 0&1\\
		0& \cdots & 0 & 0&-1\\
		1 & \cdots & 0 & 0&-\frac{1}{N}\\
		\vdots & \ddots & \vdots&\vdots&\vdots \\
		0 & \cdots & 1&0&-\frac{1}{N}\\
		-a_1 & \cdots & -a_{N}&-a_{N+1}&-a_{N+2}\\
		a_1 & \cdots & a_{N}&a_{N+1}& a_{N+2}
	\end{pmatrix},\quad\text{ with }
  a = \begin{pmatrix}-d_1 \\ \vdots \\ -d_m \\ d_{m+1} \\ \vdots \\ d_N \\ -1 \\[.5ex]\displaystyle \delta_0 - \frac1N\sum_{i=m+1}^N d_i\end{pmatrix},\quad\text{ and }
  b = 
    \begin{pmatrix}
      1 \\ -1 \\[3pt]
      \tfrac1{\epsilon_0} \\ -1 \\[3pt]
      0 \\ \vdots \\ 0
    \end{pmatrix}.
    \end{align*}
\end{example}
\begin{remark}\label{comparison_mean_CVaR}
    One can verify that the constraint $s\ge C_{\theta,\kappa}v_{N+1}\left\|\frac{\alpha}{v_{N+1}}\right\|_2^{2}$ can be reformulated as the following SOC constraint
    \begin{align*}
        \left\|\left[\begin{array}{c}2 \alpha \\s-C_{\theta,\kappa}v_{N+1}\end{array}\right]\right\|_2 \leq s+C_{\theta,\kappa}v_{N+1}.
    \end{align*}
Thus, our reformulation is a SOCP, while the reformulation of Proposition 4.6 in \cite{N24} is a SDP. From this perspective, our formulation has a simpler form and is easy to apply.
\end{remark}

\section{Scalable Solution for Large-Scale Distributionally Robust CRM}
We now highlight another key advantage of the union-ball decomposition introduced in the previous section: beyond yielding unified and tractable reformulations, it naturally facilitates the development of scalable solution methods. While our earlier convex programs--for instance, those targeting rank-dependent utility--are exact and structurally elegant, they become computationally prohibitive as the sample size $N$ grows. In this section, we show how the decomposition structure lends itself to a cutting-plane algorithm, effectively addressing large-scale distributionally robust CRM problems--an area largely unexplored in the existing literature.

We base our development on the distributonally robust conditional rank-dependent utility problem \eqref{pro4*-1}, which subsumes distributonally robust conditional expectation problem as a special case. Following Proposition \ref{pro4*}, 
we have 
\begin{align}
&\min _{\alpha \in \mathcal{A}}\sup_{(p,\delta) \in \mathcal{V}_+}\rho_{\sum_{i=1}^N p_i\mathrm{I}_{\widehat{y}_i}}^h\left[\ell( Y^\top\alpha)\right]+\operatorname{Lip}(\ell)\left\|h_{-}^{\prime}\right\|_{\infty} \delta\|{\alpha}\|_*\label{q=1,distortion}\\
=&\min _{\alpha \in \mathcal{A}}\sup_{(p,\delta) \in \mathcal{V}_+,\bar{p}\in\mathcal{V}_p^h}\sum_{i=1}^N\bar{p}_i \ell\left( \widehat{y}_i^\top\alpha\right)+\operatorname{Lip}(\ell)\left\|h_{-}^{\prime}\right\|_{\infty} \delta\|{\alpha}\|_*\notag\\
:=&\min _{\alpha \in \mathcal{A}}\sup_{(p,\delta,\bar{p}) \in \mathcal{V}_+^h}\sum_{i=1}^N\bar{p}_i \ell\left( \widehat{y}_i^\top\alpha\right)+\operatorname{Lip}(\ell)\left\|h_{-}^{\prime}\right\|_{\infty} \delta\|{\alpha}\|_*,\label{q=1,distortion,ex}
\end{align}
where $\mathcal{V}_+^h=\left\{(p,\delta,\bar{p}):(p,\delta) \in \mathcal{V}_+,\bar{p}\in\mathcal{V}_p^h\right\}$. 

The computational burden lies in solving the inner supremum problem where the set $\mathcal{V}_p^h$ involves exponentially many constraints as $N$ increases. The structure of inner maximization problem of \eqref{q=1,distortion,ex}, which is linear in variable $(\delta,\bar{p})$, for any fixed $\alpha \in \mathcal{A}$, naturally leads to the design of a cutting plane method. The central idea is to iteratively refine the uncertainty set by adding only the most violated constraints, allowing convergence without necessarily enumerating all possibilities.

\begin{algorithm}[p]
\caption{Cutting-Plane Scheme}\label{alg:cutting-plane}
\DontPrintSemicolon
\SetKwInOut{Input}{Input}\SetKwInOut{Output}{Output}

\Input{tolerance $\psi_{\mathrm{tol}}>0$; feasible $(p_f,\delta_f)\in\mathcal{V}_+$; hypothesis set $\mathcal{A}$.}
\Output{parameter ${\alpha}^*$ and an (approximate) worst-case scenario $(p^*,\delta^*,\bar p^*)$.}

\textbf{Initialization:} Set the scenario pool $\mathcal{U}_1\gets\{(p_f,\delta_f,\bar p_f)\}$ with $\bar p_f:=p_f$; set $j\gets 1$.\;

\Repeat{$u^b_j -l^b_j \le \psi_{\mathrm{tol}}$}{
  \textbf{Lower bound with uncertainty set $\mathcal{U}_j$:}\;
  \begin{equation*}
  l^b_j \gets \min_{\alpha\in\mathcal{A}} \max_{(p,\delta,\bar p)\in \mathcal{U}_j}
  \left\{
     \sum_{i=1}^N \bar p_i\,\ell(\widehat{y}_i^\top \alpha)
     + \operatorname{Lip}(\ell)\,\lVert h_{-}'\rVert_\infty\,\delta\,\lVert \alpha\rVert_*
  \right\}
  \end{equation*}
  Let $\alpha_j$ be an optimal solution to obtain the lower bound $l^b_j$.\;

  \textbf{Upper bound with $\alpha_j$:}\;
  Form the order statistics of the realized losses
  \begin{equation*}
    r_{(1)} \le \cdots \le r_{(N)},\qquad r_i := \ell(\widehat{y}_i^\top \alpha_j).
  \end{equation*}
  Solve
  \begin{align}
  u^b_j \gets \sup_{(p,\delta)\in\mathcal{V}_+}\;
   &\sum_{k=1}^{N-1} 
     h\!\left(\sum_{i=1}^k p_{(i)}\right)\,\big(r_{(k)}-r_{(k+1)}\big)
     + r_{(N)}  + \operatorname{Lip}(\ell)\,\lVert h_{-}'\rVert_\infty\,\delta\,\lVert \alpha_j\rVert_* . \label{dis}
  \end{align}
  Let $(p_j^*,\delta_j^*)$ be an optimal solution to obtain the upper bound $u^b_j$ and set
  \begin{equation*}
     \bar p_{j(i)}^*
     := h\!\left(\sum_{k=1}^{i} p_{j(k)}^*\right)
      - h\!\left(\sum_{k=1}^{i-1} p_{j(k)}^*\right),\qquad i=1,\dots,N.
  \end{equation*}

  \textbf{Gap check:} If $u^b_j - l^b_j \le \psi_{\mathrm{tol}}$, terminate with
  $\alpha^*\gets \alpha_j$ and $(p^*,\delta^*,\bar p^*)\gets(p_j^*,\delta_j^*,\bar p_j^*)$.\;

  \textbf{Pool augmentation:} Otherwise enlarge the scenario pool\;
  \begin{equation*}
     \mathcal{U}_{j+1} \gets \mathcal{U}_j \cup \{(p_j^*,\delta_j^*,\bar p_j^*)\},
  \end{equation*}
  and set $j\gets j+1$ to continue.\;
}
\end{algorithm}

To efficiently solve \eqref{q=1,distortion,ex}, we proceed as follows:
\begin{enumerate}
\item We begin with an initial single-scenario uncertainty set $\mathcal{U}_1=\{(p,\delta,\bar{p})=(p_f,\delta_f,{p}_f)\}$, where $(p_f,\delta_f)\in \mathcal{V}_+$. The associated scenario $p_f$ clearly lies within $\mathcal{V}_{p_f}^h$. Note that we can always find a feasible solution $(p_f,\delta_f)\in \mathcal{V}_+$ when adopting $\mathcal{V}$ specified in Theorems \ref{thm1} and \ref{thm2}. To guarantee the feasibility of \eqref{q=1,distortion} with $\mathcal{V}$ specified in Theorems \ref{thm1} and \ref{thm2}, let $\delta_0=\delta_{min}+\Delta$, where $\Delta>0$ and $\delta_{min}$ is given by Propositions \ref{feasible_con} and \ref{feasible_con_2}, respectively. Moreover, let $\{v_i\}_{i=1}^N$ be the optimal solution such that the optimization problem in Propositions \ref{feasible_con} and \ref{feasible_con_2} attain the optimal value $\delta_{min}$, respectively. Then for $\mathcal{V}$ specified in Theorem \ref{thm1}, let $p_{f}=(\frac{v_1}{\sum_{i=1}^Nv_i},...,\frac{v_N}{\sum_{i=1}^Nv_i})^\top$ and $\delta_f=\frac{N\Delta}{\sum_{i=1}^Nv_i}$ and for $\mathcal{V}$ specified in Theorem \ref{thm2}, let $p_{f}=(\frac{v_1}{N\epsilon},...,\frac{v_N}{N\epsilon})^\top$ and $\delta_f=\Delta$. One can verify that $(p_f,\delta_f)\in \mathcal{V}_+$ and thus $(p_f,\delta_f,{p}_f)\in \mathcal{V}_+^h$. Hence, we can always obtain a feasible scenario in $\mathcal{V}_+^h$ and then $\mathcal{U}_1\subset\mathcal{V}_+^h$.

\item At iteration $j$, we solve the restricted optimization problem:
\begin{align*}
\min_{\alpha \in \mathcal{A}} \sup_{(p,\delta,\bar{p}) \in \mathcal{U}_j}\sum_{i=1}^N \bar{p}_i \ell\left(\widehat{y}_i^\top \alpha\right)+\operatorname{Lip}(\ell)\left\|h_{-}^{\prime}\right\|_{\infty} \delta\|{\alpha}\|_*.
\end{align*}
Let the optimal solution and optimal value obtained be $(\alpha_j,l^b_j)$, which provides a lower bound to \eqref{q=1,distortion,ex}.

\item Fixing $\alpha_j$, we solve the full inner maximization of \eqref{q=1,distortion,ex} to obtain the worst-case solution $(p_j^*,\delta_j^*)$ and the corresponding upper bound $u^b_j$ to the problem \eqref{q=1,distortion,ex}. One can verify that given $\alpha_j$, we can actually directly solve the inner supremum problem of problem \eqref{q=1,distortion} because the distortion risk measure has explicit form (see \eqref{dis}) and thereby avoid an exponentially growing number of constraints in \eqref{q=1,distortion,ex}.

\item We check the convergence criterion by comparing the gap $u^b_j - l^b_j$ to a predefined tolerance $\psi_{\text{tol}}>0$. If the gap satisfies $u^b_j - l^b_j\le \psi_{\text{tol}}$, the algorithm terminates, returning $\alpha_j$ and $l^b_j$ as approximate optimal solution and optimal value. Otherwise, the scenario $(p_j^*,\delta_j^*,\bar{p}_j^*)$ is appended to the uncertainty set, forming $\mathcal{U}_{j+1}=\mathcal{U}_j\cup(p_j^*,\delta_j^*,\bar{p}_j^*)$, and the process repeats. A critical observation that significantly enhances efficiency is that, given $\alpha\in\mathcal{A}$, for each worst-case distribution $p^*$, the optimal associated distortion-weighted distribution $\bar{p}^*$, achieving the supremum, is explicitly computed by:
$$\bar{p}_{(i)}^*=h\left(\sum_{k=1}^{i} p_{(k)}^*\right)-h\left(\sum_{k=1}^{i-1} p_{(k)}^*\right)$$ 
with the $(\cdot)$ denoting sorting of scenarios according to $\ell(\widehat{y}_i^\top\alpha)$ in ascending order, i.e. $\ell\left(\widehat{y}_{(1)}^\top \alpha\right) \leq \ldots \leq\ell\left(\widehat{y}_{(N)}^\top \alpha\right).$ Hence, each iteration only needs to incorporate a single scenario $(p_j^*,\delta_j^*,\bar{p}_j^*)$ to  the uncertainty set rather than explicitly enumerating exponential constraints, $\{(p_j^*,\delta_j^*,\bar{p}_j): \bar{p}_j\in \mathcal{V}_{p_j^*}^h\}$, to  the uncertainty set.

The fact that $u^b_j - l^b_j$ converges to $0$ as $j$ increases is guaranteed by that  the lower bound of \eqref{q=1,distortion,ex}, $l^b_j$,  increases with respect to iteration $j$ and converges to \eqref{q=1,distortion,ex} as $j$ increases due to $\mathcal{U}_j\subset\mathcal{U}_{j+1}\subset...\subset\mathcal{V}_+^h$.

\end{enumerate}

\section{Numerical Experiment}
In this section, we demonstrate the practical effectiveness of our unified and tractable optimal transport framework through simulation experiments on a portfolio allocation problem governed by a distributionally robust conditional rank-dependent expected utility (RDEU) criterion. These experiments enable a systematic evaluation of our method in a controlled setting and highlight its advantages over both unconditional models and existing conditional approaches. Leveraging the cutting-plane algorithm developed in the previous section, we solve the resulting problems efficiently and underscore the value of incorporating side information in data-driven decision-making.

We adopt the setting of Proposition~\ref{pro4*}(i) for the choice of loss function and distance metric, where the loss function $\ell$ is Lipschitz continuous and convex, and the distance metric is defined as $\mathbb{D}_{\mathcal{Y}}(y_1, y_2) = \|y_1 - y_2\|_{\mathcal{Y}}$. The distributionally robust conditional model is solved using our union-ball formulation, with the admissible set $\mathcal{V}$ specified in Theorem~\ref{thm1}. 

We summarize below the models considered in the comparison, each of which is solved using the cutting-plane algorithm.

\begin{itemize}
    \item[(i)] \textbf{Union-Ball Conditional Distributionally Robust Optimization Model (UB-CDRO)}: This is our proposed model, presented as 
\begin{align*}
    &\min_{\alpha \in \mathcal{A}} \sup_{\mathbb{Q}|_{Y} \in \mathop{\bigcup}\limits_{(p, \delta) \in \mathcal{V}} \mathbb{B}_{\delta}\left(\sum_{i=1}^N p_i\mathrm{I}_{\widehat{y}_i}\right)} \rho_{\mathbb{Q}|_{Y}}\left[\ell(Y, \alpha)\right].
\end{align*}

    \item[(ii)] \textbf{Unconditional Sample Average Approximation (SAA)}: This baseline model uses the empirical distribution of $Y$ without conditioning on $X$-namely, $\widehat{\mathbb{P}}_{N} = \frac{1}{N} \sum_{i=1}^{N} \mathrm{I}_{\widehat{y}_{i}}$:
    \begin{align*}
        \min _{\alpha \in \mathcal{A}} \rho_{\widehat{\mathbb{P}}_{N}}^h\left[\ell( Y^\top\alpha)\right].
    \end{align*}

    \item[(iii)] \textbf{Unconditional Distributionally Robust Optimization Model (UDRO)}: This model solves the unconditional distributionally robust optimization problem, using the empirical distribution of $Y$ as the reference measure. Specifically, this model is presented as 
\begin{align*}
    \min_{\alpha \in \mathcal{A}} \sup_{\mathbb{Q}_{Y} \in \mathbb{B}_{\delta_0}\left(\widehat{\mathbb{P}}_{N}\right)} \rho_{\mathbb{Q}_{Y}}\left[\ell(Y, \alpha)\right].
\end{align*}
    \item[(iv)] \textbf{Conditional Sample Average Approximation (CSAA)}: This model is defined as
\begin{align*}
\min_{\alpha \in \mathcal{A}} \rho_{\widehat{\mathbb{P}}_{K_N}}^h\left[\ell(Y^\top \alpha)\right],
\end{align*}
where $\widehat{\mathbb{P}}_{K_N}$ denotes the empirical conditional distribution of $Y$ on $X\in \mathcal{N}_\gamma(x_0)$. Specifically, $K_N$ represents the number of samples whose associated covariates fall within the neighborhood $\mathcal{N}_\gamma(x_0)$. Sorting $(\widehat{x}_{i},\widehat{y}_{i})$ as $(\widehat{x}_{(j)},\widehat{y}_{(j)})$ such that $\widehat{x}_{(j)} \in \mathcal{N}_\gamma(x_0)$, for $j = 1, \dots, K_N$ and $\widehat{x}_{(j)} \notin \mathcal{N}_\gamma(x_0)$, for $j = K_N+1, \dots, N$. Then $\widehat{\mathbb{P}}_{K_N} = \frac{1}{K_N} \sum_{j=1}^{K_N} \mathrm{I}_{\widehat{y}_{(j)}}$.

    \item[(v)] \textbf{Conditional Distributionally Robust Optimization Model (CDRO)}: This model solves the distributionally robust counterpart of CSAA, using the empirical conditional distribution $\mathbb{P}_{K_N}$ as the reference distribution. Specifically, this model is presented as 
\begin{align*}
    \min_{\alpha \in \mathcal{A}} \sup_{\mathbb{Q}|_{Y} \in \mathbb{B}_{\delta_0}\left(\widehat{\mathbb{P}}_{K_N}\right)} \rho_{\mathbb{Q}|_{Y}}\left[\ell(Y, \alpha)\right].
\end{align*}
\end{itemize}

In the experiment, we run 50 times for each sample sizes $N$ , choose $\delta=\delta_{min}+\frac{1}{10}\frac{\sqrt{ln 100}}{N}$ to make sure that the model (i) is feasible, and choose $\delta_0=100/N$ such that the radius $\delta\rightarrow 0$ as $N\rightarrow \infty$. For portfolio allocation problem setting, $\alpha\in\mathbb{R}^6_+$ denotes decision vector and $Y\in\mathbb{R}^6$ denotes an uncertain quantity of interest with side information $X\in\mathbb{R}^3$. We assume the joint distribution of $(X,Y)$ is as follows. The marginal distribution of $X$ is $X_1 \rightsquigarrow \mathcal{N}(0.1,10),~X_2 \rightsquigarrow \mathcal{N}(0,1)$ and $X_3 \rightsquigarrow \mathcal{N}(0,1)$. Let $\mu^\top=
(86.8625,71.6059,75.3759,97.6258,52.7854,84.8973)$,
$$L=\begin{bmatrix}
136.687 & 0 & 0 & 0 & 0 & 0\\
8.79766 & 142.279 & 0 & 0 & 0 & 0\\
16.1504 & 15.0637 & 122.613 & 0 & 0 & 0\\
18.4944 & 15.6961 & 26.344 & 139.148 & 0 & 0\\
3.41394 & 16.5922 & 14.8795 & 13.9914 & 151.732 & 0\\
24.8156 & 18.7292 & 17.1574 & 6.36536 & 24.7703 & 144.672
\end{bmatrix},$$
which is the same as the setting in Section 5.2 of \cite{BV17}. Moreover, let $\text{gate}(x_2)=\frac{1}{1+\exp\{-4(x_2-\tfrac{1}{2})\}}$, $s\left(x_2\right)=\left(1,1-\frac{1}{2} \operatorname{gate}\left(x_2\right), 1+6 \operatorname{gate}\left(x_2\right), 1,1+6 \operatorname{gate}\left(x_2\right), 1\right),$
and $D\left(x_2\right)=\operatorname{diag}\left(s\left(x_2\right)\right)$.
 The conditional distribution of $Y$ given $X=(x_1,x_2,x_3)$ is defined as 
 $$Y \mid X=x \rightsquigarrow \mathcal{N}\left(m(x), \Sigma\left(x_2\right)\right),$$
where $$m(x)=\mu+x_1 v_{1}+x_2 v_{2}+\text{tanh}(x_3) v_3,~\text{and }\Sigma\left(x_2\right)=L D\left(x_2\right)^2 L^{\top}$$ with ${v}_1=(30, -40, 0, 15, -10,  5)^\top$, ${v}_2=(-400, 60, 70, 80, 90, 100)^\top$, and  ${v}_3=(0,  20, -25, 0,  15, -12)^\top$. And we choose the conditional set $\mathcal{N}=\{x: \|x-x_0\|_1\le 1.2\}$, where $x_0=(0.1, 0.1, 0)$.

Figures \ref{comparison_figure1} and \ref{comparison_figure2} report out-of-sample performance for the five models introduced earlier (UB-CDRO, SAA, UDRO, CSAA, CDRO) under various convex distortion functions and Lipschitz-continuous convex loss functions. The plots summarize the distribution of out-of-sample risk values; the shaded regions denote the interquartile range (IQR, 15th-85th percentiles), and the bold colored curves trace the sample means across sample sizes $N$. The figure indicates that for various distortion functions and loss functions, UB-CDRO consistently outperforms all conditional and unconditional benchmarks: its out-of-sample performance is more stable (smaller IQR), exhibits a lower mean, and approaches the oracle optimal value more rapidly as $N$ increases. Although the conditional methods (CSAA, CDRO) converge to the conditional optimum, the bands of conditional methods remain wide even at $N=400$, indicating persistent variability. Moreover, the unconditional methods (SAA, UDRO) do not converge to that optimum as $N$ grows. This is expected because the conditional distribution can differ substantially from the marginal distribution; consequently, the decision $\alpha$ that is optimal for the marginal distribution generally differs from the conditional-optimal decision $\alpha^{*}$. In fact, SAA and UDRO converge to the marginal-optimal decision rather than to $\alpha^{*}$. Thus, unconditional methods (SAA, UDRO) exhibit strong stability but poor convergence, whereas conditional methods (CSAA, CDRO) achieve fast convergence but lack stability. By contrast, UB-CDRO is strong in both stability and convergence, surpassing all conditional and unconditional benchmarks.

\begin{figure}[htbp]
  \centering
  \captionsetup[subfigure]{labelformat=parens, labelsep=space}

  \begin{subfigure}{\linewidth}
    \centering
    \includegraphics[width=.9\linewidth]{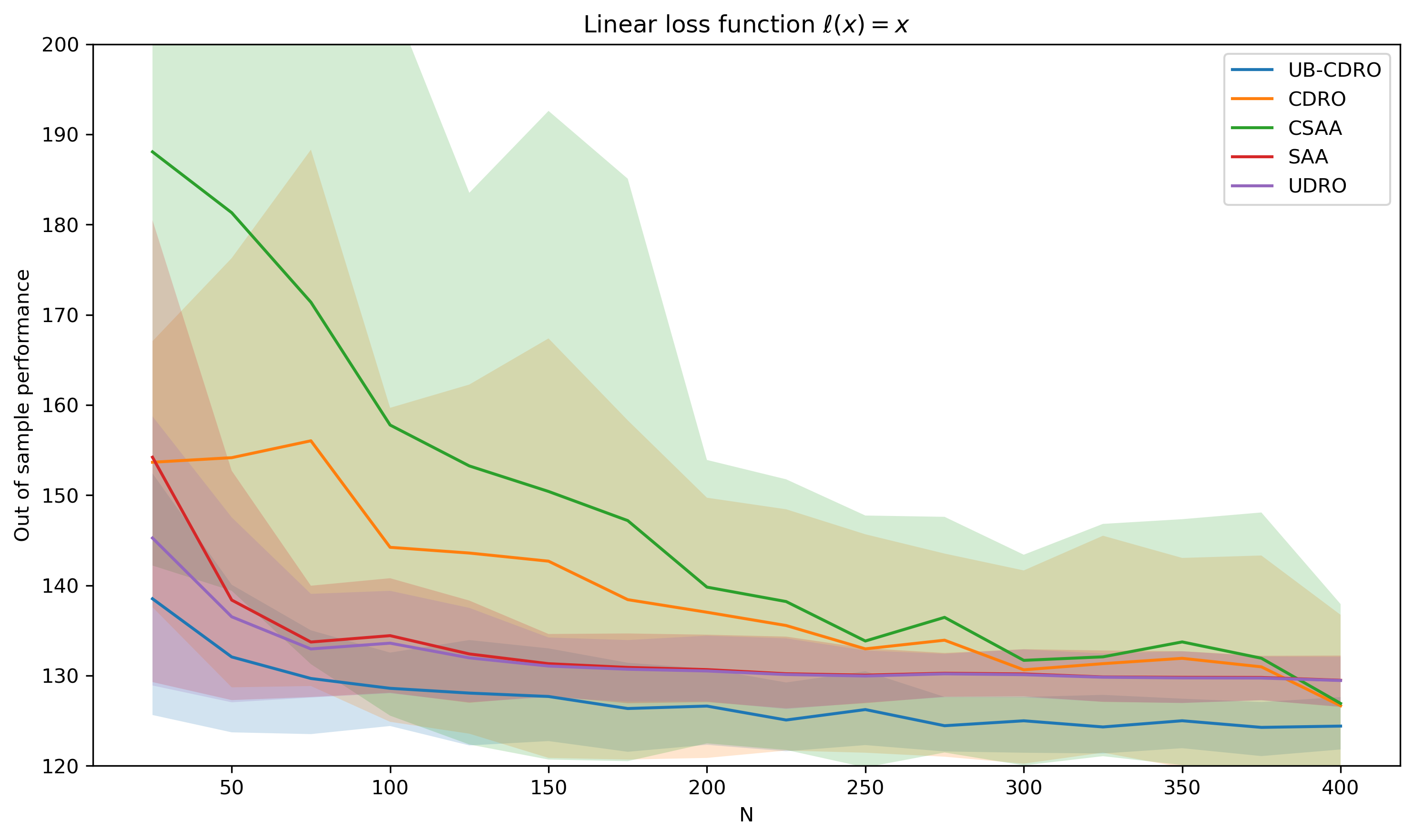}
    \caption{ }\label{fig:one-a}
  \end{subfigure}\vspace{0.6em}

  \begin{subfigure}{\linewidth}
    \centering
    \includegraphics[width=.9\linewidth]{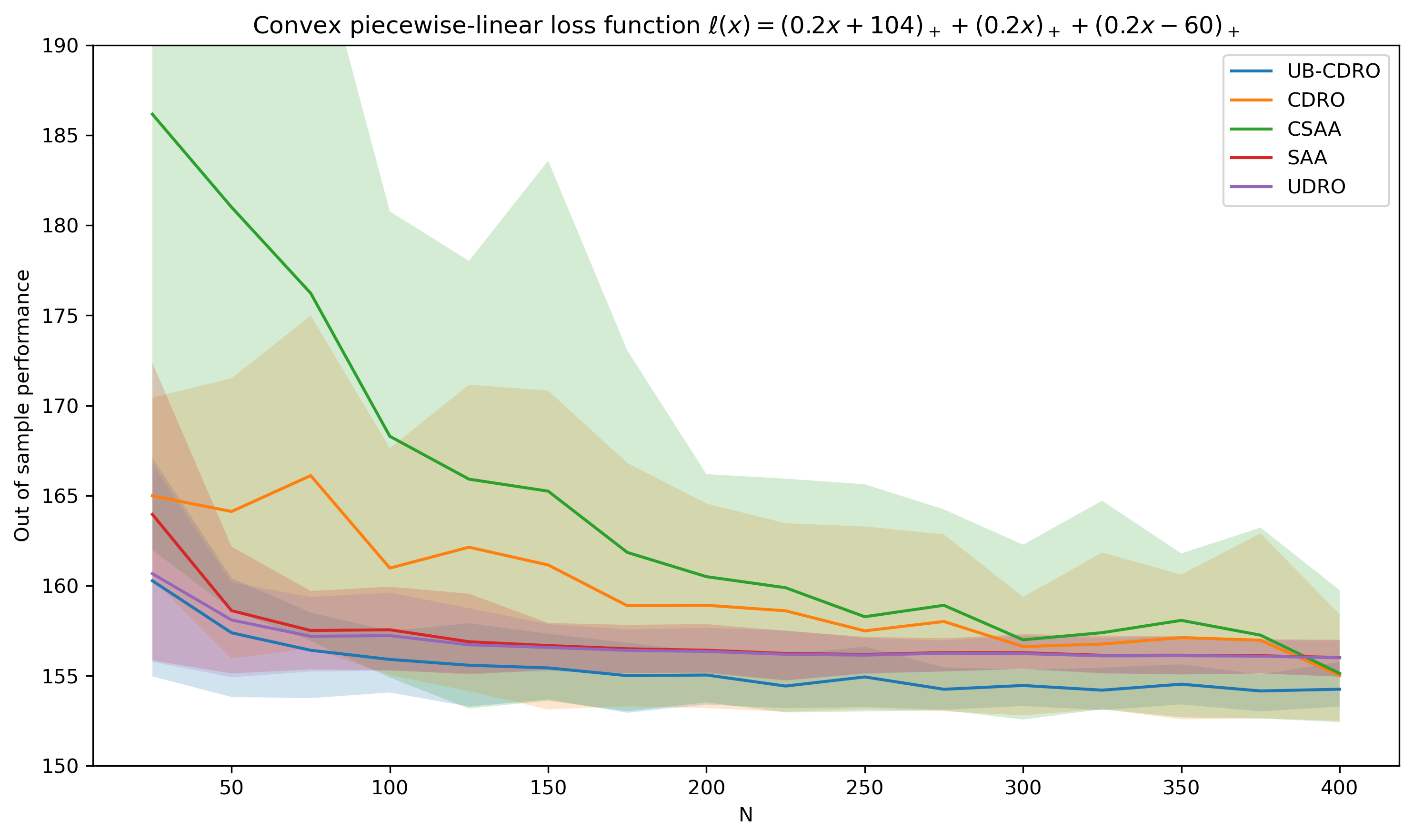}
    \caption{ }\label{fig:one-b}
  \end{subfigure}\vspace{0.6em}

  \caption{Comparison of Out-of-Sample Performance Across Models with Distortion Function $h(x)=x^2$}
  \label{comparison_figure1}
\end{figure}

\begin{figure}[htbp]
  \centering
  \captionsetup[subfigure]{labelformat=parens, labelsep=space}

  \begin{subfigure}{\linewidth}
    \centering
    \includegraphics[width=.9\linewidth]{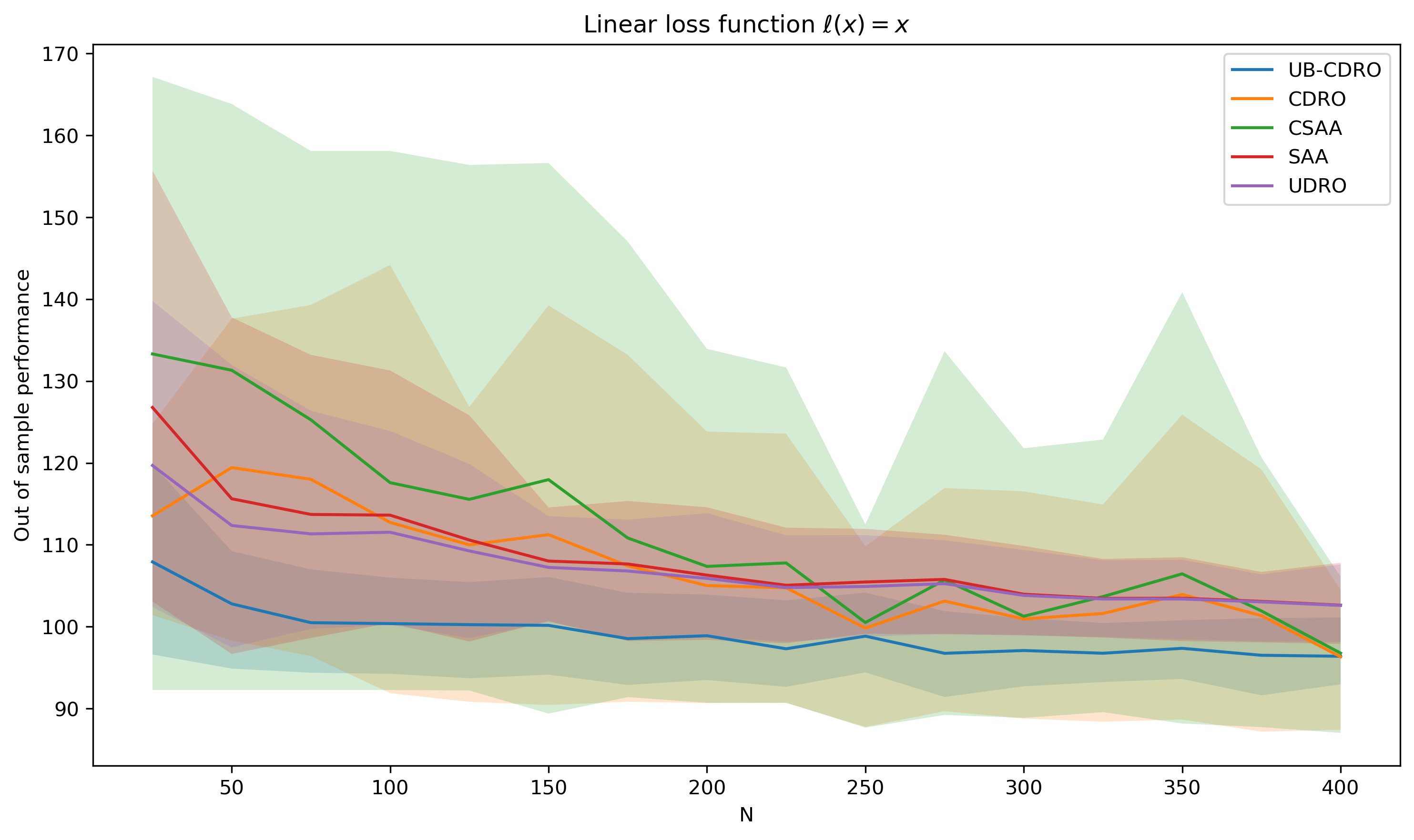}
    \caption{ }\label{fig:one-c}
  \end{subfigure}\vspace{0.6em}

  \begin{subfigure}{\linewidth}
    \centering
    \includegraphics[width=.9\linewidth]{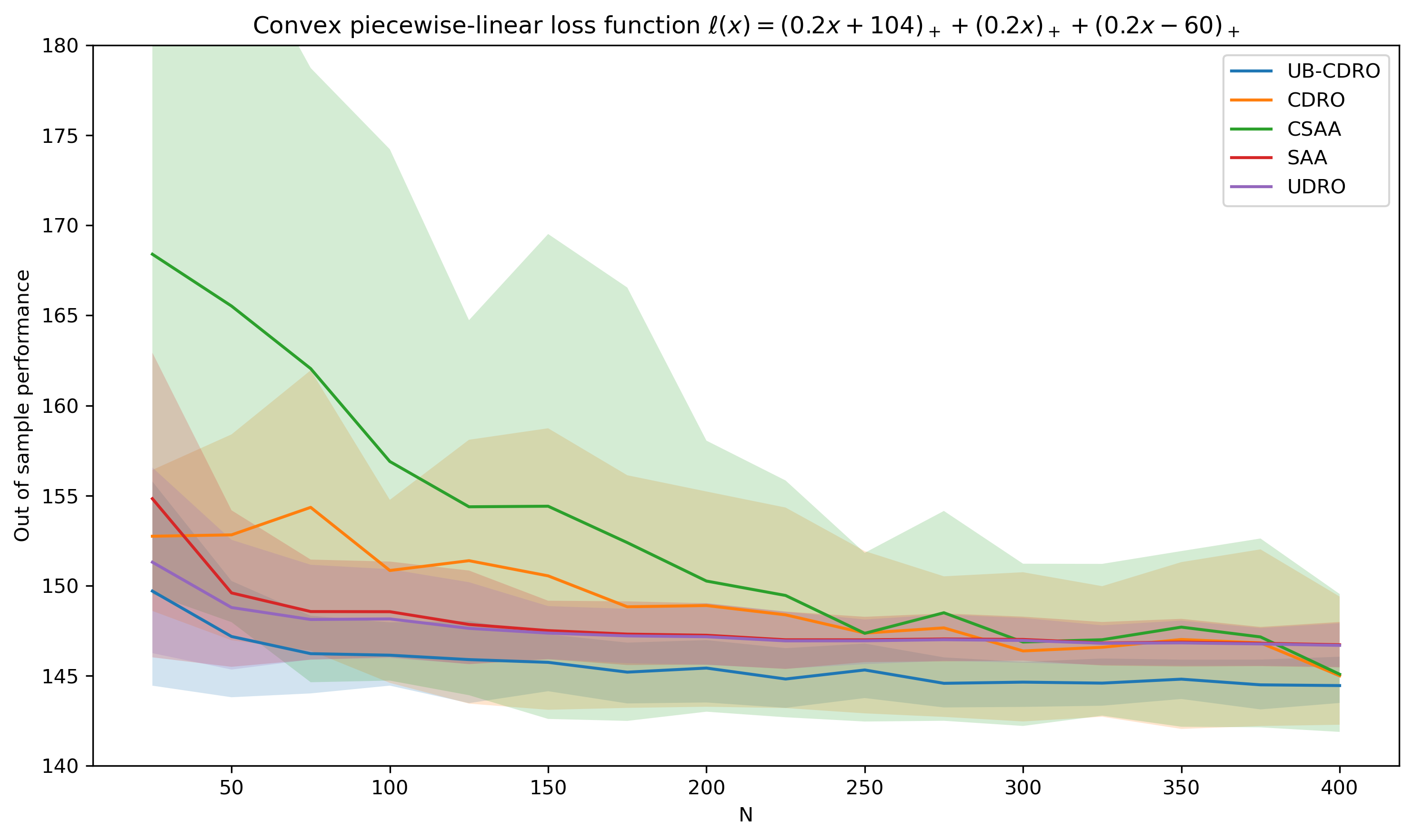}
    \caption{ }\label{fig:one-d}
  \end{subfigure}

  \caption{Comparison of Out-of-Sample Performance Across Models with Distortion Function $h(x)=\frac{e^x-1}{e-1}$}
  \label{comparison_figure2}
\end{figure}


\section*{Supplement: Technical Appendices}
\subsection*{EC.1. Regularization Reformulations for Wasserstein DRO \citep{W22}}
\label{ec:reformulations}
In this section, we first review some regularization results in \cite{W22}, which will be used throughout the proofs of Section 3. Let $L^q$ be the set of all random variables with a finite $q$th moment For $q \geqslant 1$, and $L^{\infty}$ be the set of all bounded random variables, under an arbitrary norm of $\mathcal{Y}$, $\|\cdot\|_\mathcal{Y}$. If the cost function on $\mathcal{Y}$ is $\mathbb{D}_{\mathcal{Y}}(y_1,y_2)=||y_1-y_2||_{\mathcal{Y}}$ for any $y_1,y_2\in\mathcal{Y}$, and $\ell$ is convex and Lipschitz continuous, the unconditional worst-case risk problem \eqref{r2} is equivalent to 
\begin{align}\label{regular:p=1}
	\sum_{i=1}^N p_i \ell(\widehat{y}_i^\top\alpha) + \operatorname{Lip}(\ell)\delta ||\alpha||_*,
\end{align}
where $||\cdot||_*$ is the dual norm of $||\cdot||_{\mathcal{Y}}$, and $\operatorname{Lip}(\ell)$ is the Lipschitz constant of $\ell$.  Moreover, for $\mathbb{D}_{\mathcal{Y}}(y_1,y_2)=||y_1-y_2||_{\mathcal{Y}}^q$ for any $y_1,y_2\in\mathcal{Y}$, with $q\in(1,\infty]$, \cite{W22} have some regularization results for the unconditional worst-case risk problem \eqref{r2} with some special loss functions. 
\begin{theorem}\label{W22thm3}(\cite{W22} Theorem 3)
Let $\ell: \mathbb{R} \rightarrow \mathbb{R}$ be a convex function. For $q \in(1, \infty]$, suppose that $\mathbb{E}[|\ell(Z)|]<\infty$ for all $Z \in L^q$, and $\mathbb{D}_{\mathcal{Y}}(y_1,y_2)=||y_1-y_2||_{\mathcal{Y}}^q$, for any $y_1,y_2\in\mathcal{Y}$. Then the following statements are equivalent.
\begin{itemize}
	\item [(i)]There exists $C>0$ such that for any $(p,\delta)\in \mathcal{V}$ with $\delta\ge 0$, it holds that
	$$
	\sup _{{\mathbb{Q}}|_{Y} \in \mathbb{B}_{\delta}(\sum_{i=1}^Np_i\mathrm{I}_{\widehat{y}_i})} \E_{{\mathbb{Q}}|_{Y}}\left[\ell(Y^\top \alpha) \right]=\sum_{i=1}^N p_i \ell(\widehat{y}_i^\top\alpha) +C\delta^\frac{1}{q} ||\alpha||_*
	$$
	\item[(ii)] The function $\ell$ takes one of the following two forms, multiplied by $C$ :
	\begin{itemize}
		\item [(a)]$\ell_1(x)=x+b$ or $\ell_1(x)=-x+b$ with some $b \in \mathbb{R}$;
		\item[(b)]  $\ell_2(x)=\left|x-b_1\right|+b_2$ with some $b_1, b_2 \in \mathbb{R}$.
	\end{itemize}
\end{itemize}
\end{theorem}
\begin{theorem} \label{W22thm4}(\cite{W22} Theorem 4) 
Let $\ell: \mathbb{R} \rightarrow \mathbb{R}_{+}$ be a Lipschitz continuous and convex function. For any $q \in(1, \infty)$, assume $\mathbb{D}_{\mathcal{Y}}(y_1,y_2)=||y_1-y_2||_{\mathcal{Y}}^q$, for any $y_1,y_2\in\mathcal{Y}$. Then the following statements are equivalent.
\begin{itemize}
	\item [(i)]There exists $C>0$ such that for any $(p,\delta)\in \mathcal{V}$ with $\delta\ge 0$, it holds that
	$$
	\sup _{{\mathbb{Q}}|_{Y} \in \mathbb{B}_{\delta}(\sum_{i=1}^Np_i\mathrm{I}_{\widehat{y}_i})} \E_{{\mathbb{Q}}|_{Y}}\left[\ell^q(Y^\top\alpha) \right] = \left[\left(\sum_{i=1}^N p_i \ell^q(\widehat{y}_i^\top\alpha) \right)^{1/q} + C\delta^\frac{1}{q} ||\alpha||_*\right]^{q}.
	$$
	\item[(ii)] The function $\ell$ takes one of the following four forms, multiplied by $C$:
	\begin{itemize}
		\item [(a)]$\ell_1(x)=(x-b)_{+}$with some $b \in \mathbb{R}$;
		\item[(b)] $\ell_2(x)=(x-b)_{-}$with some $b \in \mathbb{R}$;
		\item[(c)] $\ell_3(x)=\left(\left|x-b_1\right|-b_2\right)_{+}$with some $b_1 \in \mathbb{R}$ and $b_2 \geqslant 0$;
		\item[(d)] $\ell_4(x)=\left|x-b_1\right|+b_2$ with some $b_1 \in \mathbb{R}$ and $b_2>0$.
	\end{itemize}
\end{itemize}
\end{theorem}

\begin{theorem} (\cite{W22} Corollary 1)
	For any $q \in[1, \infty)$ and $C>0$, let $\rho^{(1)}$ be defined by \eqref{rho_1} with $\ell(z,t)=(C\bar{\ell})^q(z,t)$ and $\bar{\ell}(z,t)=C\bar{\ell}(z-t)$, and $\rho^{(2)}$ be defined by \eqref{rho_2} with ${\ell}(z,t)={\ell}(z-t)$, and assume $\mathbb{D}_{\mathcal{Y}}(y_1,y_2)=||y_1-y_2||_{\mathcal{Y}}^q$, for any $y_1,y_2\in\mathcal{Y}$. Take $(p,\delta)\in \mathcal{V}$ with $\delta\ge 0$.
	\begin{itemize}
		\item [(i)]Let the function $\bar{\ell}$ take one of the following two forms:
		\begin{itemize}
			\item[(a)] $\bar{\ell}_1(x)=\left(\left|x-b_1\right|-b_2\right)_{+}$with some $b_1 \in \mathbb{R}$ and $b_2 \geqslant 0$;
			\item[(b)] $\bar{\ell}_2(x)=\left|x-b_1\right|+b_2$ with some $b_1 \in \mathbb{R}$ and $b_2>0$.
		\end{itemize} Then 
		$$\sup _{{\mathbb{Q}}|_{Y} \in \mathbb{B}_{\delta}(\sum_{i=1}^Np_i\mathrm{I}_{\widehat{y}_i})} \rho^{(1)}_{{\mathbb{Q}}|_{Y}}\left(Y^\top \alpha \right) =  \left[\left(\rho^{(1)}_{\sum_{i=1}^N p_i\mathrm{I}_{\widehat{y}_i}}\left(Y^\top \alpha \right)\right)^{1/q} + C\delta^{1/q} ||\alpha||_*\right]^q.$$
		\item[(ii)] For $C>1$, let the function $\ell$ take one of the following three forms:
		\begin{itemize}
			\item [(a)]$\ell_1(x)=(x-b)_{+}$with some $b \in \mathbb{R}$;
			\item[(b)] $\ell_2(x)=\left(\left|x-b_1\right|-b_2\right)_{+}$with some $b_1 \in \mathbb{R}$ and $b_2 \geqslant 0$;
			\item[(c)] $\ell_3(x)=\left|x-b_1\right|+b_2$ with some $b_1 \in \mathbb{R}$ and $b_2>0$,
		\end{itemize}  or $\ell(z, t)=C(|z|-t)_{+}$. It holds that
		$$\sup _{{\mathbb{Q}}|_{Y} \in \mathbb{B}_{\delta}(\sum_{i=1}^Np_i\mathrm{I}_{\widehat{y}_i})} \rho^{(2)}_{{\mathbb{Q}}|_{Y}}\left(Y^\top \alpha \right) =  \rho^{(2)}_{\sum_{i=1}^N p_i\mathrm{I}_{\widehat{y}_i}}\left(Y^\top \alpha \right) + C\delta^{1/q}||\alpha||_*.$$
	\end{itemize}
\end{theorem}
\begin{theorem} \label{W22thm5}(\cite{W22} Theorem 5) 
	For $q \in[1, \infty]$, let $h:[0,1] \rightarrow \mathbb{R}$ be an increasing and convex distortion function satisfying $\lim _{x \rightarrow 1-} h(x)=h(1)$ and $\left\|h_{-}^{\prime}\right\|_{q^*} \in(0, \infty)$,  let $\ell: \mathbb{R} \rightarrow \mathbb{R}$ be a convex function, and let  $\mathbb{D}_{\mathcal{Y}}(y_1,y_2)=||y_1-y_2||_{\mathcal{Y}}^q$, for any $y_1,y_2\in\mathcal{Y}$.
	\begin{itemize}
		\item[(i)] 	Let $q=1$. If $\operatorname{Lip}(\ell)<\infty$, then for any $(p,\delta)\in \mathcal{V}$ with $\delta\ge 0$, we have
		$$
	\sup _{{\mathbb{Q}}|_{Y} \in \mathbb{B}_{\delta}(\sum_{i=1}^Np_i\mathrm{I}_{\widehat{y}_i})} \rho_{{\mathbb{Q}}|_{Y}}^h\left(\ell\left(Y^{\top} \alpha\right)\right)=\rho_{\sum_{i=1}^Np_i\mathrm{I}_{\widehat{y}_i}}^{h}\left(\ell\left(Y^{\top} \alpha\right)\right)+\operatorname{Lip}(\ell)\left\|h_{-}^{\prime}\right\|_{\infty} \delta\|\alpha\|_*.
		$$
		\item[(ii)] Let $q \in(1, \infty]$. If $\rho_h(|\ell(Z)|)<\infty$ for all $Z \in L^q$, then the following statements are equivalent.
		\begin{itemize}
			\item[(1)] For any $(p,\delta)\in \mathcal{V}$ with $\delta\ge 0$, there exists $C>0$ such that
			$$
			\sup _{{\mathbb{Q}}|_{Y} \in \mathbb{B}_{\delta}(\sum_{i=1}^Np_i\mathrm{I}_{\widehat{y}_i})} \rho_{{\mathbb{Q}}|_{Y}}^h\left(\ell\left(Y^{\top} \alpha\right)\right)=\rho_{\sum_{i=1}^Np_i\mathrm{I}_{\widehat{y}_i}}^{h}\left(\ell\left(Y^{\top} \alpha\right)\right)+C\left\|h_{-}^{\prime}\right\|_{q^*} \delta^{1/q}\|\alpha\|_*.
			$$
			\item[(2)] The function $\ell$ takes one of the following two forms, multiplied by $C$ :
			\begin{itemize}
				\item[(a)] $\ell_1(x)=x+b$ or $\ell_1(x)=-x+b$ with some $b \in \mathbb{R}$;
				\item[(b)] $\ell_2(x)=\left|x-b_1\right|+b_2$ with some $b_1, b_2 \in \mathbb{R}$.
			\end{itemize}	 
		\end{itemize}
	\end{itemize}
\end{theorem}

\newpage
\subsection*{EC.2. Main Proofs}
\subsection*{EC.2.1. Proofs of Section \ref{main-result}}
\begin{proof}[{\bf Proof of Proposition \ref{feasible_con}}]
    To obtain the sufficient and necessary feasible condition of problem \eqref{m1}, we only need to find the sufficient and necessary feasible condition of problem \eqref{innerproblem*} and thus only need to find  the sufficient and necessary condition such that the ambiguity set $\mathcal{K}({\mathcal{N}_\gamma\left(x_0\right),\mathcal{N}_{\omega},\mathbb{B}_{\delta_0}}(\widehat{\mathbb{P}}))$ is not empty. We consider the following three cases: (i) $\frac{N-m}{N}>\omega_2$; (ii) $\frac{N-m}{N}\in[\omega_1,\omega_2]=\mathcal{N}_\omega$; and (iii) $\frac{N-m}{N}<\omega_1$. Before the discussion case by case, we first note that 
    \begin{align*}
        \mathcal{K}({\mathcal{N}_\gamma\left(x_0\right),\mathcal{N}_{\omega},\mathbb{B}_{\delta_0}}(\widehat{\mathbb{P}}))=\left(\mathop{\bigcup}\limits_{(p, \delta) \in \mathcal{V}_1} \mathbb{B}_{\delta}\left(\sum_{i=1}^N p_i \mathrm{I}_{\widehat{y}_i}\right)\right)\mathop{\bigcup}\left(\mathop{\bigcup}\limits_{(p, \delta) \in \mathcal{V}_2} \mathbb{B}_{\delta}^{\circ}\left(\sum_{i=1}^N p_i \mathrm{I}_{\widehat{y}_i}\right)\right):=\mathcal{U_B}
    \end{align*}
    from the proof of Theorem \ref{thm1}, where $\mathcal{V}_1=\{(p,\delta):\exists \epsilon \text{~such that }(p,\delta,\epsilon)\in \mathcal{V}_0, p_i=\frac{\epsilon}{N}, \text{ for }i=m+1,..,N\}$ and $\mathcal{V}_2=\{(p,\delta):\exists \epsilon \text{~such that }(p,\delta,\epsilon)\in \mathcal{V}_0, \exists i\in\{m+1,..,N\}\text{~such that }p_i<\frac{\epsilon}{N}\}$. Thus,  $\mathcal{K}({\mathcal{N}_\gamma\left(x_0\right),\mathcal{N}_{\omega},\mathbb{B}_{\delta_0}}(\widehat{\mathbb{P}}))$ is not empty is equivalent to $\mathcal{U_B}$ is not empty.
    \begin{itemize}
        \item[(i)] For $\frac{N-m}{N}>\omega_2$, one can verify that $\mathcal{V}_1=\emptyset$ due to that for every $(p,\delta)\in\mathcal{V}_1$, we have $\sum_{i=m+1}^Np_i=\frac{N-m}{N}\epsilon>1$, which is contradictory to the constraint $\sum_{i=1}^Np_i=1$. Then, we have 
    \begin{align*}
        \mathcal{K}({\mathcal{N}_\gamma\left(x_0\right),\mathcal{N}_{\omega},\mathbb{B}_{\delta_0}}(\widehat{\mathbb{P}}))=\mathop{\bigcup}\limits_{(p, \delta) \in \mathcal{V}_2} \mathbb{B}_{\delta}^{\circ}\left(\sum_{i=1}^N p_i \mathrm{I}_{\widehat{y}_i}\right)=\mathop{\bigcup}\limits_{(p, \delta) \in \mathcal{V}} \mathbb{B}_{\delta}^{\circ}\left(\sum_{i=1}^N p_i \mathrm{I}_{\widehat{y}_i}\right),
    \end{align*}
    where the second equality follows from that $\mathcal{V}=\mathcal{V}_1\cup\mathcal{V}_2$ and $\mathcal{V}_1=\emptyset$. Thus, $\mathcal{K}({\mathcal{N}_\gamma\left(x_0\right),\mathcal{N}_{\omega},\mathbb{B}_{\delta_0}}(\widehat{\mathbb{P}}))$ is not empty if and only if there exists $(p,\delta)\in\mathcal{V}$ such that $\delta>0$, and thus if and only if there exists $(p,\delta,\epsilon)\in\mathcal{V}_0$ such that $\delta = \epsilon\left(\delta_0-\frac{1}{N}\sum_{i=m+1}^Nd_i\right)-\sum_{i=1}^{m}p_id_i+\sum_{i=m+1}^Np_id_i>0$. One can verify that there exists $(p,\delta,\epsilon)\in\mathcal{V}_0$ such that $\delta>0$ is equivalent to
    \begin{align*}
        \max_{\substack{ p_i\in\left[0,\frac{\epsilon}{N}\right], i=1,..., N, \\\sum_{i=1}^Np_i=1,\epsilon\in [\frac{1}{\omega_2},\frac{1}{\omega_1}]}}\epsilon\left(\delta_0-\frac{1}{N}\sum_{i=m+1}^Nd_i\right)-\sum_{i=1}^{m}p_id_i+\sum_{i=m+1}^Np_id_i>0,
    \end{align*}
    and thus equivalent to
    \begin{align*}
        \delta_0>&\min_{\substack{ p_i\in\left[0,\frac{\epsilon}{N}\right], i=1,..., N, \\\sum_{i=1}^Np_i=1,\epsilon\in [\frac{1}{\omega_2},\frac{1}{\omega_1}]}}\frac{1}{N}\left(\sum_{i=1}^{m}\frac{Np_i}{\epsilon}d_i+\sum_{i=m+1}^N\left(1-\frac{Np_i}{\epsilon}\right)d_i\right)\\
        =&\min_{\substack{v_i\in[0,1],i=1,...,N, \\ 1/N\sum_{i=1}^Nv_i\in[\omega_1,\omega_2]}}\frac{1}{N}\left(\sum_{i=1}^mv_id_i+\sum_{i=m+1}^N(1-v_i)d_i\right)\\
        =&\delta_{min}.
    \end{align*}
    We conclude that in this case problem \eqref{m1} is feasible if and only if $\delta_0>\delta_{min}$. 
    \item[(ii)] For $\frac{N-m}{N}\in[\omega_1,\omega_2]=\mathcal{N}_\omega$, we find that $\widehat{\mathbb{P}}\left(X \in \mathcal{N}_\gamma\left(x_0\right)\right) \in[\omega_1,\omega_2]=\mathcal{N}_{\omega}$ and then $\widehat{\mathbb{P}}|_Y\in \mathcal{K}({\mathcal{N}_\gamma\left(x_0\right),\mathcal{N}_{\omega},\mathbb{B}_{\delta_0}}(\widehat{\mathbb{P}}))$. Thus, $\mathcal{K}({\mathcal{N}_\gamma\left(x_0\right),\mathcal{N}_{\omega},\mathbb{B}_{\delta_0}}(\widehat{\mathbb{P}}))$ is not empty if and only if $\delta_0\ge 0$. Moreover, $\min_{v_i\in[0,1],i=1,...,N}\frac{1}{N}\left(\sum_{i=1}^mv_id_i+\sum_{i=m+1}^N(1-v_i)d_i\right)\ge 0$ and can attain lower bound $0$ in $v_i=0$ for $i=1,...,m$, and $v_i=1$ for $i=m+1,...,N$, and we find that $\sum_{i=1}^Nv_i=\frac{N-m}{N}\in[\omega_1,\omega_2]$. Thus, $$\delta_{min}=\min_{\substack{v_i\in[0,1],i=1,...,N, \\ 1/N\sum_{i=1}^Nv_i\in[\omega_1,\omega_2]}}\frac{1}{N}\left(\sum_{i=1}^mv_id_i+\sum_{i=m+1}^N(1-v_i)d_i\right)=0.$$
    From above discussion, we conclude that in this case problem \eqref{m1} is feasible if and only if $\delta_0\ge\delta_{min}$. 
    \item[(iii)] For $\frac{N-m}{N}<\omega_1$, if $\delta_0\ge \delta_{min}$, we consider in the two cases: (iiia) $\delta_0>\delta_{min}$, and (iiib) $\delta_0=\delta_{min}$.
    \begin{itemize}
        \item[(iiia)] If $\delta_0>\delta_{min}$, there exist $\{v_i\}_{i=1}^N$ satisfies that $v_i\in[0,1],i=1,...,N, 1/N\sum_{i=1}^Nv_i\in[\omega_1,\omega_2]$ such that $\delta_0>\frac{1}{N}\left(\sum_{i=1}^mv_id_i+\sum_{i=m+1}^N(1-v_i)d_i\right)$. Given $\epsilon\in[\frac{1}{\omega_2},\frac{1}{\omega_1}]$, let  $p_i=\frac{\epsilon v_i}{N}\in[0,\frac{\epsilon}{N}]$, and $\delta=\epsilon\left(\delta_0-\frac{1}{N}\sum_{i=m+1}^Nd_i\right)-\sum_{i=1}^{m}p_id_i+\sum_{i=m+1}^Np_id_i$. Then $$\delta=\epsilon\left[\delta_0-\frac{1}{N}\left(\sum_{i=1}^mv_id_i+\sum_{i=m+1}^N(1-v_i)d_i\right)\right]>0,$$ and thus there exists $(p,\delta)\in\mathcal{V}$ such that $\delta>0$. As a result, there exists $(p,\delta)\in\mathcal{V}$ such that $\mathbb{B}_{\delta}^{\circ}\left(\sum_{i=1}^N p_i \mathrm{I}_{\widehat{y}_i}\right)$ is not empty. From the construction of $\mathcal{K}({\mathcal{N}_\gamma\left(x_0\right),\mathcal{N}_{\omega},\mathbb{B}_{\delta_0}}(\widehat{\mathbb{P}}))$, we conclude that $\mathcal{K}({\mathcal{N}_\gamma\left(x_0\right),\mathcal{N}_{\omega},\mathbb{B}_{\delta_0}}(\widehat{\mathbb{P}}))$ is not empty.
        \item[(iiib)]  If $\delta_0=\delta_{min}$, from the definition of $\delta_{min}$, we find that there exist $\{v_i^*\}_{i=1}^N$ satisfies that $v_i^*\in[0,1],i=1,...,N, 1/N\sum_{i=1}^Nv_i^*\in[\omega_1,\omega_2]$ such that  $$\delta_0=\delta_{min}=\frac{1}{N}\left(\sum_{i=1}^mv_i^*d_i+\sum_{i=m+1}^N(1-v_i^*)d_i\right)$$ due to that $\delta_{min}$ is a linear problem over a closed and convex set. We claim that $\{v_i^*\}_{i=1}^N$ must satisfy that $v_i^*=1$ for $i=m+1,..,N$. Otherwise, without loss of generality, we assume $v_i^*=1$ for $i=m+1,...,k$ and $v_i^*<1$ for $i=k+1,...,N$. Let $v_i=v_i^*-\Delta_i\ge 0$ for $i=1,..,m$ and $v_i=1$ for $i=m+1,...,N$, where $\Delta_i\ge 0$ and $\sum_{i=1}^m\Delta_i=\sum_{i=k+1}^N(1-v_i^*)$. We find $\{v_i\}_{i=1}^N$ satisfies that $v_i\in[0,1],i=1,...,N,  1/N\sum_{i=1}^Nv_i\in[\omega_1,\omega_2]$ and such $\{v_i\}_{i=1}^N$ must exists due to that $\frac{N-m}{N}<\omega_1\le \omega_2$. Thus we have $\sum_{i=1}^{m}v_id_i+\sum_{i=m+1}^N(1-v_i)d_i<\sum_{i=1}^{m}v_i^*d_i+\sum_{i=m+1}^N(1-v_i^*)d_i=\delta_{min}$ and the strict inequality contradicts the definition of $\delta_{min}$. As a result, the claim holds. Given $\epsilon\in[\frac{1}{\omega_2},\frac{1}{\omega_1}]$, using the same construction on (iiia), we find that there exists $(p,\delta)\in\mathcal{V}$ such that $\delta=0$ with $p_i=\frac{\epsilon}{N}$ for $i=m+1,...,N$. Thus, $(p,\delta)\in\mathcal{V}_1$. As a result, there exists $(p,\delta)\in\mathcal{V}_1$ such that $\mathbb{B}_{\delta}\left(\sum_{i=1}^N p_i \mathrm{I}_{\widehat{y}_i}\right)=\{\sum_{i=1}^N p_i \mathrm{I}_{\widehat{y}_i}\}$ is not empty. From the construction of $\mathcal{K}({\mathcal{N}_\gamma\left(x_0\right),\mathcal{N}_{\omega},\mathbb{B}_{\delta_0}}(\widehat{\mathbb{P}}))$, we conclude that $\mathcal{K}({\mathcal{N}_\gamma\left(x_0\right),\mathcal{N}_{\omega},\mathbb{B}_{\delta_0}}(\widehat{\mathbb{P}}))$ is not empty.
    \end{itemize}
    Combining the above two cases, we find that if $\delta_0\ge \delta_{min}$, $\mathcal{K}({\mathcal{N}_\gamma\left(x_0\right),\mathcal{N}_{\omega},\mathbb{B}_{\delta_0}}(\widehat{\mathbb{P}}))$ is not empty, and thus problem \eqref{m1} is feasible. For the ``only if" part, if $\delta_0<\delta_{min}$, one can verify that for any $(p,\delta)\in\mathcal{V}$,
    \begin{align*}
        \delta=&\epsilon\left(\delta_0-\frac{1}{N}\sum_{i=m+1}^Nd_i\right)-\sum_{i=1}^{m}p_id_i+\sum_{i=m+1}^Np_id_i\\
        <&\epsilon\left(\delta_{min}-\frac{1}{N}\sum_{i=m+1}^Nd_i\right)-\sum_{i=1}^{m}p_id_i+\sum_{i=m+1}^Np_id_i\\
        =&\min_{\substack{v_i\in[0,1],i=1,...,N, \\ 1/N\sum_{i=1}^Nv_i\in[\omega_1,\omega_2]}}\sum_{i=1}^m\left(\frac{\epsilon v_i}{N}-p_i\right)d_i+\sum_{i=m+1}^N\left(p-\frac{\epsilon v_i}{N}\right)d_i\\
        \le&0,
    \end{align*}
    where the last inequality follows from that there exists $\{v_i^0\}_{i=1}^N=\{Np_i/\epsilon\}_{i=1}^N$ satisfies $v_i^0\in[0,1],i=1,...,N,$ and  $1/N\sum_{i=1}^Nv_i^0\in[\omega_1,\omega_2]$ such that $\sum_{i=1}^m\left(\frac{\epsilon v_i^0}{N}-p_i\right)d_i+\sum_{i=m+1}^N\left(p-\frac{\epsilon v_i^0}{N}\right)d_i=0$. From the construction of $\mathcal{K}({\mathcal{N}_\gamma\left(x_0\right),\mathcal{N}_{\omega},\mathbb{B}_{\delta_0}}(\widehat{\mathbb{P}}))$, we conclude that $\mathcal{K}({\mathcal{N}_\gamma\left(x_0\right),\mathcal{N}_{\omega},\mathbb{B}_{\delta_0}}(\widehat{\mathbb{P}}))$ is empty. From above discussion, we conclude that in this case problem \eqref{m1} is feasible if and only if $\delta_0\ge\delta_{min}$. 
    \end{itemize}
    \qedhere
\end{proof}

\begin{proof}[{\bf Proof of Theorem \ref{thm1}}]
	Let $\mathcal{V}_1=\{(p,\delta):\exists \epsilon \text{~such that }(p,\delta,\epsilon)\in \mathcal{V}_0, p_i=\frac{\epsilon}{N}, \text{ for }i=m+1,..,N\}$ and $\mathcal{V}_2=\{(p,\delta):\exists \epsilon \text{~such that }(p,\delta,\epsilon)\in \mathcal{V}_0, \exists i\in\{m+1,..,N\}\text{~such that }p_i<\frac{\epsilon}{N}\}$. Then $\mathcal{V}=\mathcal{V}_1\cup\mathcal{V}_2$. We first prove that 
	\begin{align*}
		\mathcal{K}({\mathcal{N}_\gamma\left(x_0\right),\mathcal{N}_{\omega},\mathbb{B}_{\delta_0}}(\widehat{\mathbb{P}}))= \left(\mathop{\bigcup}\limits_{(p, \delta) \in \mathcal{V}_1} \mathbb{B}_{\delta}\left(\sum_{i=1}^N p_i \mathrm{I}_{\widehat{y}_i}\right)\right)\mathop{\bigcup}\left(\mathop{\bigcup}\limits_{(p, \delta) \in \mathcal{V}_2} \mathbb{B}_{\delta}^{\circ}\left(\sum_{i=1}^N p_i \mathrm{I}_{\widehat{y}_i}\right)\right),
	\end{align*}
	where $\mathbb{B}_{\delta}\left(\sum_{i=1}^Np_i\mathrm{I}_{\widehat{y}_i}\right)$ is the ambiguity set on $\mathcal{Y}$, defined as \begin{align*}
		\mathbb{B}_{\delta}\left(\sum_{i=1}^Np_i\mathrm{I}_{\widehat{y}_i}\right)=\left\{\mathbb{Q}|_Y\in \mathcal{M}(\mathcal{Y}):\mathcal{W}\left(\sum_{i=1}^Np_i\mathrm{I}_{\widehat{y}_i},\mathbb{Q}|_Y\right)\le \delta\right\},
	\end{align*}
	and $\mathbb{B}_{\delta}^{\circ}\left(\sum_{i=1}^Np_i\mathrm{I}_{\widehat{y}_i}\right)$ is its interior point set, defined as
	\begin{align*}
		\mathbb{B}_{\delta}^{\circ}\left(\sum_{i=1}^Np_i\mathrm{I}_{\widehat{y}_i}\right)=\left\{\mathbb{Q}|_Y\in\mathcal{M}(\mathcal{Y}):\mathcal{W}\left(\sum_{i=1}^Np_i\mathrm{I}_{\widehat{y}_i},\mathbb{Q}|_Y\right)< \delta\right\}.
	\end{align*}
	To obtain the set equivalence, we first prove the following containment relation \begin{align*}
		\mathcal{K}({\mathcal{N}_\gamma\left(x_0\right),\mathcal{N}_{\omega},\mathbb{B}_{\delta_0}}(\widehat{\mathbb{P}}))\subset\left(\mathop{\bigcup}\limits_{(p, \delta) \in \mathcal{V}_1} \mathbb{B}_{\delta}\left(\sum_{i=1}^N p_i \mathrm{I}_{\widehat{y}_i}\right)\right)\mathop{\bigcup}\left(\mathop{\bigcup}\limits_{(p, \delta) \in \mathcal{V}_2} \mathbb{B}_{\delta}^{\circ}\left(\sum_{i=1}^N p_i \mathrm{I}_{\widehat{y}_i}\right)\right).
	\end{align*}
	For any $\mathbb{Q} \in \mathbb{B}_{\delta_0}(\widehat{\mathbb{P}}), \mathbb{Q}\left(X \in \mathcal{N}_\gamma\left(x_0\right)\right) \in\mathcal{N}_{\omega}$, assume that the optimal transport between $\widehat{\mathbb{P}}$ and $\mathbb{Q}$ is $\pi^*$. Then  $\mathbb{Q}$ have the following representation by support set and $\pi^*$,
	$$
	\mathbb{Q}=\sum_{i=1}^Np_i\mathbb{Q}_i^1+\sum_{i=1}^N\left(\frac{1}{N}-p_i\right)\mathbb{Q}_i^2,
	$$
	where $\text{supp }\mathbb{Q}_i^1\subset\{(x,y):x\in\mathcal{N}_\gamma\left(x_0\right)\}$, $\text{supp }\mathbb{Q}_i^2\subset\mathcal{X}\times\mathcal{Y}\backslash\{(x,y):x\in\mathcal{N}_\gamma\left(x_0\right)\}$, $\mathbb{Q}_i^1$ and $\mathbb{Q}_i^2$ are transferred from $\mathrm{I}_{(\widehat{x}_i,\widehat{y}_i)}$ under $\pi^*$ with probability $p_i$ and $\frac{1}{N}-p_i$, respectively, $p_i\in[0,\frac{1}{N}]$ and $\sum_{i=1}^Np_i\in\mathcal{N}_{\omega}$. Based on this decomposition of $\mathbb{Q}$, if the joint distribution of $(X,Y)$ is $\mathbb{Q}$, the conditional distribution of $Y$ given $X\in \mathcal{N}_\gamma(x_0)$, $\mathbb{Q}|_Y$, can be written as $\mathbb{Q}|_Y=\frac{\sum_{i=1}^Np_i\mathbb{Q}_i^1|_Y}{\sum_{i=1}^Np_i}$. Note that $\text{supp }\mathbb{Q}_i^{2}\subset \mathcal{X}\times \mathcal{Y}\backslash  \{(x,y):x \in \mathcal{N}_\gamma\left(x_0\right)\}$, $(\widehat{x}_i,\widehat{y}_i)\in\{(x,y):x \in \mathcal{N}_\gamma\left(x_0\right)\}$, $i=m+1,...,N$, and  $\widehat{x}_i^p=\arg\min_{x \in \partial\mathcal{N}_\gamma\left(x_0\right)} \mathbb{D}_{\mathcal{X}}\left(x, \widehat{x}_i\right)$.  Then $\int\mathbb{D}_{\mathcal{X}}(\widehat{x}_i,x)~d\mathbb{Q}_i^{2}(x,y)>d_i$ for $i=m+1,...,N$. Due to the continuity of $\mathbb{D}_{\mathcal{X}}$, there exists $\phi_\mathbb{Q}^i>0$ and $\psi_\mathbb{Q}^i>0,~i=m+1,...,N$ such that $\int\mathbb{D}_{\mathcal{X}}(\widehat{x}_i,x)~d\mathbb{Q}_i^{2}(x,y)\ge \mathbb{D}_{\mathcal{X}}(\widehat{x}_i,\widehat{x}_i^p+\phi^i_\mathbb{Q}{(\widehat{x}_i^p-\widehat{x}_i)})=d_i+\psi^i_\mathbb{Q}$. Let $\psi_\mathbb{Q}=\min_{i=m+1,...,N}\psi_\mathbb{Q}^i$. 
	One can verify that under the optimal transport $\pi^*$, we have
	$$
	\begin{aligned}
		\mathcal{W}(\widehat{\mathbb{P}},\mathbb{Q})=&\sum_{i=1}^mp_i\int \mathbb{D}_{\mathcal{X}}(\widehat{x}_i,x)+\mathbb{D}_{\mathcal{Y}}(\widehat{y}_i,y)~d\mathbb{Q}_i^{1}(x,y)\\
		&+\sum_{i=1}^m(\frac{1}{N}-p_i)\int \mathbb{D}_{\mathcal{X}}(\widehat{x}_i,x)+\mathbb{D}_{\mathcal{Y}}(\widehat{y}_i,y)~d\mathbb{Q}_i^{2}(x,y)\\
		&+\sum_{i=m+1}^Np_i\int \mathbb{D}_{\mathcal{X}}(\widehat{x}_i,x)+\mathbb{D}_{\mathcal{Y}}(\widehat{y}_i,y)~d\mathbb{Q}_i^{1}(x,y)\\
		&+\sum_{i=m+1}^N(\frac{1}{N}-p_i)\int \mathbb{D}_{\mathcal{X}}(\widehat{x}_i,x)+\mathbb{D}_{\mathcal{Y}}(\widehat{y}_i,y)~d\mathbb{Q}_i^{2}(x,y)\\
		\ge&\sum_{i=1}^mp_i\int d_i+\mathbb{D}_{\mathcal{Y}}(\widehat{y}_i,y)~d\mathbb{Q}_i^{1}(x,y)\\
		&+\sum_{i=1}^m(\frac{1}{N}-p_i)\int \mathbb{D}_{\mathcal{X}}(\widehat{x}_i,\widehat{x}_i)+\mathbb{D}_{\mathcal{Y}}(\widehat{y}_i,\widehat{y}_i)~d\mathbb{Q}_i^{2}(x,y)\\
		&+\sum_{i=m+1}^Np_i\int \mathbb{D}_{\mathcal{X}}(\widehat{x}_i,\widehat{x}_i)+\mathbb{D}_{\mathcal{Y}}(\widehat{y}_i,y)~d\mathbb{Q}_i^{1}(x,y)\\
		&+\sum_{i=m+1}^N(\frac{1}{N}-p_i)\int \mathbb{D}_{\mathcal{X}}(\widehat{x}_i,\widehat{x}_i^p+\phi_{\mathbb{Q}}^i{(\widehat{x}_i^p-\widehat{x}_i)})+\mathbb{D}_{\mathcal{Y}}(\widehat{y}_i,\widehat{y}_i)~d\mathbb{Q}_i^{2}(x,y).
	\end{aligned}
	$$
	Rearranging and combining the equations, and applying the definition of the optimal transport, we obtain
	\begin{align}\label{decompose}
		\mathcal{W}(\widehat{\mathbb{P}},\mathbb{Q})\ge&\sum_{i=1}^Np_i\int\mathbb{D}_{\mathcal{Y}}(\widehat{y}_i,y)~d\mathbb{Q}_i^{1}|_{Y}(y)+\sum_{i=1}^mp_id_i+\sum_{i=m+1}^N\left(\frac{1}{N}-p_i\right)d_i+\sum_{i=m+1}^N\left(\frac{1}{N}-p_i\right)\psi_{\mathbb{Q}}\notag\\
		\ge&\left(\sum_{j=1}^Np_j\right)\mathcal{W}(\sum_{i=1}^N\frac{p_i}{\sum_{j=1}^Np_j}\mathrm{I}_{\widehat{y}_i},\sum_{i=1}^N\frac{p_i}{\sum_{j=1}^Np_j}\mathbb{Q}_i^{1}|_{Y})\notag\\
		&+\sum_{i=1}^mp_id_i+\sum_{i=m+1}^N\left(\frac{1}{N}-p_i\right)d_i+\sum_{i=m+1}^N\left(\frac{1}{N}-p_i\right)\psi_{\mathbb{Q}}.
	\end{align}
	Denote $\epsilon=\frac{1}{\sum_{j=1}^Np_j}$, and $p_i'=\frac{p_i}{\sum_{j=1}^Np_j}$, $i=1,...,N$. Then $\frac{1}{\epsilon}\in\mathcal{N}_{\omega}$, $p_i'\in[0,\frac{\epsilon}{N}]$, $i=1,...,N$, $\sum_{i=1}^Np_i'=1$, and $\mathbb{Q}|_Y=\sum_{i=1}^Np_i'\mathbb{Q}_i^1|_Y$. Let $\widehat{\mathbb{P}}^*|_{Y}=\sum_{i=1}^Np_i'\mathrm{I}_{\widehat{y}_i}.$ Then inequality \eqref{decompose} can be rearranged as
	\begin{align}\label{decompse1}
		\mathcal{W}(\widehat{\mathbb{P}}^*|_{Y},\mathbb{Q}|_Y)\le& \epsilon \mathcal{W}(\widehat{\mathbb{P}},\mathbb{Q})-\sum_{i=1}^mp_i'd_i-\sum_{i=m+1}^N\left(\frac{\epsilon}{N}-p_i'\right)d_i-\sum_{i=m+1}^N\left(\frac{\epsilon}{N}-p_i'\right)\psi_{\mathbb{Q}}\notag\\
		\le&\epsilon\delta_0-\sum_{i=1}^mp_i'd_i-\sum_{i=m+1}^N\left(\frac{\epsilon}{N}-p_i'\right)d_i-\sum_{i=m+1}^N\left(\frac{\epsilon}{N}-p_i'\right)\psi_{\mathbb{Q}}.
	\end{align}
    We consider in the following two cases.
    \begin{itemize}
        \item[(i)] If $p_i'=\frac{\epsilon}{N}$ for $i=m+1,...,N$, from inequality \eqref{decompse1}, we have
        \begin{align}\label{decompose2_V1}
            \mathcal{W}(\widehat{\mathbb{P}}^*|_{Y},\mathbb{Q}|_Y)
		\le&\epsilon\delta_0-\sum_{i=1}^mp_i'd_i-\sum_{i=m+1}^N\left(\frac{\epsilon}{N}-p_i'\right)d_i-\sum_{i=m+1}^N\left(\frac{\epsilon}{N}-p_i'\right)\psi_{\mathbb{Q}}\notag\\
        =&\epsilon\delta_0-\sum_{i=1}^mp_i'd_i=\delta.
        \end{align}
        \item[(ii)] If there exists $i\in\{m+1,...,N\}$ such that $p_i'<\frac{\epsilon}{N}$, from inequality \eqref{decompse1}, we have
        \begin{align}\label{decompose2_V2}
            \mathcal{W}(\widehat{\mathbb{P}}^*|_{Y},\mathbb{Q}|_Y)
		\le&\epsilon\delta_0-\sum_{i=1}^mp_i'd_i-\sum_{i=m+1}^N\left(\frac{\epsilon}{N}-p_i'\right)d_i-\sum_{i=m+1}^N\left(\frac{\epsilon}{N}-p_i'\right)\psi_{\mathbb{Q}}\notag\\
        <&\epsilon\delta_0-\sum_{i=1}^mp_i'd_i-\sum_{i=m+1}^N\left(\frac{\epsilon}{N}-p_i'\right)d_i=\delta.
        \end{align}
    \end{itemize}
	We conclude that, for any $\mathbb{Q}\in  \mathbb{B}_{\delta_0}(\widehat{\mathbb{P}}), \mathbb{Q}\left(X \in \mathcal{N}_\gamma\left(x_0\right)\right) \in\mathcal{N}_{\omega}$, the conditional distribution of $Y$ given $X\in \mathcal{N}_\gamma\left(x_0\right)$, $\mathbb{Q}|_Y$, satisfies the inequality \eqref{decompose2_V1} or inequality \eqref{decompose2_V2}, where $\frac{1}{\epsilon}\in\mathcal{N}_{\omega}$, $p_i'\in[0,\frac{\epsilon}{N}]$, $i=1,...,N$, $\sum_{i=1}^Np_i'=1$. Thus, we have 
	\begin{align*}
		\mathcal{K}({\mathcal{N}_\gamma\left(x_0\right),\mathcal{N}_{\omega},\mathbb{B}_{\delta_0}(\widehat{\mathbb{P}})})\subset\left(\mathop{\bigcup}\limits_{(p, \delta) \in \mathcal{V}_1} \mathbb{B}_{\delta}\left(\sum_{i=1}^N p_i \mathrm{I}_{\widehat{y}_i}\right)\right)\mathop{\bigcup}\left(\mathop{\bigcup}\limits_{(p, \delta) \in \mathcal{V}_2} \mathbb{B}_{\delta}^{\circ}\left(\sum_{i=1}^N p_i \mathrm{I}_{\widehat{y}_i}\right)\right).
	\end{align*}
	Next, we want to prove the inverse inclusion. For any distribution $$\mathbb{Q|}_{Y}
	\in \left(\mathop{\bigcup}\limits_{(p, \delta) \in \mathcal{V}_1} \mathbb{B}_{\delta}\left(\sum_{i=1}^N p_i \mathrm{I}_{\widehat{y}_i}\right)\right)\mathop{\bigcup}\left(\mathop{\bigcup}\limits_{(p, \delta) \in \mathcal{V}_2} \mathbb{B}_{\delta}^{\circ}\left(\sum_{i=1}^N p_i \mathrm{I}_{\widehat{y}_i}\right)\right),$$ there exists $(p, \delta) \in \mathcal{V}_1$ such that $$\mathcal{W}(\sum_{i=1}^Np_i\mathrm{I}_{\widehat{y}_i},\mathbb{Q}|_Y)\le\delta,$$ or exists $(p, \delta) \in \mathcal{V}_2$ such that $$\mathcal{W}(\sum_{i=1}^Np_i\mathrm{I}_{\widehat{y}_i},\mathbb{Q}|_Y)<\delta.$$  
    Assume that the optimal transport between $\sum_{i=1}^Np_i\mathrm{I}_{\widehat{y}_i}$ and $\mathbb{Q|}_{Y}$ is $\pi^*|_{Y}$. Under the optimal transport, $\mathbb{Q|}_{Y}$ has the following decomposition:
	$$
	\mathbb{Q}|_{Y}=\sum_{i=1}^Np_i\mathbb{Q}_i|_{Y},
	$$
	where $\mathbb{Q}_i|_{Y}$ is transferred from $\mathrm{I}_{\widehat{y}_i}$ under the optimal transport. One can verify that for $(p,\delta)\in\mathcal{V}$, there exists $\psi_\mathbb{Q}>0$ such that $\mathcal{W}(\sum_{i=1}^Np_i\mathrm{I}_{\widehat{y}_i},\mathbb{Q}|_{{Y}})\le \delta-\sum_{i=m+1}^N(\frac{\epsilon}{N}-p_i)\psi_{\mathbb{Q}}$. Specifically, if $(p,\delta)\in\mathcal{V}_1$, we have $p_i=\frac{\epsilon}{N}$ for $i=m+1,...,N$ and $\mathcal{W}(\sum_{i=1}^Np_i\mathrm{I}_{\widehat{y}_i},\mathbb{Q}|_{{Y}})\le \delta=\delta-\sum_{i=m+1}^N(\frac{\epsilon}{N}-p_i)\psi_{\mathbb{Q}}$ for any $\psi_{\mathbb{Q}}>0$. Otherwise, if $(p,\delta)\in\mathcal{V}_2$, we have $\mathcal{W}(\sum_{i=1}^Np_i\mathrm{I}_{\widehat{y}_i},\mathbb{Q}|_{{Y}})\le \delta-\sum_{i=m+1}^N(\frac{\epsilon}{N}-p_i)\psi_{\mathbb{Q}}$ for some $\psi_{\mathbb{Q}}>0$, due to $\mathcal{W}(\sum_{i=1}^Np_i\mathrm{I}_{\widehat{y}_i},\mathbb{Q}|_{{Y}})<\delta$. Then we can construct 
	$$
	\begin{aligned}
		\mathbb{Q}=&\sum_{i=1}^m\frac{p_i}{\epsilon}\mathrm{I}_{\widehat{x}_i^p}\times\mathbb{Q}_i|_{Y}+\sum_{i=m+1}^N\frac{p_i}{\epsilon}\mathrm{I}_{\widehat{x}_i}\times\mathbb{Q}_i|_{Y}\\
		&+\sum_{i=1}^m(\frac{1}{N}-\frac{p_i}{\epsilon})\mathrm{I}_{(\widehat{x}_i,\widehat{y}_i)}+\sum_{i=m+1}^N(\frac{1}{N}-\frac{p_i}{\epsilon})\mathrm{I}_{(\widehat{x}_i^p+\phi_{\mathbb{Q}}^i{(\widehat{x}_i^p-\widehat{x}_i)},\widehat{y}_i)},
	\end{aligned}
	$$
	where $\phi_\mathbb{Q}^i$ satisfies that $\mathbb{D}_{\mathcal{X}}(\widehat{x}_i,\widehat{x}_i^p+\phi^i_\mathbb{Q}{(\widehat{x}_i^p-\widehat{x}_i)})=d_i+\psi^i_\mathbb{Q}=d_i+\psi^i_\mathbb{Q}$. Note that the conditional distribution of $Y$ given $X\in \mathcal{N}_\gamma\left(x_0\right)$ under the joint distribution $\mathbb{Q}$ is $\mathbb{Q}|_{Y}$. Moreover, one can verify that $\widehat{\mathbb{P}}$ can be represented as 
	$$
	\begin{aligned}
		\widehat{\mathbb{P}}=&\sum_{i=1}^m\frac{p_i}{\epsilon}\mathrm{I}_{(\widehat{x}_i,\widehat{y}_i)}+\sum_{i=m+1}^N\frac{p_i}{\epsilon}\mathrm{I}_{(\widehat{x}_i,\widehat{y}_i)}\\
		&+\sum_{i=1}^m(\frac{1}{N}-\frac{p_i}{\epsilon})\mathrm{I}_{(\widehat{x}_i,\widehat{y}_i)}+\sum_{i=m+1}^N(\frac{1}{N}-\frac{p_i}{\epsilon})\mathrm{I}_{(\widehat{x}_i,\widehat{y}_i)},
	\end{aligned}
	$$
	and using the definition of optimal transport, we obtain
	$$
	\begin{aligned}
		\mathcal{W}(\widehat{\mathbb{P}},\mathbb{Q})\le &\sum_{i=1}^m\frac{p_i}{\epsilon}(d_i+\int \mathbb{D}_{\mathcal{Y}}(\widehat{y}_i,y)~d\mathbb{Q}_i|_{Y}(y))\\
		&+\sum_{i=m+1}^N\frac{p_i}{\epsilon}(\mathbb{D}_{\mathcal{X}}(\widehat{x}_i,\widehat{x}_i)+\int \mathbb{D}_{\mathcal{Y}}(\widehat{y}_i,y)~d\mathbb{Q}_i|_{Y}(y))\\
		&+\sum_{i=1}^m(\frac{1}{N}-\frac{p_i}{\epsilon})(\mathbb{D}_{\mathcal{X}}(\widehat{x}_i,\widehat{x}_i)+\mathbb{D}_{\mathcal{Y}}(\widehat{y}_i,\widehat{y}_i))\\
		&+\sum_{i=m+1}^N(\frac{1}{N}-\frac{p_i}{\epsilon}) (\mathbb{D}_{\mathcal{X}}(\widehat{x}_i,\widehat{x}_i^p+\phi_{\mathbb{Q}}{(\widehat{x}_i^p-\widehat{x}_i)})+\mathbb{D}_{\mathcal{Y}}(\widehat{y}_i,\widehat{y}_i)).
	\end{aligned}
	$$
	And thus, we have
	\begin{align*}
		\mathcal{W}(\widehat{\mathbb{P}},\mathbb{Q})\le&\sum_{i=1}^m\frac{p_i}{\epsilon}\int \mathbb{D}_{\mathcal{Y}}(\widehat{y}_i,y)~d\mathbb{Q}_i|_{Y}(y)+\sum_{i=m+1}^N\frac{p_i}{\epsilon}\int \mathbb{D}_{\mathcal{Y}}(\widehat{y}_i,y)~d\mathbb{Q}_i|_{Y}(y)\\
		&+\sum_{i=1}^m\frac{p_i}{\epsilon}d_i+\sum_{i=m+1}^N(\frac{1}{N}-\frac{p_i}{\epsilon})d_i+\sum_{i=m+1}^N(\frac{1}{N}-\frac{p_i}{\epsilon})\psi_{\mathbb{Q}}\\
		=&\frac{1}{\epsilon}\mathcal{W}(\sum_{i=1}^Np_i\mathrm{I}_{\widehat{y}_i},\mathbb{Q}|_{Y})+\sum_{i=1}^m\frac{p_i}{\epsilon}d_i+\sum_{i=m+1}^N(\frac{1}{N}-\frac{p_i}{\epsilon})d_i+\sum_{i=m+1}^N(\frac{1}{N}-\frac{p_i}{\epsilon})\psi_{\mathbb{Q}}\\
		\le&\delta_0.
	\end{align*}
	We conclude that for any $$\mathbb{Q|}_{Y}
	\in \left(\mathop{\bigcup}\limits_{(p, \delta) \in \mathcal{V}_1} \mathbb{B}_{\delta}\left(\sum_{i=1}^N p_i \mathrm{I}_{\widehat{y}_i}\right)\right)\mathop{\bigcup}\left(\mathop{\bigcup}\limits_{(p, \delta) \in \mathcal{V}_2} \mathbb{B}_{\delta}^{\circ}\left(\sum_{i=1}^N p_i \mathrm{I}_{\widehat{y}_i}\right)\right),$$ there exists $\mathbb{Q}\in \mathbb{B}_{\delta_0}$ satisfies $\mathbb{Q}\left(X \in \mathcal{N}_\gamma\left(x_0\right)\right) \geq \varepsilon$ such that the conditional distribution of $Y$ given $X\in \mathcal{N}_\gamma\left(x_0\right)$ is $\mathbb{Q}|_{Y}$. Then we have 
	\begin{align*}
		\left(\mathop{\bigcup}\limits_{(p, \delta) \in \mathcal{V}_1} \mathbb{B}_{\delta}\left(\sum_{i=1}^N p_i \mathrm{I}_{\widehat{y}_i}\right)\right)\mathop{\bigcup}\left(\mathop{\bigcup}\limits_{(p, \delta) \in \mathcal{V}_2} \mathbb{B}_{\delta}^{\circ}\left(\sum_{i=1}^N p_i \mathrm{I}_{\widehat{y}_i}\right)\right)\subset\mathcal{K}({\mathcal{N}_\gamma\left(x_0\right),\mathcal{N}_\omega,\mathbb{B}_{\delta_0}(\widehat{\mathbb{P}})}),
	\end{align*} and thus 
	\begin{align*}
		\left(\mathop{\bigcup}\limits_{(p, \delta) \in \mathcal{V}_1} \mathbb{B}_{\delta}\left(\sum_{i=1}^N p_i \mathrm{I}_{\widehat{y}_i}\right)\right)\mathop{\bigcup}\left(\mathop{\bigcup}\limits_{(p, \delta) \in \mathcal{V}_2} \mathbb{B}_{\delta}^{\circ}\left(\sum_{i=1}^N p_i \mathrm{I}_{\widehat{y}_i}\right)\right)=\mathcal{K}({\mathcal{N}_\gamma\left(x_0\right),\mathcal{N}_\omega,\mathbb{B}_{\delta_0}(\widehat{\mathbb{P}})}).
	\end{align*}
Denote $$\mathcal{U_B}=\left(\mathop{\bigcup}\limits_{(p, \delta) \in \mathcal{V}_1} \mathbb{B}_{\delta}\left(\sum_{i=1}^N p_i \mathrm{I}_{\widehat{y}_i}\right)\right)\mathop{\bigcup}\left(\mathop{\bigcup}\limits_{(p, \delta) \in \mathcal{V}_2} \mathbb{B}_{\delta}^{\circ}\left(\sum_{i=1}^N p_i \mathrm{I}_{\widehat{y}_i}\right)\right).$$ From the above discussion, the problem \eqref{m1} is equivalent to $$
\min_{\alpha \in \mathcal{A}} \sup_{{\mathbb{Q}}|_{Y} \in \mathcal{U_B}} 
\rho_{{\mathbb{Q}}|_{Y}}\left[\ell(Y, \alpha) \right].
$$
We next prove that given $\alpha\in\mathcal{A}$,
\begin{align}\label{supremum_equivalence}
    \sup_{{\mathbb{Q}}|_{Y} \in \mathcal{U_B}} 
\rho_{{\mathbb{Q}}|_{Y}}\left[\ell(Y, \alpha) \right]=\sup_{{\mathbb{Q}}|_{Y} \in \bigcup\limits_{(p,\delta) \in \mathcal{V}} 
\mathbb{B}_{\delta}\left(\sum_{i=1}^N p_i\mathrm{I}_{\widehat{y}_i}\right)} 
\rho_{{\mathbb{Q}}|_{Y}}\left[\ell(Y, \alpha) \right].
\end{align}
We recall that the risk measure $\rho$ and the loss function $\ell$ satisfy Assumption \ref{assum2}. 
Based on Assumption \ref{assum2}, we only need to prove the following claim.
 \begin{quote}
    {\it \textbf{Claim} For any ${\mathbb{Q}}|_{Y}\in\mathbb{B}_{\delta}\left(\sum_{i=1}^N p_i \mathrm{I}_{\widehat{y}_i}\right)$ with $(p,\delta)\in\mathcal{V}_2$ and sufficiently large $n\in\mathbb{N}$, denoted by $Y\sim {\mathbb{Q}}|_{Y}$ and $Y_0\sim\sum_{i=1}^N p_i \mathrm{I}_{\widehat{y}_i}$, there exist ${\mathbb{Q}}|_{Y_n}\in\mathcal{U_B}$ and $Y_n\sim {\mathbb{Q}}|_{Y_n}$ such that $\mathcal{W}(Y,{Y}_n)\le \frac{1}{n}$. }
 \end{quote}
  In fact, if the claim holds, for any ${\mathbb{Q}}|_{Y}\in\mathbb{B}_{\delta}\left(\sum_{i=1}^N p_i \mathrm{I}_{\widehat{y}_i}\right)$ with $(p,\delta)\in\mathcal{V}_2$, we find that $\{{\mathbb{Q}}|_{Y_n}\}\subset\mathcal{U_B}$ satisfies that $\mathcal{W}({\mathbb{Q}}|_{Y_n},{\mathbb{Q}}|_{Y})\rightarrow 0$ as $n\to 0$ and then $\rho\left[\ell(Y_n, \alpha) \right]\rightarrow\rho\left[\ell(Y, \alpha) \right]$ as $n\rightarrow\infty$, due to the OT continuity of $\rho\circ\ell$. Moreover, for any ${\mathbb{Q}}|_{Y}\in\mathbb{B}_{\delta}\left(\sum_{i=1}^N p_i \mathrm{I}_{\widehat{y}_i}\right)$ with $(p,\delta)\in\mathcal{V}_1$, let $\mathbb{Q}|_{Y_n}={\mathbb{Q}}|_{Y}$. Then for any $\epsilon>0$, we have 
  \begin{align*}      \sup_{{\mathbb{Q}}|_{Y} \in \mathcal{U_B}} 
\rho_{{\mathbb{Q}}|_{Y}}\left[\ell(Y, \alpha) \right]-\epsilon\le\sup_{{\mathbb{Q}}|_{Y} \in \bigcup\limits_{(p,\delta) \in \mathcal{V}} 
\mathbb{B}_{\delta}\left(\sum_{i=1}^N p_i\mathrm{I}_{\widehat{y}_i}\right)} 
\rho_{{\mathbb{Q}}|_{Y}}\left[\ell(Y, \alpha) \right]\le\sup_{{\mathbb{Q}}|_{Y} \in \mathcal{U_B}} 
\rho_{{\mathbb{Q}}|_{Y}}\left[\ell(Y, \alpha) \right]+\epsilon,
  \end{align*}
and thus the equation \eqref{supremum_equivalence} holds. Next we prove the claim in the two cases: $\delta>0$ and $\delta=0$.
\begin{itemize}
    \item[(i)] For ${\mathbb{Q}}|_{Y}\in\mathbb{B}_{\delta}\left(\sum_{i=1}^N p_i \mathrm{I}_{\widehat{y}_i}\right)$ with $(p,\delta)\in\mathcal{V}_2$ and $\delta>0$, we find that $\mathbb{B}_{\delta}\left(\sum_{i=1}^N p_i \mathrm{I}_{\widehat{y}_i}\right)$ is the closure of $\mathbb{B}_{\delta}^\circ\left(\sum_{i=1}^N p_i \mathrm{I}_{\widehat{y}_i}\right)$ under the cost $\mathcal{W}$. Thus, there exist ${\mathbb{Q}}|_{Y_n}\in\mathbb{B}_{\delta}^\circ\left(\sum_{i=1}^N p_i \mathrm{I}_{\widehat{y}_i}\right)\subset\mathcal{U_B}$ and $Y_n\sim {\mathbb{Q}}|_{Y_n}$ such that $\mathcal{W}(Y,{Y}_n)\le \frac{1}{n}$ for sufficiently large $n\in\mathbb{N}$. 
    \item[(ii)] For ${\mathbb{Q}}|_{Y}\in\mathbb{B}_{\delta}\left(\sum_{i=1}^N p_i \mathrm{I}_{\widehat{y}_i}\right)$ with $(p,\delta)\in\mathcal{V}_2$ and $\delta=0$, we find that $\mathbb{B}_{\delta}\left(\sum_{i=1}^N p_i \mathrm{I}_{\widehat{y}_i}\right)$ is a single distribution set and ${\mathbb{Q}}|_{Y}=\sum_{i=1}^N p_i \mathrm{I}_{\widehat{y}_i}$. We consider the following two cases.
    \begin{itemize}
        \item[(iia)] If $\delta_0>\delta_{min}$, from the proof of Proposition \ref{feasible_con}, there exist $(p^*,\delta^*)\in\mathcal{V}$ such that $\delta^*>0$. Let $Y\sim\sum_{i=1}^N p_i \mathrm{I}_{\widehat{y}_i}$ and $Y^*\sim\sum_{i=1}^Np^*_i\mathrm{I}_{\widehat{y}_i}$. One can verify that $\mathcal{W}(Y,Y^*)\le \max_{i,j=1,...,N}\mathbb{D}_\mathcal{Y}(\widehat{y}_i,\widehat{y}_j):=D$. For $n\in\mathbb{N}$ satisfies that $nD\ge 1$, let $p^n=(1-\frac{1}{nD}) p+\frac{1}{nD}p^*$ and $\delta^n=(1-\frac{1}{nD})\delta+\frac{1}{nD}\delta^*$. Due to the convexity of $\mathcal{V}$, we have $(p^n,\delta^n)\in\mathcal{V}$ and thus $\sum_{i=1}^Np^n_i\mathrm{I}_{\widehat{y}_i}\in\mathcal{U_B}$ because  $\delta^*>0$ and then $\delta_n>0$. Let $Y_n\sim\sum_{i=1}^Np^n_i\mathrm{I}_{\widehat{y}_i}$. Then we find that $\mathcal{W}(Y,{Y}_n)\le \frac{1}{nD}\mathcal{W}(Y,{Y}^*) \le \frac{1}{n}$.
        \item[(iib)] If $\delta_0=\delta_{min}$, we claim that for every $(p,\delta)\in\mathcal{V}_2$, $\delta<0$. Note that $\delta_0$ can take the smallest radius $\delta_{min}$ if and only if $\frac{N-m}{N}\le \omega_2$. In fact, if $(p,\delta)\in\mathcal{V}_2$, there exists $\epsilon$ such that $(p,\delta,\epsilon)\in \mathcal{V}_0$, and exists $i\in\{m+1,..,N\}$ such that $p_i<\frac{\epsilon}{N}$. Note that $\delta=\epsilon\left(\delta_{0}-\frac{1}{N}\sum_{i=m+1}^Nd_i\right)-\sum_{i=1}^{m}p_id_i+\sum_{i=m+1}^Np_id_i$. If there exist $(p,\delta)\in\mathcal{V}_2$ satisfies that $\delta\ge0$, we have 
        \begin{align*}
            \delta_0\ge&\sum_{i=1}^{m}\frac{p_i}{\epsilon}d_i+\sum_{i=m+1}^N(\frac{1}{N}-\frac{p_i}{\epsilon})d_i\\
            =&\frac{1}{N}\left(\sum_{i=1}^{m}\frac{Np_i}{\epsilon}d_i+\sum_{i=m+1}^N(1-\frac{Np_i}{\epsilon})d_i\right)\\
            :=&\frac{1}{N}\left(\sum_{i=1}^{m}v_i'd_i+\sum_{i=m+1}^N(1-v_i')d_i\right)\\
            >&\min_{\substack{v_i\in[0,1],i=1,...,N, \\ 1/N\sum_{i=1}^Nv_i\in[\omega_1,\omega_2]}}\frac{1}{N}\left(\sum_{i=1}^{m}v_id_i+\sum_{i=m+1}^N(1-v_i)d_i\right)=\delta_{min},
        \end{align*}
        where the strict inequality follows from that we can always find $\{v_i\}_{i=1}^N$ such that $\sum_{i=1}^{m}v_i'd_i+\sum_{i=m+1}^N(1-v_i')d_i>\sum_{i=1}^{m}v_id_i+\sum_{i=m+1}^N(1-v_i)d_i$. In fact, for $(p,\delta)\in\mathcal{V}_2$, we find that $v_i'$s satisfy $v_i '\in[0,1], i=1, \ldots, N$, $ 1 / N \sum_{i=1}^N v_i' \in\left[\omega_1, \omega_2\right]$, and there exists $i \in\{m+1, \ldots, N\}$ such that $v_i'<1$. Without loss of generality, we assume $v_i'=1$ for $i=m+1,...,k$ and $v_i'<1$ for $i=k+1,...,N$. If $1 / N \sum_{i=1}^N v_i'\le\frac{N-m}{N}$, let $v_i=0$ for $i=1,..,m$ and $v_i=1$ for $i=m+1,...,N$. Otherwise, let $v_i=v_i'-\Delta_i\ge 0$ for $i=1,..,m$ and $v_i=1$ for $i=m+1,...,N$, where $\Delta_i\ge 0$ and $\sum_{i=1}^m\Delta_i=\sum_{i=k+1}^N(1-v_i')$. We find $\{v_i\}_{i=1}^N$ satisfies that $v_i\in[0,1],i=1,...,N,  1/N\sum_{i=1}^Nv_i\in[\omega_1,\omega_2]$ and such $\{v_i\}_{i=1}^N$ must exists due to that $\frac{N-m}{N}\le \omega_2$. Under the above construction of $\{v_i\}_{i=1}^N$, we have $\sum_{i=1}^{m}v_id_i+\sum_{i=m+1}^N(1-v_i)d_i<\sum_{i=1}^{m}v_i'd_i+\sum_{i=m+1}^N(1-v_i')d_i$. The above strict inequality contradicts the assumption $\delta_0=\delta_{min}$ and thus we have that for every $(p,\delta)\in\mathcal{V}_2$, $\delta<0$. Thus, ${\mathbb{Q}}|_{Y}\in\mathbb{B}_{\delta}\left(\sum_{i=1}^N p_i \mathrm{I}_{\widehat{y}_i}\right)$ is an empty set for every $(p,\delta)\in\mathcal{V}_2$.
    \end{itemize}
\end{itemize}
Combining the above two cases, we complete the proof of claim and thus complete the proof of Theorem \ref{thm1}.\qedhere
\end{proof}

\begin{proof}[{\bf Proof of Proposition \ref{feasible_con_2}}]
    To obtain the sufficient and necessary feasible condition of problem \eqref{m2}, we only need to find the sufficient and necessary feasible condition of problem \eqref{conditional-version} and thus only need to find the sufficient and necessary condition such that the ambiguity set $\mathcal{K}({\mathcal{N}_\gamma(x_0),\mathbb{B}_{\delta_0}(\widehat{\mathbb{P}};\mathcal{N}_{\omega})})$ is not empty. Note that
\begin{align*}
		\mathcal{K}({\mathcal{N}_\gamma(x_0),\mathbb{B}_{\delta_0}(\widehat{\mathbb{P}};\mathcal{N}_{\omega})})=\mathop{\bigcup}\limits_{(p, \delta) \in \mathcal{V}} \mathbb{B}_{\delta}\left(\sum_{i=1}^N p_i\mathrm{I}_{\widehat{y}_i}\right)
	\end{align*}
 from the proof of Theorem \ref{thm2}. Thus, $\mathcal{K}({\mathcal{N}_\gamma(x_0),\mathbb{B}_{\delta_0}(\widehat{\mathbb{P}};\mathcal{N}_{\omega})})$ is not empty if and only if there exists $(p,\delta) \in\mathcal{V}$ such that $\delta=\delta_0-\sum_{i=1}^m p_i d_i\ge0$. One can verify that there exists $(p,\delta)\in\mathcal{V}$ such that $\delta\ge 0$ is equivalent to
 \begin{align*}
     \max_{\substack{p_i\in[0,\frac{1}{N\epsilon}],i=1,...,N,\\ \sum_{i=1}^N p_i=1}}\delta_0-\sum_{i=1}^m p_i d_i\ge0,
 \end{align*}
 and thus equivalent to
 \begin{align*}
     \delta_0\ge& \min_{\substack{p_i\in[0,\frac{1}{N\epsilon}],i=1,...,N,\\ \sum_{i=1}^N p_i=1}}\sum_{i=1}^m p_i d_i\\
     =&\min_{\substack{v_i\in[0,1],i=1,...,N,\\ \sum_{i=1}^N v_i=N\epsilon}}\frac{1}{N\epsilon}\sum_{i=1}^m v_i d_i=\delta_{min}.
 \end{align*}
 \qedhere
\end{proof}
\begin{proof}[{\bf Proof of Theorem \ref{thm2}}]
	For problem \eqref{m2}, note that $\mathbb{Q}$ is the joint distribution of $(X,Y)$ and the objective function is only related to $Y$. Then we can reformulate the inner problem of problem \eqref{m2} as 
	\begin{align}\label{innerproblem-m2}
		\sup _{\mathbb{Q}|_{{Y}} \in\mathcal{K}({\mathcal{N}_\gamma(x_0),\mathbb{B}_{\delta_0}(\widehat{\mathbb{P}};\mathcal{N}_{\omega})}) } \rho_{\mathbb{Q}|_{{Y}}}\left[\ell(Y, \alpha) \right],
	\end{align}
	where $\mathbb{Q}|_{{Y}}$ is the marginal distribution of $Y$ and $\mathcal{K}({\mathcal{N}_\gamma(x_0),\mathbb{B}_{\delta_0}(\widehat{\mathbb{P}};\mathcal{N}_{\omega})})$ is the marginal distribution set on $\mathcal{Y}$, which is defined as
	\begin{align*}
		\mathcal{K}({\mathcal{N}_\gamma(x_0),\mathbb{B}_{\delta_0}(\widehat{\mathbb{P}};\mathcal{N}_{\omega})})=\Bigg\{\mathbb{Q}|_Y:&\mathbb{Q}|_Y\text{ is the marginal distribution of }\mathbb{Q}\text{ on }\mathcal{Y},\\
		&\text{ where } \mathbb{Q}\in\mathbb{B}_{\delta_0}(\widehat{\mathbb{P}};\mathcal{N}_{\omega})\text{ satisfies that }\mathbb{Q}\left(X \in \mathcal{N}_\gamma(x_0)\right) =1
		\Bigg\}.
	\end{align*}
	Thus, we only need to prove that 
	\begin{align*}
		\mathcal{K}({\mathcal{N}_\gamma(x_0),\mathbb{B}_{\delta_0}(\widehat{\mathbb{P}};\mathcal{N}_{\omega})})=\mathop{\bigcup}\limits_{(p, \delta) \in \mathcal{V}} \mathbb{B}_{\delta}\left(\sum_{i=1}^N p_i\mathrm{I}_{\widehat{y}_i}\right).
	\end{align*}
	For any $\mathbb{Q}\in \mathbb{B}_{\delta_0}(\widehat{\mathbb{P}};\mathcal{N}_{\omega})$, based on definition of the trimming set and the aforementioned notifications, we can find that there exists $p$ satisfies that given $\epsilon=\min_{\beta\in\mathcal{N}_{\omega}}\beta$, $p_i\in[0,\frac{1}{N\epsilon}], i=1,..., N, \sum_{i=1}^Np_i=1$, such that $\mathcal{W}\left(\sum_{i=1}^Np_i\mathrm{I}_{(\widehat{x}_i,\widehat{y}_i)}, \mathbb{Q}\right) \leqslant \delta_0$. Furthermore, if $\mathbb{Q}(X\in\mathcal{N}_\gamma(x_0))=1$, assume that the optimal transport between $\mathbb{Q}$ and $\sum_{i=1}^Np_i\mathrm{I}_{(\widehat{x}_i,\widehat{y}_i)}$ is $\pi$.
	Then we can decompose $\mathbb{Q}$ as $\mathbb{Q}=\sum_{i=1}^Np_i\mathbb{Q}^i$, where $\mathbb{Q}^i$ is transferred by $\mathrm{I}_{(\widehat{x}_i,\widehat{y}_i)}$ under the optimal transport $\pi$. Similar to the proof the Theorem \ref{thm1}, we have 
	$$
	\begin{aligned}
		\mathcal{W}(\sum_{i=1}^Np_i\mathrm{I}_{(\widehat{x}_i,\widehat{y}_i)},\mathbb{Q})=&\sum_{i=1}^mp_i\int \mathbb{D}_{\mathcal{X}}(\widehat{x}_i,x)+\mathbb{D}_{\mathcal{Y}}(\widehat{y}_i,y)~d\mathbb{Q}_i(x,y)\\
		&+\sum_{i=m+1}^Np_i\int \mathbb{D}_{\mathcal{X}}(\widehat{x}_i,x)+\mathbb{D}_{\mathcal{Y}}(\widehat{y}_i,y)~d\mathbb{Q}_i(x,y)\\
		\ge&\sum_{i=1}^mp_i\int d_i+\mathbb{D}_{\mathcal{Y}}(\widehat{y}_i,y)~d\mathbb{Q}_i(x,y)\\
		&+\sum_{i=m+1}^Np_i\int \mathbb{D}_{\mathcal{X}}(\widehat{x}_i,\widehat{x}_i)+\mathbb{D}_{\mathcal{Y}}(\widehat{y}_i,y)~d\mathbb{Q}_i(x,y)\\
		=&\sum_{i=1}^Np_i\int\mathbb{D}_{\mathcal{Y}}(\widehat{y}_i,y)~d\mathbb{Q}_i|_Y(y)+\sum_{i=1}^mp_id_i\\
		\ge&\mathcal{W}(\sum_{i=1}^Np_i\mathrm{I}_{\widehat{y}_i},\mathbb{Q}|_Y)+\sum_{i=1}^mp_id_i.
	\end{aligned}
	$$
	Rearranging the above inequality, we find that the marginal distribution on $\mathcal{Y}$ of $\mathbb{Q}$ satisfies that
	$$
	\mathcal{W}(\sum_{i=1}^Np_i\mathrm{I}_{\widehat{y}_i},\mathbb{Q}|_Y)\le\mathcal{W}(\sum_{i=1}^Np_i\mathrm{I}_{(\widehat{x}_i,\widehat{y}_i)},\mathbb{Q})-\sum_{i=1}^mp_id_i\le\delta.
	$$
	And thus, 
	\begin{align*}
		\mathcal{K}({\mathcal{N}_\gamma(x_0),\mathbb{B}_{\delta_0}(\widehat{\mathbb{P}};\mathcal{N}_{\omega})})\subset\mathop{\bigcup}\limits_{(p, \delta) \in \mathcal{V}} \mathbb{B}_{\delta}\left(\sum_{i=1}^N p_i\mathrm{I}_{\widehat{y}_i}\right).
	\end{align*}
	Next, we prove the inverse containment. For any $$\mathbb{Q}|_Y\in\mathop{\bigcup}\limits_{(p, \delta) \in \mathcal{V}} \mathbb{B}_{\delta}\left(\sum_{i=1}^N p_i\mathrm{I}_{\widehat{y}_i}\right),$$
	there exists $(p, \delta) \in \mathcal{V}$ such that $\mathbb{Q}|_Y\in\mathbb{B}_{\delta}\left(\sum_{i=1}^N p_i\mathrm{I}_{\widehat{y}_i}\right)$. Assume the optimal transport between $\mathbb{Q}|_Y$ and $\sum_{i=1}^N p_i\mathrm{I}_{\widehat{y}_i}$ is $\pi|_Y$. Then we can decompose it as $\mathbb{Q}|_Y=\sum_{i=1}^Np_i\mathbb{Q}^i|_Y$, where $\mathbb{Q}^i|_Y$ is transferred by $\mathrm{I}_{\widehat{y}_i}$ under the optimal transport $\pi|_Y$. Then we can construct a distribution $\widetilde{\mathbb{Q}}=\sum_{i=1}^mp_i\mathrm{I}_{\widehat{x}_i^p}\times\mathbb{Q}^i|_Y+\sum_{i=m+1}^Np_i\mathrm{I}_{\widehat{x}_i}\times\mathbb{Q}^i|_Y$ and verify that the marginal distribution on $\mathcal{Y}$ of  $\widetilde{\mathbb{Q}}$ is  $\mathbb{Q}|_Y$. Furthermore, we have
	\begin{align*}
		\mathcal{W}\left(\sum_{i=1}^Np_i\mathrm{I}_{(\widehat{x}_i,\widehat{y}_i)}, \widetilde{\mathbb{Q}}\right)\le& \sum_{i=1}^mp_i\int \mathbb{D}_{\mathcal{X}}(\widehat{x}_i,x)+\mathbb{D}_{\mathcal{Y}}
		(\widehat{y}_i,y)~d\mathrm{I}_{\widehat{x}_i^p}(x)\times\mathbb{Q}^i|_Y(y)\\
		&+\sum_{i=m+1}^Np_i\int \mathbb{D}_{\mathcal{X}}(\widehat{x}_i,x)+\mathbb{D}_{\mathcal{Y}}
		(\widehat{y}_i,y)~d\mathrm{I}_{\widehat{x}_i}(x)\times\mathbb{Q}^i|_Y(y)\\
		=&\sum_{i=1}^Np_i\int \mathbb{D}_{\mathcal{Y}}
		(\widehat{y}_i,y)~d\mathbb{Q}^i|_Y(y)+\sum_{i=1}^mp_id_i\\
		=&\mathcal{W}\left(\sum_{i=1}^N p_i\mathrm{I}_{\widehat{y}_i},\mathbb{Q}|_Y\right)+\sum_{i=1}^mp_id_i\\
		\le& \delta_0.
	\end{align*}
	From the above inequality and the construction of $\widetilde{\mathbb{Q}}$, we have $\widetilde{\mathbb{Q}}\in\mathbb{B}_{\delta_0}(\widehat{\mathbb{P}};\mathcal{N}_{\omega})$ and $\mathbb{Q}\left(X \in \mathcal{N}_\gamma(x_0)\right) =1$. Then we have 
	\begin{align*}
		\mathop{\bigcup}\limits_{(p, \delta) \in \mathcal{V}} \mathbb{B}_{\delta}\left(\sum_{i=1}^N p_i\mathrm{I}_{\widehat{y}_i}\right)\subset\mathcal{K}({\mathcal{N}_\gamma(x_0),\mathbb{B}_{\delta_0}(\widehat{\mathbb{P}};\mathcal{N}_{\omega})}),
	\end{align*}
	and thus
	\begin{align*}
		\mathcal{K}({\mathcal{N}_\gamma(x_0),\mathbb{B}_{\delta_0}(\widehat{\mathbb{P}};\mathcal{N}_{\omega})})=\mathop{\bigcup}\limits_{(p, \delta) \in \mathcal{V}} \mathbb{B}_{\delta}\left(\sum_{i=1}^N p_i\mathrm{I}_{\widehat{y}_i}\right).
	\end{align*}
\qedhere
\end{proof}

\subsection*{EC.2.2. Proofs of Section \ref{tractability}}
\begin{proof}[{\bf Proof of Proposition \ref{gen}}]
    Suppose $\rho_{\mathbb{Q}|_Y}$ is a law-invariant risk measure that is concave in distribution, then $g(\alpha,(p,\delta))$ is concave in $(p,\delta)$. In fact, one can verify that
\begin{align*}
    g(\alpha,(p,\delta))=&\sup_{{\mathbb{Q}}|_{Y}}\left\{\rho_{{\mathbb{Q}}|_{Y}}\left[\ell(Y, \alpha) \right]+\sigma({\mathbb{Q}}|_{Y},p,\delta|\mathbb{B}_{\delta}(\sum_{i=1}^Np_i\mathrm{I}_{\widehat{y}_i}))\right\}\\
    :=&\sup_{{\mathbb{Q}}|_{Y}}H({\mathbb{Q}}|_{Y},p,\delta),
\end{align*}
where $\sigma({\mathbb{Q}}|_{Y},p,\delta|\mathbb{B}_{\delta}(\sum_{i=1}^Np_i\mathrm{I}_{\widehat{y}_i}))$ is the indicator function of $\mathbb{B}_{\delta}(\sum_{i=1}^Np_i\mathrm{I}_{\widehat{y}_i})$ and $\mathbb{B}_{\delta}(\sum_{i=1}^Np_i\mathrm{I}_{\widehat{y}_i})$ is a convex set of $({\mathbb{Q}}|_{Y},p,\delta)$. Thus, we find that $H({\mathbb{Q}}|_{Y},p,\delta)$ is jointly concave in $({\mathbb{Q}}|_{Y},p,\delta)$ and then $g(\alpha,(p,\delta))$ is jointly concave in $(p,\delta)$ by taking the  supremum of $H({\mathbb{Q}}|_{Y},p,\delta)$ in ${\mathbb{Q}}|_{Y}$. Moreover, suppose $\rho_{\mathbb{Q}|_Y}\circ \ell$ is convex in $\alpha$. We find that
$g(\alpha,(p,\delta))=\sup _{{\mathbb{Q}}|_{Y} \in \mathbb{B}_{\delta}(\sum_{i=1}^Np_i\mathrm{I}_{\widehat{y}_i})} \rho_{{\mathbb{Q}}|_{Y}}\left[\ell(Y, \alpha) \right]$ is a supremum of convex functions with respect to $\alpha$ and thus $g(\alpha,(p,\delta))$ is convex in $\alpha$.\qedhere
\end{proof}

\begin{proof}[{\bf Proof of Lemma \ref{OT-continuity}}]
We begin by introducing an auxiliary lemma that will be used throughout the proof.

\begin{lemma}\label{UI_probability_ell(Yalpha)}
   Assume $Y_n\in L^q\sim \mathbb{Q}_n$ for $n\in\mathbb{N}$ and $Y\in L^q\sim \mathbb{Q}$ such that $\E\left[\mathbb{D}_{\mathcal{Y}}(Yn,Y)\right]=\mathcal{W}(\mathbb{Q}_n, \mathbb{Q})\rightarrow 0$ as $n\rightarrow\infty$, where $\mathbb{D}_{\mathcal{Y}}(y_1,y_2)=||y_1-y_2||_{\mathcal{Y}}^q$ for any $y_1,y_2\in\mathcal{Y}$ with $q\in[1,\infty)$. If $\ell$ is continuous and satisfies that there exists some $C>0$ such that $|\ell(x)|\le C(1+|x|^q)$, given $\alpha\in\mathcal{A}$, $\E\left[\ell(Y_n^\top\alpha)\right]\to \E\left[\ell(Y^\top\alpha)\right]$ as $n\to\infty$.
\end{lemma}
\begin{proof}[{\bf Proof of Lemma \ref{UI_probability_ell(Yalpha)}}]
    In the following proof, denote $\mathcal{W}(\mathbb{Q}_n, \mathbb{Q})$ as $\mathcal{W}_q(\mathbb{Q}_n, \mathbb{Q})$, $||\cdot||_{\mathcal{Y}}$ as $||\cdot||$ and $||\cdot||_*$ as the dual norm of $||\cdot||$. We first prove that $\{\|Y_n\|^q\}_n$ is uniformly integrable, that is $\lim_{M\rightarrow\infty}\sup_n \E[\|Y_n\|^q\mathrm{I}_{\{\|Y_n\|^q>M\}}]=0$. From the assumption that $\mathcal{W}_q(\mathbb{Q}_n, \mathbb{Q})\rightarrow 0$ as $n\rightarrow\infty$, we have these two facts: (i) $\mathbb{Q}_n\implies\mathbb{Q}$ (weak convergence); (ii) $\E\left[\|Y_n\|^q\right]\to\E\left[\|Y\|^q\right]$ as $n\to\infty$. One can verify that $\mathcal{W}_1(\mathbb{Q}_n, \mathbb{Q})\le\mathcal{W}_q(\mathbb{Q}_n, \mathbb{Q})$ and thus  $\mathcal{W}_1(\mathbb{Q}_n, \mathbb{Q})\rightarrow 0$ as $n\rightarrow\infty$. Using the Kantorovich–Rubinstein duality for $\mathcal{W}_1$, for any $1$-Lipschitz continuous function $f$, we have $\E_{\mathbb{Q}_n}\left[f(Y_n)\right]\rightarrow\E_{\mathbb{Q}}\left[f(Y)\right]$. This implies (i). To see (ii), using triangle inequality, we have 
    \begin{align*}
        \left|\left(\E\left[\|Y_n\|^q\right]\right)^{1/q}-\left(\E\left[\|Y\|^q\right]\right)^{1/q}\right|\le\left(\E\left[\|Y_n-Y\|^q\right]\right)^{1/q}=\mathcal{W}_q(\mathbb{Q}_n, \mathbb{Q})\to 0 \text{ as }n\to \infty.
    \end{align*}
    Hence, we have $\E\left[\|Y_n\|^q\right]\to\E\left[\|Y\|^q\right]$ as $n\to\infty$ and conclude the proof of (ii). Let $g(y)=||y||^q$ and $g_M(y)=\min\{g(y),M\}$ with $M>0$. Due to that $g_M$ is bounded and continuous, by (i), we have $\E\left[g_M(Y_n)\right]\to \E\left[g_M(Y)\right]$ as $n\to \infty$. Let $T_n(M)=\E\left[\|Y_n\|^q\mathrm{I}_{\{\|Y_n\|^q>M\}}\right]$ with $M>0$. Then $T_n(M)=\E\left[g(Y_n)\right]-\E\left[g_M(Y_n)\right]+M\E\left[\mathrm{I}_{\{g(Y_n)>M\}}\right]$. If 
    \begin{align}\label{YnUI}
        \sup_n T_n(M)\to 0 \text{ as } M\to\infty,
    \end{align} by the definition of uniformly integrable, we conclude that $\{\|Y_n\|^q\}_n$ is uniformly integrable. To prove \eqref{YnUI},  we only need to prove the following claim. \begin{quote}
    {\it \textbf{Claim}   (i) and (ii) imply that $\sup_n\E\left[g(Y_n)\right]-\E\left[g_M(Y_n)\right] \to 0$ as $M\to \infty$. }\end{quote}
    In fact, we first note that $M\mathrm{I}_{\{g>M\}}\le 2(g-g_{\frac{M}{2}})$. If the claim holds, we have 
    \begin{align*}
        \sup_n T_n(M)=&\sup_n \left\{\E\left[g(Y_n)\right]-\E\left[g_M(Y_n)\right]+M\E\left[\mathrm{I}_{\{g(Y_n)>M\}}\right]\right\}\\
        \le &\sup_n \left\{\E\left[g(Y_n)\right]-\E\left[g_M(Y_n)\right]+2\left(\E\left[g(Y_n)\right]-\E\left[g_{\frac{M}{2}}(Y_n)\right]\right)\right\}\\
        \le &\sup_n \left\{\E\left[g(Y_n)\right]-\E\left[g_M(Y_n)\right]\right\}+\sup_n \left\{2\left(\E\left[g(Y_n)\right]-\E\left[g_{\frac{M}{2}}(Y_n)\right]\right)\right\}\\
        \to&0 \text{ as } M\to\infty.
    \end{align*}
    Thus, we have $\{\|Y_n\|^q\}_n$ is uniformly integrable. Next, we prove the claim. For any $n\in\mathbb{N}$ and $M>0$, we have 
    \begin{align*}
        &\E\left[g(Y_n)\right]-\E\left[g_M(Y_n)\right]\le\left|\E\left[g(Y_n)\right]-\E\left[g_M(Y_n)\right]\right|\\
        \le&\left|\E\left[g(Y_n)\right]-\E\left[g(Y)\right]\right|+\left|\E\left[g(Y)\right]-\E\left[g_M(Y)\right]\right|+\left|\E\left[g_M(Y)\right]-\E\left[g_M(Y_n)\right]\right|.
    \end{align*}
    By (ii), given $\epsilon>0$, there exists $N_1$ such that for any $n\ge N_1$, we have $\left|\E\left[g(Y_n)\right]-\E\left[g(Y)\right]\right|<\epsilon/3$. Using (i) and the fact that $g_M$ is bounded and continuous, given $\epsilon>0$, for any $M>0$, there exists $N_2(M)$ such that for any $n\ge N_2(M)$, we have $\left|\E\left[g_M(Y)\right]-\E\left[g_M(Y_n)\right]\right|<\epsilon/3$. Note that $g_M\nearrow g$ as $M\to\infty$ and $g\in L^1$ ($Y_n\in L^q$). By monotone convergence theorem, we have $\E\left[g(Y)\right]-\E\left[g_M(Y)\right]\searrow 0$ as $M\to\infty$. Thus, given $\epsilon>0$, there exists $M_0>0$ such that for any $M\ge M_0$, we have $\left|\E\left[g(Y)\right]-\E\left[g_M(Y)\right]\right|<\epsilon/3$. Combining the above cases, we have the following result. Fix any $M\ge M_0$. Let $N=\max\{N_1,N_2(M)\}$. Then for any $n\ge N$, we have $\E\left[g(Y_n)\right]-\E\left[g_M(Y_n)\right]<\epsilon$. Moreover, for any $1\le n\le N$, we have $\E\left[g(Y_n)\right]-\E\left[g_M(Y_n)\right]\searrow 0$ as $M\to\infty$ by monotone convergence theorem. Thus, given $\epsilon>0$, there exists $M_1$ such that for $M\ge M_1$, we have $\E\left[g(Y_n)\right]-\E\left[g_M(Y_n)\right]\le \epsilon$ with $1\le n\le N$. Then, for $M\ge\max\{M_0,M_1\}$, we have 
    \begin{align*}
        \sup_n\E\left[g(Y_n)\right]-\E\left[g_M(Y_n)\right]=&\max\{\max_{1\le n\le N}\E\left[g(Y_n)\right]-\E\left[g_M(Y_n)\right],\sup_{n\ge N}\E\left[g(Y_n)\right]-\E\left[g_M(Y_n)\right]\}\\
        \to&0 \text{ as } M\to 0.
    \end{align*}
    Thus, we conclude the proof of that $\{\|Y_n\|^q\}_n$ is uniformly integrable. In the following, we are going to prove the following two facts.
    \begin{itemize}
        \item[(a)] Given $\alpha\in\mathcal{A}$, $\{\ell(Y_n^\top\alpha)\}_n$ is uniformly integrable.
        \item[(b)] Given $\alpha\in\mathcal{A}$, $\{\ell(Y_n^\top\alpha)\}_n\to \ell(Y^\top\alpha)$ in probability.
    \end{itemize}
    Because $|\ell(x)|\le C(1+|x|^q)$ with $C>0$, we have $|\ell(Y_n^\top\alpha)|\le C(1+|Y_n^\top\alpha|^q)\le C(1+\|\alpha\|_*^q\|Y_n\|^q)$. Then $|\ell(Y_n^\top\alpha)|\mathrm{I}_{\{|\ell(Y_n^\top\alpha)|>M\}}\le C(1+\|\alpha\|_*^q\|Y_n\|^q)\mathrm{I}_{\{C(1+\|\alpha\|_*^q\|Y_n\|^q)>M\}}$. Thus,
    \begin{align*}
        \sup_n \E\left[|\ell(Y_n^\top\alpha)|\mathrm{I}_{\{|\ell(Y_n^\top\alpha)|>M\}}\right]\le& \sup_n \E\left[C(1+\|\alpha\|_*^q\|Y_n\|^q)\mathrm{I}_{\{C(1+\|\alpha\|_*^q\|Y_n\|^q)>M\}}\right]\\
        \le & \sup_n C\E\left[C\mathrm{I}_{\{(1+\|\alpha\|_*^q\|Y_n\|^q)>M\}}\right]+\sup_n C\|\alpha\|_*^q\E\left[\|Y_n\|^q\mathrm{I}_{\{\|Y_n\|^q>\frac{M-C}{C\|\alpha\|_*^q}\}}\right]\\
        \to&0 \text{ as } M\to \infty.
    \end{align*}
    We complete the proof of (a). To see (b), due to (ii), we find that $Y_n\to Y$ in $L^q$. Then $Y_n^\top\alpha\to Y^\top\alpha$ in $L^q$ and thus in probability. Because of $\ell$ is continuous, we obtain $\ell(Y_n^\top\alpha)\to \ell(Y^\top\alpha)$ in probability. Combining (a) and (b), using Vitali's convergence theorem, we have $\E\left[|\ell(Y_n^\top\alpha)- \ell(Y^\top\alpha)|\right]\to 0$ as $n\to\infty$, and thus $\E\left[\ell(Y_n^\top\alpha)\right]\to \E\left[\ell(Y^\top\alpha)\right]$ as $n\to\infty$.
    \qedhere
\end{proof}

We are now ready to complete the proof of Lemma \ref{OT-continuity}.

{{\bf Continued proof of Lemma \ref{OT-continuity}}}
Assume $Y_n\in L^q\sim \mathbb{Q}_n$ for $n\in\mathbb{N}$ and $Y\in L^q\sim \mathbb{Q}$ such that $\E\left[\mathbb{D}_{\mathcal{Y}}(Yn,Y)\right]=\mathcal{W}(\mathbb{Q}_n, \mathbb{Q})\rightarrow 0$ as $n\rightarrow\infty$, where $\mathbb{D}_{\mathcal{Y}}(y_1,y_2)=||y_1-y_2||_{\mathcal{Y}}^q$ for any $y_1,y_2\in\mathcal{Y}$ with $q\in[1,\infty)$. We consider the following risk measure case by case.
    \begin{itemize}
        \item[(i)] $\rho(Y)=\E\left[\ell(Y^\top\alpha) \right]$. One can verify that $\ell$ in Propositions \ref{regular_exp} and  \ref{general_exp} is continuous and satisfies that there exists $C>0$ such that $|\ell(x)|\le C(1+|x|^q)$. Applying Lemma \ref{UI_probability_ell(Yalpha)}, we have $\E\left[\ell(Y_n^\top\alpha)\right]\to \E\left[\ell(Y^\top\alpha)\right]$ as $n\to \infty$ if $\mathcal{W}(\mathbb{Q}_n, \mathbb{Q})\rightarrow 0$ as $n\rightarrow\infty$. We conclude that $\rho$ is OT-continuity.
        \item[(ii)] $\rho(Y)=\inf _{t \in \mathbb{R}^k} \mathbb{E}\left[\ell(Y^\top\alpha, t)\right]$. Because $\ell(z,t): \mathbb{R}\times\mathbb{R}^k\rightarrow\mathbb{R}$ in Propositions \ref{general_min_exp} and \ref{pro2*} is convex and level-bounded in $t$ for all $z \in \mathbb{R}$, the infimum is attained. Thus for any $n\in\mathbb{N}$, there exists $t_n\in \mathbb{R}$ such that
            \begin{align}\label{inf_exp_convergence}
                \left|\inf _{t \in \mathbb{R}^k} \mathbb{E}\left[\ell(Y_n^\top\alpha, t)\right]-\inf _{t \in \mathbb{R}^k} \mathbb{E}\left[\ell(Y^\top\alpha, t)\right]\right|\le\left| \mathbb{E}\left[\ell(Y_n^\top\alpha, t_n)\right]-\mathbb{E}\left[\ell(Y^\top\alpha, t_n)\right]\right|.
            \end{align}
            One can verify that $\ell$ in Propositions \ref{general_min_exp} and  \ref{pro2*} is continuous and satisfies that there exists $C>0$ such that $|\ell(x)|\le C(1+|x|^q)$. Applying Lemma \ref{UI_probability_ell(Yalpha)}, we have $\mathbb{E}\left[\ell(Y_n^\top\alpha, t_n)\right]\to \mathbb{E}\left[\ell(Y_n^\top\alpha, t_n)\right]$ as $n\to \infty$ if $\mathcal{W}(\mathbb{Q}_n, \mathbb{Q})\rightarrow 0$ as $n\rightarrow\infty$. Thus, by inequality \eqref{inf_exp_convergence}, we have $\inf _{t \in \mathbb{R}^k} \mathbb{E}\left[\ell(Y_n^\top\alpha, t)\right]\to\inf _{t \in \mathbb{R}^k} \mathbb{E}\left[\ell(Y^\top\alpha, t)\right]$ as $n\to \infty$.  We conclude that $\rho$ is OT-continuity.
        \item[(iii)] $\rho(Y)=\inf _{t \in \mathbb{R}}\left\{t+\left(\mathbb{E}\left[\ell^q(Y^\top\alpha, t)\right]\right)^{1 /q}\right\}$. Because $t+\ell(z,t): \mathbb{R}\times\mathbb{R}\rightarrow\mathbb{R}$ in Propositions \ref{general_min_exp_1} and \ref{pro_rho2_general} is convex and level-bounded in $t$ for all $z \in \mathbb{R}$, the infimum is attained. This applies that there exists $t_n\in \mathbb{R}$ such that\begin{align}\label{inf_q_exp_convergence}              &\left|\inf _{t \in \mathbb{R}^k} \left\{t+\mathbb{E}\left[\ell^q(Y_n^\top\alpha, t)\right]^{1/q}\right\}-\inf _{t \in \mathbb{R}^k} \left\{t+\mathbb{E}\left[\ell^q(Y^\top\alpha, t)\right]^{1/q}\right\}\right|\\\le&\left|\E\left[\ell^q(Y_n^\top\alpha,t_n) \right]^{1/q}-\E\left[\ell^q(Y^\top\alpha,t_n) \right]^{1/q}\right|.            \end{align}
        One can verify that $\ell^q$ in Propositions \ref{general_min_exp} and  \ref{pro2*} is continuous and satisfies that there exists $C>0$ such that $|\ell^q(x)|\le C(1+|x|^q)$. Applying Lemma \ref{UI_probability_ell(Yalpha)}, we have $\mathbb{E}\left[\ell^q(Y_n^\top\alpha, t_n)\right]\to \mathbb{E}\left[\ell^q(Y_n^\top\alpha, t_n)\right]$ as $n\to \infty$ if $\mathcal{W}(\mathbb{Q}_n, \mathbb{Q})\rightarrow 0$ as $n\rightarrow\infty$. Thus, by inequality \eqref{inf_q_exp_convergence}, we have $\inf _{t \in \mathbb{R}^k} \left\{t+\mathbb{E}\left[\ell^q(Y_n^\top\alpha, t)\right]^{1/q}\right\}\to\inf _{t \in \mathbb{R}^k} \left\{t+\mathbb{E}\left[\ell^q(Y^\top\alpha, t)\right]^{1/q}\right\}$ as $n\to \infty$.  We conclude that $\rho$ is OT-continuity.
            \item[(iv)] $\rho(Y)=\inf \left\{\kappa \in \mathbb{R}: \mathbb{E}\left[u(-Y^\top\alpha-\kappa)\right] \le l\right\}$, where $u$ is convex and increasing and $l$ is a fixed constant in the interior of the range of $u$.
            One can verify that $u$ in Propositions \ref{pro:shortfall_1} and  \ref{pro:shortfall_2}, and Remark \ref{q_order_utility_shortfall} is continuous and satisfies that there exists $C>0$ such that $|u(x)|\le C(1+|x|^q)$. Applying Lemma \ref{UI_probability_ell(Yalpha)}, for every $\kappa\in\mathbb{R}$,  we have $\mathbb{E}\left[u(-Y_n^\top\alpha-\kappa)\right]\to \mathbb{E}\left[u(-Y^\top\alpha-\kappa)\right]$ as $n\to \infty$ if $\mathcal{W}(\mathbb{Q}_n, \mathbb{Q})\rightarrow 0$ as $n\rightarrow\infty$. 
            Because $l$ is a fixed constant in the interior of the range of $u$, the sublevel set $\left\{\kappa: \mathbb{E}\left[u(-Y^\top\alpha-\kappa)\right] \le l\right\}$ is nonempty, closed and bounded-below interval, that is $\left\{\kappa: \mathbb{E}\left[u(-Y^\top\alpha-\kappa)\right] \le l\right\}=[\rho(Y),\infty)=[\rho(\mathbb{Q}),\infty)$, where $\rho(Y)=\rho(\mathbb{Q})\in\mathbb{R}$, and for any $\epsilon>0$, we have $\mathbb{E}\left[u(-Y^\top\alpha-\rho(\mathbb{Q})-\epsilon)\right]<l$. Moreover, by the definition of shortfall risk measure, we have $\mathbb{E}\left[u(-Y^\top\alpha-\rho(\mathbb{Q})+\epsilon)\right]>l$ for any $\epsilon>0$. 
            Because of that $\mathbb{E}\left[u(-Y_n^\top\alpha-\kappa)\right]\to \mathbb{E}\left[u(-Y^\top\alpha-\kappa)\right]$ as $n\to \infty$, for sufficiently large $n$, we have $\mathbb{E}\left[u(-Y_n^\top\alpha-\rho(\mathbb{Q})-\epsilon)\right]<l$ and $\mathbb{E}\left[u(-Y_n^\top\alpha-\rho(\mathbb{Q})+\epsilon)\right]>l$. Thus, $\rho(\mathbb{Q}_n)\in(\rho(\mathbb{Q})-\epsilon,\rho(\mathbb{Q})+\epsilon]$ for sufficiently large $n$. Letting $\epsilon$ decrease to $0$ yields $\rho(\mathbb{Q}_n)\rightarrow \rho(\mathbb{Q})$ as $n\rightarrow \infty$. We conclude that $\rho$ is OT-continuity.
            \item[(v)] $\rho(Y)=\rho^h\left[\ell( Y^\top\alpha)\right]$, where $h$ is a convex distortion function. Let $\ell: \mathbb{R} \rightarrow \mathbb{R}$ be a convex and Lipschitz continuous function. By H\"older inequality and the Lipschitz continuous of $\ell$, we have 
            \begin{align*}
                \left|\rho^h\left[\ell(Y_n^\top\alpha) \right]-\rho^h\left[\ell(Y^\top\alpha) \right]\right|&\le \|h_-'\|_{q^*}\E\left[\left|\ell(Y_n^\top\alpha)-\ell(Y^\top\alpha) \right|^q\right]^{1/q}\\
                &\le \|h_-'\|_{q^*}\operatorname{Lip}(\ell)\|\alpha\|_*\E\left[\|Y_n-Y\|^q\right]^{1/q}\\
                &= \|h_-'\|_{q^*}\operatorname{Lip}(\ell)\|\alpha\|_*\mathcal{W}(\mathbb{Q}_n, \mathbb{Q}).
            \end{align*}
            Following the above inequality, we have $\E\left[\ell(Y_n^\top\alpha) \right]\rightarrow\E\left[\ell(Y^\top\alpha) \right]$ as $n\rightarrow\infty$, and thus $\rho$ is OT-continuity.
    \end{itemize}
    \qedhere
\end{proof}

\subsection*{EC.2.2.1. Proofs of Section \ref{CE}}

\vspace{0.1in}
\begin{proof}[{\bf Proof of Proposition \ref{regular_exp}}]
    \begin{itemize}
		\item [(i)]Let $\mathbb{D}_{\mathcal{Y}}(y_1,y_2)=||y_1-y_2||_{\mathcal{Y}}$, $\forall y_1,y_2\in\mathcal{Y}$ and $||\cdot||_*$ be the dual norm of $||\cdot||_{\mathcal{Y}}$. Applying the equation \eqref{regular:p=1}, we find that the problem \eqref{pro1:exp1-1} can be reformulated as
		\begin{align*}
			\min_{\alpha\in\mathcal{A}}\sup_{(p,\delta) \in \mathcal{V},\delta\ge0}\sup_{\mathbb{Q}|_{Y}\in \mathbb{B}_{\delta}\left(\sum_{i=1}^N p_i\mathrm{I}_{\widehat{y}_i}\right)}\E_{\mathbb{Q}|_{Y}}\left[\ell(Y^\top\alpha)\right]
			=\min_{\alpha\in\mathcal{A}}\sup_{(p,\delta) \in \mathcal{V}_+}\sum_{i=1}^N p_i \ell( \widehat{y}_i^\top\alpha)+\operatorname{Lip}(\ell)\delta\|{\alpha}\|_*.
		\end{align*}
        Let $p'=(p,\delta)$ and $$f(p')=f(p,\delta)=\sum_{i=1}^N p_i \ell( \widehat{y}_i^\top\alpha)+\operatorname{Lip}(\ell)\delta\|{\alpha}\|_*.$$ 
		Using the Fenchel duality theory, we can reformulate the above inner supremum problem as
		\begin{align*}
			&\sup_{p'\in\mathcal{V}_+}\sum_{i=1}^N p_i \ell( \widehat{y}_i^\top\alpha)+\operatorname{Lip}(\ell)\delta\|{\alpha}\|_*\\
			=&\sup_{p'}\left\{\sum_{i=1}^N p_i \ell( \widehat{y}_i^\top\alpha)+\operatorname{Lip}(\ell)\delta\|{\alpha}\|_*-\sigma(p'\mid \mathcal{V}_+)\right\} \\
			=&\inf_{v}\left\{\sigma^*(v\mid\mathcal{V}_+)-f_*(v)\right\},
		\end{align*}
		where $\sigma^*(v\mid \mathcal{V}_+)$ is the convex conjugate function of $\sigma(p'|\mathcal{V}_+)$,
		and $f_*(v)$ is the concave conjugate function of $f(p')$,
		\begin{align*}
			f_*(v)=\inf_{p'\ge 0}v^\top p'-\nu(\alpha)^\top p'=\begin{cases}0 & \text { if } v-\nu(\alpha)\ge 0, \\ -\infty & \text { else, }\end{cases}
		\end{align*}
		with $\nu(\alpha)^\top=(\ell(\widehat{y}_1^\top\alpha),\cdots,\ell(\widehat{y}_{N}^\top\alpha),\operatorname{Lip}(\ell)||\alpha||_*)\in\mathbb{R}^{N+1}.$ Thus, we can reformulate the problem \eqref{pro1:exp1-1} as 
		\begin{align*}
			&\inf_{\alpha,v}\sigma^*(v\mid\mathcal{V}_+)\\
			&~\text{s.t.}~\alpha\in\mathcal{A}, v-\nu(\alpha)\ge0.
		\end{align*}
        \item[(ii)]Let $\mathbb{D}_{\mathcal{Y}}(y_1,y_2)=||y_1-y_2||_{\mathcal{Y}}^q$ for any $y_1,y_2\in\mathcal{Y}$ and $q\in(1,\infty)$.
        \begin{itemize}
            \item[(1)] Applying the result in Theorem \ref{W22thm3},  we find that the problem \eqref{pro1:exp1-1} can be reformulated as
		\begin{align*}
			\min_{\alpha\in\mathcal{A}}\sup_{(p,\delta) \in \mathcal{V},\delta\ge0}\sup_{\mathbb{Q}|_{Y}\in \mathbb{B}_{\delta}\left(\sum_{i=1}^N p_i\mathrm{I}_{\widehat{y}_i}\right)}\E_{\mathbb{Q}|_{Y}}\left[\ell(Y^\top\alpha)\right]
        =\min_{\alpha\in\mathcal{A}}\sup_{p'\in\mathcal{V}_+}\sum_{i=1}^N p_i \ell( \widehat{y}_i^\top\alpha)+C\delta^{\frac{1}{q}}\|{\alpha}\|_*.
		\end{align*}
		Let 
		$g(p')=\sum_{i=1}^N p_i \ell( \widehat{y}_i^\top\alpha)+C\delta^\frac{1}{q}\|{\alpha}\|_*.$  
		Using the Fenchel duality theory, we can reformulate the above inner supremum problem as
		\begin{align*}
			\sup_{p'\in\mathcal{V}_+}\sum_{i=1}^N p_i \ell( \widehat{y}_i^\top\alpha)+C\delta^\frac{1}{q}\|{\alpha}\|_*=\inf_{v}\left\{\sigma^*(v\mid\mathcal{V}_+)-g_*(v)\right\},
		\end{align*}
		where $\sigma^*(v\mid \mathcal{V}_+)$ is the convex conjugate function of $\sigma(p'|\mathcal{V}_+)$,
		and $g_*(v)$ is the concave conjugate function of $g(p')$, defined as 
		\begin{align*}
			g_*(v)=&\inf_{p'\ge 0}v^\top p'-\sum_{i=1}^N p_i \ell( \widehat{y}_i^\top\alpha)-C\delta^\frac{1}{q}\|{\alpha}\|_*\\
			=&\begin{cases}-C_qv_{N+1}\left\|\frac{\alpha}{v_{N+1}}\right\|_*^{\frac{q}{q-1}} & \text { if } v-\nu'(\alpha)\ge 0, \\ -\infty & \text { else, }\end{cases}
		\end{align*}
		with $C_q=\left(q^{\frac{1}{1-q}}-q^{\frac{q}{1-q}}\right)C^{\frac{q}{q-1}}$, and $\nu'(\alpha)^\top=(\ell(\widehat{y}_1^\top\alpha),\cdots,\ell(\widehat{y}_{N}^\top\alpha),0)\in\mathbb{R}^{N+1}$. Thus, we can reformulate the problem \eqref{pro1:exp1-1} as 
		\begin{align*}
			&\inf_{\alpha,v}~\sigma^*(v|\mathcal{V}_+)+C_qv_{N+1}\left\|\frac{\alpha}{v_{N+1}}\right\|_*^{\frac{q}{q-1}}\\
			&~\text{s.t.}~\alpha\in\mathcal{A},~v-\nu'(\alpha)\ge0,
		\end{align*}
		\item[(2)] Applying the result in Theorem \ref{W22thm4},  we can reformulate the problem \eqref{pro1:exp1-1} as
		\begin{align*}
			&\min_{\alpha\in\mathcal{A}}\sup_{(p,\delta) \in \mathcal{V},\delta\ge0}\sup_{\mathbb{Q}|_{Y}\in \mathbb{B}_{\delta}\left(\sum_{i=1}^N p_i\mathrm{I}_{\widehat{y}_i}\right)}\E_{{\mathbb{Q}}|_{Y}}\left[\ell^q(Y^\top \alpha) \right]\\
			=&\min_{\alpha\in\mathcal{A}}\sup_{(p,\delta) \in \mathcal{V}_+}\left(\left( \E_{\sum_{i=1}^N p_i\mathrm{I}_{\widehat{y}_i}}\left[\ell^q(Y, \alpha) \right]\right)^{1/q} + C\delta^{\frac{1}{q}} ||\alpha||_*\right)^{q}\\
			=&\min_{\alpha\in\mathcal{A}}\left(\sup_{(p,\delta) \in \mathcal{V}_+}\left( \E_{\sum_{i=1}^N p_i\mathrm{I}_{\widehat{y}_i}}\left[\ell^q(Y, \alpha) \right]\right)^{1/q} +C \delta^{\frac{1}{q}} ||\alpha||_*\right)^{q}\\
			=&\min_{\alpha\in\mathcal{A}}\left(\sup_{p'\in \mathcal{V}_+}\left( \E_{\sum_{i=1}^N p_i\mathrm{I}_{\widehat{y}_i}}\left[\ell^q(Y, \alpha) \right]\right)^{1/q} +C \delta^{\frac{1}{q}} ||\alpha||_*\right)^{q}.
		\end{align*}
		Next, we focus on solving the following problem
		\begin{align*}
			\sup_{p'\in \mathcal{V}_+}\left( \E_{\sum_{i=1}^N p_i\mathrm{I}_{\widehat{y}_i}}\left[\ell^q(Y, \alpha) \right]\right)^{1/q} +C \delta^{\frac{1}{q}} ||\alpha||_*.
		\end{align*}
		Let 
		$h(p')=\left(\sum_{i=1}^N p_i \ell^q( \widehat{y}_i^\top\alpha)\right)^\frac{1}{q}+C\delta^\frac{1}{q}\|{\alpha}\|_*.$ Using the Fenchel duality theory, we can reformulate the above inner supremum problem as
		\begin{align*}
			\sup_{p'}\left\{h(p')-\sigma(p'\mid \mathcal{V}_+)\right\}=\inf_{v}\left\{\sigma^*(v\mid\mathcal{V}_+)-h_*(v)\right\},
		\end{align*}
		where $\sigma^*(v\mid \mathcal{V}_+)$ is the convex conjugate function of $\sigma(p'|\mathcal{V}_+)$,
		and $h_*(v)$ is the concave conjugate function of $h(p')$,
		\begin{align*}
			h_*(v)=&\inf_{p'\ge 0}v^\top p'-h(p')\\
			=&\sup-C_q^1y^{\frac{1}{1-q}}-C_qv^2_{N+1}\left\|\frac{\alpha}{v^2_{N+1}}\right\|_*^{\frac{q}{q-1}}\\
			&~\text{s.t.}~v=y\cdot\nu''(\alpha)+v^2,~y\ge0,~v^2\ge 0,\\
			=&\sup-C_q^1y^{\frac{1}{1-q}}-C_qv_{N+1}\left\|\frac{\alpha}{v_{N+1}}\right\|_*^{\frac{q}{q-1}}\\
			&~\text{s.t.}~v-y\cdot\nu''(\alpha)\ge 0,~y\ge0,
		\end{align*}
		where $C_q=C^{\frac{q}{q-1}}\left(q^{\frac{1}{1-q}}-q^{\frac{q}{1-q}}\right)$, $C_q^1=q^{\frac{1}{1-q}}-q^{\frac{q}{1-q}}$, and \begin{align*}
			\nu''(\alpha)^\top=(\ell^q(\widehat{y}_1^\top\alpha),..., \ell^q(\widehat{y}_N^\top\alpha),0)\in\mathbb{R}^{N+1}.
		\end{align*}
		Thus, we have
		\begin{align*}
			&\min_{\alpha\in\mathcal{A}}\left(\sup_{p'\in \mathcal{V}_+}\left( \E_{\sum_{i=1}^N p_i\mathrm{I}_{\widehat{y}_i}}\left[\ell^q(Y, \alpha) \right]\right)^{1/q} +C \delta^{\frac{1}{q}} ||\alpha||_*\right)^{q}\\
			=&\min_{\alpha\in\mathcal{A}}\left(\inf_{v}\left\{\sigma^*(v\mid\mathcal{V}_+)-h_*(v)\right\}\right)^q\\
			=&\inf_{\alpha\in\mathcal{A},v}\left(\sigma^*(v\mid\mathcal{V}_+)-h_*(v)\right)^q\\
			=&\inf_{\alpha,v,y}\left[\sigma^*(v|\mathcal{V}_+)+C_q^1y^{\frac{1}{1-q}}+C_qv_{N+1}\left\|\frac{\alpha}{v_{N+1}}\right\|_*^{\frac{q}{q-1}}\right]^q\\
			&~~\text{s.t.}~\alpha\in\mathcal{A},~v- y\cdot\nu''(\alpha)\ge 0,~y\ge0.
		\end{align*}
		Thus, we obtain the desired result.
        \end{itemize}
        \end{itemize}
        \qedhere
\end{proof}

The following result, which reduces the worst-case expectation problem via the projection property of \cite{W22}, will be instrumental in simplifying subsequent proofs.

{\begin{theorem}\label{regular-projection}
Let $\mathbb{D}_{\mathcal{Y}}(y_1,y_2)=||y_1-y_2||_{\mathcal{Y}}^q$ for any $y_1,y_2\in\mathcal{Y}$, with $q\in(1,\infty]$. Let $1/q+1/q^*=1$. Assume $\ell: \mathbb{R} \rightarrow \mathbb{R}$ is convex and satisfies $\ell(x)-\ell(x_0) \leqslant L|x-x_0|^q+M, x \in \mathbb{R}$ for some $L,M>0$, and $x_0\in\mathbb{R}$, for any distribution $F_0$, and  $\delta \geqslant 0$.  we have \begin{align}
	&\sup _{{\mathbb{Q}}|_{Y} \in \mathbb{B}_{\delta}(F_0)} \E_{{\mathbb{Q}}|_{Y}}\left[\ell(Y^\top \alpha) \right]\notag\\
    =&\inf _{\lambda \geqslant 0}\E_{F_0}\left[\sup_{y \in \mathbb{R}}\left\{\ell(y)-\lambda|y-Y^\top\alpha|^q\right\}\right]+\lambda \delta\|\alpha\|_*^q\label{S23_regular_0}\\
	=&\inf _{\eta\ge0} \E_{F_0}\left[\sup _{z \in \mathbb{R}}\left\{z \cdot {Y}^{\top} \alpha-\ell^*(z)+\eta \cdot \frac{|z|^{q^*}}{q^* q^{q^*-1}}\right\}\right]+\delta \frac{\|\alpha\|_*^q}{\eta^{q-1}}.\label{S23_regular}
\end{align}
Moreover, we also find that problems \eqref{S23_regular_0} and \eqref{S23_regular} are convex optimization problems.
\end{theorem}
\begin{proof}[{\bf Proof of Theorem \ref{regular-projection}}]
	By the projection result in Proposition 1 of \cite{W22}, we have 
	\begin{align*}
		\sup _{{\mathbb{Q}}|_{Y} \in \mathbb{B}_{\delta}(F_0)} \E_{{\mathbb{Q}}|_{Y}}\left[\ell(Y^\top \alpha) \right]=&\sup _{\left\{{\mathbb{Q}}|_{Y^\top\alpha}: {\mathbb{Q}}|_{Y} \in \mathbb{B}_{\delta}(F_0)\right\}} \E_{{\mathbb{Q}}|_{Y^\top\alpha}}\left[\ell(Y^\top \alpha) \right]\\
		=&\sup _{{\mathbb{Q}}|_{Y^\top\alpha} \in \mathbb{B}_{\delta\|\alpha\|_*}(F_0|_{Y^\top\alpha})} \E_{{\mathbb{Q}}|_{Y^\top\alpha}}\left[\ell(Y^\top \alpha) \right].
	\end{align*}
	Using the strong duality result of Theorem 1 in \cite{G23}, we have that 
	\begin{align*}
		\sup _{{\mathbb{Q}}|_{Y^\top\alpha} \in \mathbb{B}_{\delta\|\alpha\|_*}(F_0|_{Y^\top\alpha})} \E_{{\mathbb{Q}}|_{Y^\top\alpha}}\left[\ell(Y^\top \alpha) \right]=\inf _{\lambda \geqslant 0}\left\{\E_{F_0} \left[\sup _{y \in \mathbb{R}}\left\{\ell(y)-\lambda|y-Y^\top\alpha|^q\right\}\right]+\lambda \delta\|\alpha\|_*^q\right\}.
	\end{align*}
    Thus the first equality holds. Now we focus on the second equality. Since $\ell$ is a real-valued convex function on $\mathbb{R}$, we have $\ell=\ell^{* *}$, i.e., $\ell(y)=\sup _{z \in \mathbb{R}}\left\{z y-\ell^*(z)\right\}$ for all $y \in \mathbb{R}$. Using this relation, \eqref{S23_regular_0} can be reformulated as follows.
	$$
	\begin{aligned}
		& \inf_{\lambda \geqslant 0} \mathbb{E}_{F_0}\left[\sup _{y \in \mathbb{R}}\left\{\ell(y)-\lambda\left|y-Y^{\top} \alpha\right|^q\right\}\right]+\lambda \delta\|\alpha\|_*^q \\
		= & \inf_{\lambda \geqslant 0} \mathbb{E}_{F_0}\left[\sup _{\Delta \in \mathbb{R}}\left\{\ell\left(Y^{\top} \alpha+\Delta\right)-\lambda|\Delta|^q\right\}\right]+\lambda \delta\|\alpha\|_*^q \\
		= & \inf_{\lambda \geqslant 0} \mathbb{E}_{F_0}\left[\sup _{\Delta \in \mathbb{R}}\left\{\sup _{z \in \mathbb{R}}\left\{z\left(Y^{\top} \alpha+\Delta\right)-\ell^*(z)\right\}-\lambda|\Delta|^q\right\}\right]+\lambda \delta\|\alpha\|_*^q \\
		= & \inf_{\lambda \geqslant 0} \mathbb{E}_{F_0}\left[\sup _{z \in \mathbb{R}}\left\{z \cdot Y^{\top} \alpha-\ell^*(z)+\sup _{\Delta \in \mathbb{R}}\left\{z \Delta-\lambda|\Delta|^q\right\}\right\}\right]+\lambda \delta\|\alpha\|_*^q \\
		= & \inf_{\lambda \geqslant 0} \mathbb{E}_{F_0}\left[\sup _{z \in \mathbb{R}}\left\{z \cdot Y^{\top} \alpha-\ell^*(z)+\left(\frac{1}{\lambda q}\right)^{1 /(q-1)}\left(1-\frac{1}{q}\right)|z|^{q /(q-1)}\right\}\right]+\lambda \delta\|\alpha\|_*^q \\
		= & \inf_{\lambda \geqslant 0} \mathbb{E}_{F_0}\left[\sup _{z \in \mathbb{R}}\left\{z \cdot Y^{\top} \alpha-\ell^*(z)+\left(\frac{1}{\lambda}\right)^{q^*-1} \frac{|z|^{q^*}}{q^* q^{q^*-1}}\right\}\right]+\lambda \delta\|\alpha\|_*^q .
	\end{aligned}
	$$
	By setting $\eta=(1 / \lambda)^{1 /(q-1)}$, the above equation reduces to
	$$
	\inf_{\eta \geqslant 0} \mathbb{E}_{F_0}\left[\sup _{z \in \mathbb{R}}\left\{z \cdot Y^{\top} \alpha-\ell^*(z)+\eta \cdot \frac{|z|^{q^*}}{q^* q^{q^*-1}}\right\}\right]+\delta\cdot \frac{\|\alpha\|_*^q}{\eta^{q-1}}.
	$$
	Hence, we have verified that \eqref{S23_regular_0} and \eqref{S23_regular} are equivalent. Let us now prove that \eqref{S23_regular} is a convex optimization problem. The first term in \eqref{S23_regular} is convex since the mapping $\eta \mapsto z \cdot Y^{\top} \alpha-\ell^*(z)+$ $\eta|z|^q /\left(q p^{q-1}\right)$ is convex for any fixed $z \in \mathbb{R}$ and any fixed $\alpha \in \mathcal{A}$. Thus, taking the spremum of $z$ is also convex. Also noting that $\|\alpha\|_*^p / \eta^{p-1}$ is a convex function in $\eta \ge 0$.  Therefore, we conclude that \eqref{S23_regular_0} is a convex optimization problem. This completes the proof. \qedhere
\end{proof}

\begin{proof}[{\bf Proof of Proposition \ref{general_exp}}]
    Applying Theorem \ref{regular-projection}, we find that problem \eqref{pro1:exp1-1} can be reformulated as
		\begin{align*}
			&\min_{\alpha\in\mathcal{A}}\sup_{(p,\delta) \in \mathcal{V},\delta\ge0}\sup_{\mathbb{Q}|_{Y}\in \mathbb{B}_{\delta(p,\epsilon)}\left(\sum_{i=1}^N p_i\mathrm{I}_{\widehat{y}_i}\right)}\E_{\mathbb{Q}|_{Y}}\left[\ell(Y^\top\alpha)\right]\\
			=&\min_{\alpha\in\mathcal{A}}\sup_{(p,\delta) \in \mathcal{V}_+}\inf _{\eta\ge0} \sum_{i=1}^Np_i\cdot\sup _{z \in \mathbb{R}}\left\{z \cdot \widehat{y}_i^{\top} \alpha-\ell^*(z)+\eta \cdot \frac{|z|^{q^*}}{q^* q^{q^*-1}}\right\}+\delta \frac{\|\alpha\|_*^q}{\eta^{q-1}}\\
			=&\min_{\alpha\in\mathcal{A}}\sup_{(p,\delta^) \in \mathcal{V}_+}\inf _{\eta\ge0} \sum_{i=1}^Np_i\cdot\sup _{z \in \mathbb{R}}\left\{z \cdot \widehat{y}_i^{\top} \alpha-\ell^*(z)+\eta \cdot \frac{|z|^{q^*}}{q^* q^{q^*-1}}\right\}+\delta\frac{\|\alpha\|_*^q}{\eta^{q-1}}.
		\end{align*}
		One can verify that the objective function of the above optimization problem is linear in $(p,\delta)\in \mathcal{V}_+$ and convex in $\eta\ge0$, and that $\mathcal{V}_+$ is a convex compact set. Applying min-max theorem in \cite{S58}, we have 
		\begin{align*}
			&\min_{\alpha\in\mathcal{A}}\sup_{(p,\delta) \in \mathcal{V}_+}\inf _{\eta\ge0} \sum_{i=1}^Np_i\cdot\sup _{z \in \mathbb{R}}\left\{z \cdot \widehat{y}_i^{\top} \alpha-\ell^*(z)+\eta \cdot \frac{|z|^{q^*}}{q^* q^{q^*-1}}\right\}+\delta\frac{\|\alpha\|_*^q}{\eta^{q-1}}\\
			=&\inf_{\alpha\in\mathcal{A},\eta\ge0}\sup_{(p,\delta) \in \mathcal{V}_+} \sum_{i=1}^Np_i\cdot\sup _{z \in \mathbb{R}}\left\{z \cdot \widehat{y}_i^{\top} \alpha-\ell^*(z)+\eta \cdot \frac{|z|^{q^*}}{q^* q^{q^*-1}}\right\}+\delta\frac{\|\alpha\|_*^q}{\eta^{q-1}}.
		\end{align*}
		Let $f(p,\delta)=\sum_{i=1}^Np_i\cdot \sup _{z \in \mathbb{R}}\left\{z \cdot \widehat{y}_i^{\top} \alpha-\ell^*(z)+\eta \cdot \frac{|z|^{q^*}}{q^* q^{q^*-1}}\right\}+\delta\frac{\|\alpha\|_*^q}{\eta^{q-1}}$. Similar to the proof of Proposition \ref{regular_exp}, using Fenchel duality, we obtain the desired result. \qedhere
\end{proof}

\subsection*{EC.2.2.2. Proofs of Section \ref{Expectation Form}}

In the following, we give a minimax lemma which will be used in the proof of propositions in Section \ref{Expectation Form}.\begin{lemma}\label{high_dimen_min_max}    Let $\mathbb{D}_{\mathcal{Y}}(y_1,y_2)=||y_1-y_2||_{\mathcal{Y}}^q$ for any $y_1,y_2\in\mathcal{Y}$ and $q\in[1,\infty)$, $\delta\ge 0$, and $\mathbb{Q}_0$ be a distribution on $\mathcal{Y}$. We have the following results.
\begin{itemize}
    \item[(i)]Assume the loss function $\ell(z,t):\mathbb{R}\times\mathbb{R}^k\rightarrow \mathbb{R}$ is the form $\ell(z,t)=\sum_{q_1\in\mathcal{I}}C_{q_1}\left(\bar{\ell}_{q_1}(z,t)\right)^{q_1}$, with $\mathcal{I}\subset[1,q]$ and $C_{q_1}\ge 0$, where $\bar{\ell}_{q_1}(z,t)$ is nonnegative on $\mathbb{R}^{k+1}$, convex and level bounded in $t\in\mathbb{R}^k$ for all $z \in \mathbb{R}$, and Lipschitz continuous in $z$ for  all $t \in \mathbb{R}^k$ with a uniform Lipschitz constant $\operatorname{Lip}(\bar{\ell}_{q_1})$, for $q_1\in[1,q]$. Then we have     \begin{align*}&\sup_{{\mathbb{Q}}|_{Y} \in 	\mathbb{B}_{\delta}\left(\mathbb{Q}_0\right)}\inf_{t\in\mathbb{R}^k}\E_{\mathbb{Q}|_Y}\left[\ell(Y^\top\alpha,t)\right]=\inf_{t\in\mathbb{R}^k}\sup_{{\mathbb{Q}}|_{Y} \in 	\mathbb{B}_{\delta}\left(\mathbb{Q}_0\right)}\E_{\mathbb{Q}|_Y}\left[\ell(Y^\top\alpha,t)\right].	\end{align*} 
    \item[(ii)] Let $k=1$. Assume the loss function $\ell(z,t):\mathbb{R}\times\mathbb{R}\rightarrow \mathbb{R}$ is nonnegative on $\mathbb{R}^{2}$, convex in $t\in\mathbb{R}$ for all $z \in \mathbb{R}$, and Lipschitz continuous in $z$ for  all $t \in \mathbb{R}$ with a uniform Lipschitz constant $\operatorname{Lip}({\ell})$. Moreover, assume $t+\ell(z,t)$ is level bounded in $t\in\mathbb{R}$ for all $z \in \mathbb{R}$. Then we have     \begin{align*}&\sup_{{\mathbb{Q}}|_{Y} \in 	\mathbb{B}_{\delta}\left(\mathbb{Q}_0\right)}\inf_{t\in\mathbb{R}}\left\{t+\left(\E_{\mathbb{Q}|_Y}\left[\ell^q(Y^\top\alpha,t)\right]\right)^{1/q}\right\}=\inf_{t\in\mathbb{R}}\sup_{{\mathbb{Q}}|_{Y} \in 	\mathbb{B}_{\delta}\left(\mathbb{Q}_0\right)}\left\{t+\left(\E_{\mathbb{Q}|_Y}\left[\ell^q(Y^\top\alpha,t)\right]\right)^{1/q}\right\}.	\end{align*} 
\end{itemize}\end{lemma}
\begin{proof}[{\bf Proof of Lemma \ref{high_dimen_min_max}}]
     \begin{itemize}
         \item[(i)]Denoted by $\pi^1_\ell\left(\mathbb{Q}|_Y,t\right)=\E_{\mathbb{Q}|_Y}\left[\ell(Y^\top\alpha,t)\right]$. First, we show three facts below. (a) $\pi^1_\ell\left(\mathbb{Q}|_Y,t\right)$ is concave in $\mathbb{Q}|_Y$ for all $t \in \mathbb{R}^k$; (b) $\pi^1_\ell\left(\mathbb{Q}|_Y,t\right)$ is convex in $t$ for all $\mathbb{Q}|_Y$; (c) $\pi^1_\ell\left(\mathbb{Q}|_Y,t\right)$ is level bounded in $t$ for all $\mathbb{Q}|_Y$. The fact (a) is trivial. Note that $\ell$ is nonnegative on $\mathbb{R}^{k+1}$ and convex in $t\in\mathbb{R}^k$ for all $z\in\mathbb{R}$.  For any distribution $\mathbb{Q|}_Y$ on $\mathcal{Y}$, $\lambda \in[0,1]$ and $t_1, t_2 \in \mathbb{R}^k$, it holds that
    \begin{align*}
\E_{\mathbb{Q}|_Y}\left[\ell(Y^\top\alpha,\lambda t_1+(1-\lambda)t_2)\right]&\le \E_{\mathbb{Q}|_Y}\left[\lambda\ell(Y^\top\alpha,t_1)+(1-\lambda)\ell(Y^\top\alpha,t_2)\right]\\&=\lambda\E_{\mathbb{Q}|_Y}\left[\ell(Y^\top\alpha,t_1\right]+(1-\lambda)\E_{\mathbb{Q}|_Y}\left[\ell(Y^\top\alpha,t_2)\right],
    \end{align*}
where the inequality from the convexity of $\ell$ with respect to $t\in\mathbb{R}^k$. This implies (b). We claim that $\pi^1_\ell\left(\mathbb{Q}|_Y,t\right)$ is level bounded in $t$ for all $\mathbb{Q}|_Y$. Otherwise, for fixed $\mathbb{Q}|_Y$, there exists $A\in\mathbb{R}$ such that $S_A:=\{t\in\mathbb{R}^k:\pi^1_\ell\left(\mathbb{Q}|_Y,t\right)\le A\}$ is unbounded. Then for each integer $n\ge 1$, there exist $t_n\in S_A$ satisfies that $\|t_n\|\ge n$ for arbitrary norm in $\mathbb{R}^k$. For any $q_1\in\mathcal{I}$, because $\bar{\ell}_{q_1}(z,t)$ is level bounded for any $z\in \mathbb{R}$, we have $\bar{\ell}_{q_1}(z,t_n)\rightarrow\infty$ when $n\rightarrow \infty$. In fact, if there exists $M\in\mathbb{R}$ such that $\bar{\ell}_{q_1}(z,t_n)\le M$ for any $n$, one can verify that the set $\{t:\bar{\ell}_{q_1}(z,t)\le M\}$ is unbounded, which is contradictory to that $\bar{\ell}_{q_1}(z,t)$ is level bounded in $t$ for any $z$. And then $\bar{\ell}_{q_1}(z,t_n)\rightarrow\infty$ when $n\rightarrow \infty$, implies that $\ell(z,t_n)\rightarrow\infty$ when $n\rightarrow \infty$ for any $z\in\mathbb{R}$ due to the form of $\ell$. By nonnegativity of $\ell$, using Fatou lemma. we have 
\begin{align*}
    \liminf_{n\rightarrow\infty}\pi^1_\ell\left(\mathbb{Q}|_Y,t_n\right)=\liminf_{n\rightarrow\infty}\E_{\mathbb{Q}|_Y}\left[\ell(Y^\top\alpha,t_n)\right]\ge \E_{\mathbb{Q}|_Y}\left[\liminf_{n\rightarrow\infty}\ell(Y^\top\alpha,t_n)\right]=\infty.
\end{align*}
This is contradictory to that $\liminf_{n\rightarrow\infty}\pi^1_\ell\left(\mathbb{Q}|_Y,t_n\right)\le A$. Hence, we conclude the proof of (c). By (b), we find that $\pi^1_\ell\left(\mathbb{Q}|_Y,t\right)$ is lower semicontinuous in $t$. Combining this fact with (c), and applying Weierstrass theorem, there exists a level set $S_{A_0}$ of $\pi^1_\ell\left(\mathbb{Q}|_Y,t\right)$ such that all minimizers of the problem $\inf _{t \in \mathbb{R}} \pi^1_{\ell}(\mathbb{Q}|_Y, t)$ are contained in it. The set $S_{A_0}$ is closed and bounded because that $\pi^1_\ell\left(\mathbb{Q}|_Y,t\right)$ is lower semicontinuous and level bounded, and thus, $S_{A_0}$ is a compact set. Then, the set of all minimizers of the problem $\inf _{t \in \mathbb{R}^k} \pi^1_{\ell}(\mathbb{Q}|_Y, t)$ is a subset of a compact set. Denote by $t(\mathbb{Q}|_Y):=\arg\inf_{\|t\|} \{\arg \min _t \pi^1_{\ell}(\mathbb{Q}|_Y, t)\}$. We will show that $\left\{t(\mathbb{Q}|_Y): \mathbb{Q}|_Y \in \mathbb{B}_\delta\left(\mathbb{Q}_0\right)\right\}$ is a subset of a compact set. For any $\mathbb{Q}|_Y \in \mathbb{B}_\delta\left(\mathbb{Q}_0\right)$ and $t\in\mathbb{R}^k$, let $Y\sim\mathbb{Q}|_Y$ and $Y_0\sim\mathbb{Q}_0$, and $q_j\in\mathcal{I}$, we have
\begin{align}\label{higher_dimen_1}
   &\left|\left(\E\left[C_{q_j}\left(\bar{\ell}_{q_j}(Y^\top\alpha,t)\right)^{q_j}\right]\right)^{1/{q_j}}-\left(\E\left[C_{q_j}\left(\bar{\ell}_{q_j}(Y_0^\top\alpha,t)\right)^{q_j}\right]\right)^{1/{q_j}}\right| \notag\\
   \le& C_{q_j}\left(\mathbb{E}\left[\left|\bar{\ell}_{q_j}(Y^\top\alpha,t)-\bar{\ell}_{q_j}(Y_0^\top\alpha,t)\right|^{q_j}\right]\right)^{1 / {q_j}}\notag\\
   \le &C_{q_j}\left(\mathbb{E}\left[\operatorname{Lip}(\bar{\ell}_{q_j})\left|(Y-Y_0)^\top\alpha\right|^{q_j}\right]\right)^{1 / {q_j}}\notag\\
   \le &C_{q_j}\operatorname{Lip}(\bar{\ell}_{q_j})\left(\mathbb{E}\left[\left|(Y-Y_0)^\top\alpha\right|^{q}\right]\right)^{1 / {q}}\notag\\
   \le &C_{q_j}\operatorname{Lip}(\bar{\ell}_{q_j})\|\alpha\|_*\left(\mathbb{E}\left[\left\|Y-Y_0\right\|^q\right]\right)^{1 /q} \le C_{q_j}\operatorname{Lip}(\bar{\ell}_{q_j})\|\alpha\|_* \delta^{1/q},
\end{align}
where the first, third and fourth inequalities from the triangle inequality, Jensen's inequality and \" Holder inequality, respectively, and we use the definition of OT ball in the last inequality. From \eqref{higher_dimen_1}, we have that
\begin{align}\label{higher_dimen_2}
    \pi^1_\ell\left(\mathbb{Q}|_Y,t\left(\mathbb{Q}_0\right)\right)\le \sum_{j=1,2}\left(\left(\E\left[C_{q_j}\left(\bar{\ell}_{q_j}(Y_0^\top\alpha,t\left(\mathbb{Q}_0\right))\right)^{q_j}\right]\right)^{1/{q_j}}+C_{q_j}\operatorname{Lip}(\bar{\ell}_{q_j})\|\alpha\|_* \delta^{1/q}\right)^{{q_j}}.
\end{align}
Note that $\left(\E\left[C_{q_j}\left(\bar{\ell}_{q_j}(Y_0^\top\alpha,t)\right)^{q_j}\right]\right)^{1/{q_j}} \rightarrow \infty$ as $\|t\| \rightarrow \infty$ for $q_j\in\mathcal{I}$ due to that $\bar{\ell}_{q_j}(z,t)$ is  level bounded. There exists $\Delta>0$ such that $$\left(\E\left[C_{q_j}\left(\bar{\ell}_{q_j}(Y_0^\top\alpha,t)\right)^{q_j}\right]\right)^{1/{q_j}}\ge C_{q_j}\operatorname{Lip}(\bar{\ell}_{q_j})\|\alpha\|_* \delta^{1/q},$$ and $$\left(\E\left[C_{q_j}\left(\bar{\ell}_{q_j}(Y_0^\top\alpha,t)\right)^{q_j}\right]\right)^{1/{q_j}}>\left(\E\left[C_{q_j}\left(\bar{\ell}_{q_j}(Y_0^\top\alpha,t\left(\mathbb{Q}_0\right))\right)^{q_j}\right]\right)^{1/{q_j}}+2C_{q_j}\operatorname{Lip}(\bar{\ell}_{q_j})\|\alpha\|_* \delta^{1/q},$$ for all $t \notin\{t:\|t-t(\mathbb{Q}_0)\|\le\Delta\}$ and $q_j\in\mathcal{I}$. This, combined with \eqref{higher_dimen_1}, imply that
\begin{align}\label{higher_dimen_3}
 \pi^1_{\ell}(\mathbb{Q}|_Y, t) \ge& \sum_{j=1,2}\left(\left(\E\left[C_{q_j}\left(\bar{\ell}_{q_j}(Y_0^\top\alpha,t)\right)^{q_j}\right]\right)^{1/{q_j}}-C_{q_j}\operatorname{Lip}(\bar{\ell}_{q_j})\|\alpha\|_* \delta^{1/q}\right)^{{q_j}}\notag\\
 >&\sum_{j=1,2}\left(\left(\E\left[C_{q_j}\left(\bar{\ell}_{q_j}(Y_0^\top\alpha,t\left(\mathbb{Q}_0\right))\right)^{q_j}\right]\right)^{1/{q_j}}+C_{q_j}\operatorname{Lip}(\bar{\ell}_{q_j})\|\alpha\|_* \delta^{1/q}\right)^{{q_j}},   
\end{align}
for all $t \notin\{t:\|t-t(\mathbb{Q}_0)\|\le\Delta\}$. Applying \eqref{higher_dimen_2} and \eqref{higher_dimen_3}, we have $\left\{t(\mathbb{Q}|_Y): \mathbb{Q}|_Y \in \mathbb{B}_\delta\left(\mathbb{Q}_0\right)\right\} \subseteq\{t:\|t-t(\mathbb{Q}_0)\|\le\Delta\}$. Using the minimax theorem in \cite{S58}, it holds that
$$
\begin{aligned}
\sup_{{\mathbb{Q}}|_{Y} \in \mathbb{B}_{\delta}\left(\mathbb{Q}_0\right)}\inf_{t\in\mathbb{R}^k} \pi^1_{\ell}(\mathbb{Q}|_Y, t) & =\sup_{{\mathbb{Q}}|_{Y} \in \mathbb{B}_{\delta}\left(\mathbb{Q}_0\right)}\inf_{t\in\{t:\|t-t(\mathbb{Q}_0)\|\le\Delta\}} \pi^1_{\ell}(\mathbb{Q}|_Y, t)\\
& =\inf_{t\in\{t:\|t-t(\mathbb{Q}_0)\|\le\Delta\}}\sup_{{\mathbb{Q}}|_{Y} \in \mathbb{B}_{\delta}\left(\mathbb{Q}_0\right)} \pi^1_{\ell}(\mathbb{Q}|_Y, t) \ge \inf_{t\in\mathbb{R}^k}\sup_{{\mathbb{Q}}|_{Y} \in \mathbb{B}_{\delta}\left(\mathbb{Q}_0\right)} \pi^1_{\ell}(\mathbb{Q}|_Y, t).
\end{aligned}
$$
The converse direction is trivial. Hence, we complete the proof of (i).
\item[(ii)] The proof is similar to (i).  
     \end{itemize}
     \qedhere
\end{proof}

\begin{proof}[{\bf Proof of Proposition \ref{general_min_exp}}]
    \begin{itemize}
			\item[(i)] Let $\mathbb{D}_{\mathcal{Y}}(y_1,y_2)=||y_1-y_2||_{\mathcal{Y}}$, $\forall y_1,y_2\in\mathcal{Y}$. Applying Lemma \ref{high_dimen_min_max}, equation \eqref{regular:p=1} and minimax theorem, problem \eqref{r11} with \eqref{r2} can be reformulated as 
			\begin{align*}
				&\min_{\alpha\in\mathcal{A}}\sup_{(p,\delta) \in \mathcal{V}}\sup_{\mathbb{Q}|_{Y}\in \mathbb{B}_{\delta}\left(\sum_{i=1}^N p_i\mathrm{I}_{\widehat{y}_i}\right)}\rho^{(1)}_{\mathbb{Q}|_{Y}}\left(Y^\top\alpha\right)\\
				=&\min_{\alpha\in\mathcal{A}}\sup_{(p,\delta) \in \mathcal{V}}\sup_{\mathbb{Q}|_{Y}\in \mathbb{B}_{\delta}\left(\sum_{i=1}^N p_i\mathrm{I}_{\widehat{y}_i}\right)}\inf_{t\in\mathbb{R}^k}\mathbb{E}_{\mathbb{Q}|_{Y}}\left[\ell(Y^\top\alpha, t)\right]\\
				=&\min_{\alpha\in\mathcal{A}}\sup_{(p,\delta) \in \mathcal{V}}\inf_{t\in\mathbb{R}^k}\sup_{\mathbb{Q}|_{Y}\in \mathbb{B}_{\delta}\left(\sum_{i=1}^N p_i\mathrm{I}_{\widehat{y}_i}\right)}\mathbb{E}_{\mathbb{Q}|_{Y}}\left[\ell(Y^\top\alpha, t)\right]\\
				=&\min_{\alpha\in\mathcal{A}}\sup_{(p,\delta) \in \mathcal{V},\delta\ge0}\inf_{t\in\mathbb{R}^k}\mathbb{E}_{\sum_{i=1}^N p_i\mathrm{I}_{\widehat{y}_i}}\left[\ell(Y^\top\alpha, t)\right]+\operatorname{Lip}(\ell)\delta||\alpha||_*\\
				=&\min_{\alpha\in\mathcal{A}}\inf_{t\in\mathbb{R}^k}\sup_{(p,\delta) \in \mathcal{V},\delta\ge0}\mathbb{E}_{\sum_{i=1}^N p_i\mathrm{I}_{\widehat{y}_i}}\left[\ell(Y^\top\alpha, t)\right]+\operatorname{Lip}(\ell)\delta ||\alpha||_*\\
                =&\min_{\alpha\in\mathcal{A}}\inf_{t\in\mathbb{R}^k}\sup_{(p,\delta) \in \mathcal{V}_+}\mathbb{E}_{\sum_{i=1}^N p_i\mathrm{I}_{\widehat{y}_i}}\left[\ell(Y^\top\alpha, t)\right]+\operatorname{Lip}(\ell)\delta ||\alpha||_*.
			\end{align*}
			Applying the Fenchel duality, we have the desired result.
            \item[(ii)] Applying Corollary 1 of \cite{W22} and minimax theorem, we find that problem \eqref{r11} with \eqref{r2} can be reformulated as
		\begin{align*}
			&\min_{\alpha\in\mathcal{A}}\sup_{(p,\delta) \in \mathcal{V}}\sup_{\mathbb{Q}|_{Y}\in \mathbb{B}_{\delta}\left(\sum_{i=1}^N p_i\mathrm{I}_{\widehat{y}_i}\right)}\rho^{(1)}_{\mathbb{Q}|_{Y}}\left(Y^\top\alpha\right)\\
			=&\min_{\alpha\in\mathcal{A}}\sup_{(p,\delta) \in \mathcal{V},\delta\ge0}\left(\left( \rho^{(1)}_{\sum_{i=1}^N p_i\mathrm{I}_{\widehat{y}_i}}\left(Y^\top\alpha \right)\right)^{1/q} + C\delta^{\frac{1}{q}} ||\alpha||_*\right)^{q}\\
			=&\min_{\alpha\in\mathcal{A}}\sup_{(p,\delta) \in \mathcal{V}_+}\left(\left( \inf_{t\in\mathbb{R}}\E_{\sum_{i=1}^N p_i\mathrm{I}_{\widehat{y}_i}}\left[\ell(Y^\top\alpha,t) \right]\right)^{1/q} + C\delta^{\frac{1}{q}} ||\alpha||_*\right)^{q}\\
			=&\min_{\alpha\in\mathcal{A}}\left(\sup_{(p,\delta) \in \mathcal{V}_+}\inf_{t\in\mathbb{R}}\left( \E_{\sum_{i=1}^N p_i\mathrm{I}_{\widehat{y}_i}}\left[\ell(Y^\top\alpha,t) \right]\right)^{1/q} + C\delta^{\frac{1}{q}} ||\alpha||_*\right)^{q}\\
			=&\min_{\alpha\in\mathcal{A}}\left(\inf_{t\in\mathbb{R}}\sup_{(p,\delta) \in \mathcal{V}_+}\left( \E_{\sum_{i=1}^N p_i\mathrm{I}_{\widehat{y}_i}}\left[\ell(Y^\top\alpha,t) \right]\right)^{1/q} + C\delta^{\frac{1}{q}} ||\alpha||_*\right)^{q}\\
			=&\inf_{\alpha\in\mathcal{A},t\in\mathbb{R}}\left(\sup_{(p,\delta) \in \mathcal{V}_+}\left( \E_{\sum_{i=1}^N p_i\mathrm{I}_{\widehat{y}_i}}\left[\ell(Y^\top\alpha,t) \right]\right)^{1/q} + C\delta^{\frac{1}{q}} ||\alpha||_*\right)^{q}.
		\end{align*}
		Thus we only need to consider the inner problem
		\begin{align*}
			\sup_{(p,\delta) \in \mathcal{V}_+}\left( \E_{\sum_{i=1}^N p_i\mathrm{I}_{\widehat{y}_i}}\left[\ell(Y^\top\alpha,t) \right]\right)^{1/q} + C\delta^{\frac{1}{q}} ||\alpha||_*,  
		\end{align*} for $q\in(1,\infty)$. 
		Similar to the proof of Proposition \ref{regular_exp}, we can apply the Fenchel duality theory to the above problem and then obtain the desired result. 
    \end{itemize}
    \qedhere
\end{proof}

\begin{proof}[{\bf Proof of Proposition \ref{pro2*}}]
     Applying Lemma \ref{high_dimen_min_max}, Theorem \ref{regular-projection} and minimax theorem in \cite{S58}, problem \eqref{r11} with \eqref{r2} can be reformulated as 
			\begin{align*}
				&\min_{\alpha\in\mathcal{A}}\sup_{(p,\delta) \in \mathcal{V}}\sup_{\mathbb{Q}|_{Y}\in \mathbb{B}_{\delta}\left(\sum_{i=1}^N p_i\mathrm{I}_{\widehat{y}_i}\right)}\rho^{(1)}_{\mathbb{Q}|_{Y}}\left(Y^\top\alpha\right)\\
				=&\min_{\alpha\in\mathcal{A}}\sup_{(p,\delta) \in \mathcal{V}}\inf_{t\in\mathbb{R}^k}\sup_{\mathbb{Q}|_{Y}\in \mathbb{B}_{\delta}\left(\sum_{i=1}^N p_i\mathrm{I}_{\widehat{y}_i}\right)}\E_{\mathbb{Q}|_Y}\left[\ell(Y^\top\alpha,t)\right]\\
				=&\min_{\alpha\in\mathcal{A}}\sup_{(p,\delta) \in \mathcal{V},\delta\ge0}\inf_{t\in\mathbb{R}^k,\eta\ge0}\sum_{i=1}^Np_i\cdot\sup _{z \in \mathbb{R}}\left\{z \cdot \widehat{y}_i^{\top} \alpha-(\ell(\cdot,t))^*(z)+\eta \cdot \frac{|z|^{q^*}}{q^* q^{q^*-1}}\right\}+\delta \frac{\|\alpha\|_*^q}{\eta^{q-1}}\\
				=&\inf_{\alpha\in\mathcal{A},t\in\mathbb{R}^k,\eta\ge0}\sup_{(p,\delta) \in \mathcal{V}_+}\sum_{i=1}^Np_i\cdot\sup _{z \in \mathbb{R}}\left\{z \cdot \widehat{y}_i^{\top} \alpha-(\ell(\cdot,t))^*(z)+\eta \cdot \frac{|z|^{q^*}}{q^* q^{q^*-1}}\right\}+\delta \frac{\|\alpha\|_*^q}{\eta^{q-1}}.
			\end{align*}
			Applying the Fenchel duality, we have the desired result. \qedhere
\end{proof}

\begin{proof}[{\bf Proof of Example \ref{re:variance}}]
    The robustified conditional  mean-variance portfolio allocation problem in \cite{N24}  is as follows: \begin{align*}
\min _{\alpha \in \mathcal{A}} \sup _{\substack{\mathbb{Q} \in \mathbb{B}_{\delta_0}(\widehat{\mathbb{P}}) \\ \mathbb{Q}\left(X \in \mathcal{N}_\gamma(x_0)\right) \in \mathcal{N}_{\omega}}} \operatorname{Variance}_{\mathbb{Q}}\left[Y^{\top} \alpha \mid X \in \mathcal{N}_\gamma\left(x_0\right)\right]-\theta \cdot \mathbb{E}_{\mathbb{Q}}\left[Y^{\top} \alpha \mid X \in \mathcal{N}_\gamma\left(x_0\right)\right],    
\end{align*}
where $\mathcal{N}_{\omega}=[\epsilon_0,1]$, $\epsilon_0\in(0,1]$ and $\theta\ge 0$. Suppose in addition that $\mathcal{X}=\mathbb{R}^{n_1}, \mathcal{Y}=\mathbb{R}^{n_2}, \mathbb{D}_{\mathcal{X}}(x, \widehat{x})=\|x-\widehat{x}\|^2$ and $\mathbb{D}_{\mathcal{Y}}(y, \widehat{y})=\|y-\widehat{y}\|_2^2$. By Theorem \ref{thm1} and the minimization representation of variance, one can verify that the problem \eqref{variance} can be reformulated as 
\begin{align*}
&\min _{\alpha \in \mathcal{A}} \sup_{(p,\delta) \in {\cal V}} \sup_{{\mathbb{Q}}|_{Y} \in 
	\mathbb{B}_{\delta}\left(\sum_{i=1}^N p_i\mathrm{I}_{\widehat{y}_i}\right)} 
\operatorname{Variance}_{\mathbb{Q}|_Y}\left[Y^{\top} \alpha \right]-\theta \cdot \mathbb{E}_{\mathbb{Q}|_Y}\left[Y^{\top} \alpha \right]\\
=&\min _{\alpha \in \mathcal{A}} \sup_{(p,\delta) \in {\cal V}} \sup_{{\mathbb{Q}}|_{Y} \in 
	\mathbb{B}_{\delta}\left(\sum_{i=1}^N p_i\mathrm{I}_{\widehat{y}_i}\right)}\inf_{t\in\mathbb{R}}\E_{\mathbb{Q}|_Y}\left[\left(Y^{\top} \alpha-t\right)^2-\theta\cdot Y^{\top} \alpha\right]\\
:=&\min _{\alpha \in \mathcal{A}} \sup_{(p,\delta) \in {\cal V}} \sup_{{\mathbb{Q}}|_{Y} \in 
	\mathbb{B}_{\delta}\left(\sum_{i=1}^N p_i\mathrm{I}_{\widehat{y}_i}\right)}\inf_{t\in\mathbb{R}}\E_{\mathbb{Q}|_Y}\left[\ell(Y^\top\alpha,t)\right],
	\end{align*}
	where $\ell(z,t)=\left(z-t\right)^2-\theta z$. By Lemma \ref{high_dimen_min_max}, we have that 
	\begin{align*}
&\sup_{{\mathbb{Q}}|_{Y} \in 
	\mathbb{B}_{\delta}\left(\sum_{i=1}^N p_i\mathrm{I}_{\widehat{y}_i}\right)}\inf_{t\in\mathbb{R}}\E_{\mathbb{Q}|_Y}\left[\ell(Y^\top\alpha,t)\right]=\inf_{t\in\mathbb{R}}\sup_{{\mathbb{Q}}|_{Y} \in 
	\mathbb{B}_{\delta}\left(\sum_{i=1}^N p_i\mathrm{I}_{\widehat{y}_i}\right)}\E_{\mathbb{Q}|_Y}\left[\ell(Y^\top\alpha,t)\right].
	\end{align*}
 Due to the above equation, we can reformulate problem \eqref{variance} as
	\begin{align*}
\min _{\alpha \in \mathcal{A}} \sup_{(p,\delta) \in {\cal V}} \inf_{t\in\mathbb{R}}\sup_{{\mathbb{Q}}|_{Y} \in 
	\mathbb{B}_{\delta}\left(\sum_{i=1}^N p_i\mathrm{I}_{\widehat{y}_i}\right)}\E_{\mathbb{Q}|_Y}\left[\ell(Y^\top\alpha,t)\right].
	\end{align*}
	Applying Theorem \ref{regular-projection} and min-max theorem, the problem \eqref{variance} can be reformulated as
	\begin{align*}
&\min _{\alpha \in \mathcal{A}} \sup_{(p,\delta) \in {\cal V}} \inf_{t\in\mathbb{R}}\sup_{{\mathbb{Q}}|_{Y} \in 
	\mathbb{B}_{\delta}\left(\sum_{i=1}^N p_i\mathrm{I}_{\widehat{y}_i}\right)}\E_{\mathbb{Q}|_Y}\left[\ell(Y^\top\alpha,t)\right]\\
=&\min_{\alpha\in\mathcal{A}}\sup_{(p,\delta) \in {\cal V},\delta \ge 0}\inf_{t\in\mathbb{R},\eta\ge0}\sum_{i=1}^Np_i\cdot\sup _{z \in \mathbb{R}}\left\{z \cdot \widehat{y}_i^{\top} \alpha-\ell^*(z,t)+\eta \cdot \frac{|z|^{q^*}}{q^* q^{q^*-1}}\right\}+\delta \frac{\|\alpha\|_*^q}{\eta^{q-1}}\\
=&\inf_{\alpha\in\mathcal{A},t\in\mathbb{R},\eta\ge0}\sup_{(p,\delta) \in \mathcal{V}_+}\sum_{i=1}^Np_i\cdot\sup _{z \in \mathbb{R}}\left\{z \cdot \widehat{y}_i^{\top} \alpha-\ell^*(z,t)+\eta \cdot \frac{z^2}{4}\right\}+\delta \frac{\|\alpha\|_*^2}{\eta}\\
=&\inf_{\alpha\in\mathcal{A},t\in\mathbb{R},\eta\ge0}\sup_{(p,\delta) \in \mathcal{V}_+}\sum_{i=1}^Np_i\cdot\sup _{z \in \mathbb{R}}\left\{z \cdot \widehat{y}_i^{\top} \alpha-\ell^*(z,t)+\eta \cdot \frac{z^2}{4}\right\}+\delta \frac{\|\alpha\|_*^2}{\eta}\\
=&\inf_{\alpha\in\mathcal{A},t\in\mathbb{R},1\ge\eta\ge0}\sup_{(p,\delta) \in \mathcal{V}_+}\sum_{i=1}^Np_i\frac{\left(\widehat{y}_i^{\top} \alpha-\theta/2-t\right)^2}{1-\eta}+\delta \frac{\|\alpha\|_*^2}{\eta}-\frac{\theta^2}{4}-t\theta.
\end{align*}
Applying the Fenchel duality, the problem \eqref{variance} can be reformulated as an infimum problem
\begin{align*}
&\inf_{\alpha\in\mathcal{A},0\le\eta\le 1,t\in\mathbb{R},v}~\sigma^*(v|\mathcal{V}_+)-\frac{\theta^2}{4}-t\theta\\
&~~~~~~~~~~\text{s.t.}~v_i\ge \frac{\left(\widehat{y}_i^{\top} \alpha-\theta/2-t\right)^2}{1-\eta}, ~i=1,...,N,\\
&~~~~~~~~~~~~~~~~v_{N+1}\ge ||\alpha||_*^2/\eta,
\end{align*}
where $\sigma^*(v|\mathcal{V}_+)$, the conjugate function of $\sigma(v|\mathcal{V}_+)$, is given in Example \ref{full-conjugate}. Thus, substituting the $\sigma^*(v|\mathcal{V}_+)$ in Example \ref{full-conjugate}, the robustified conditional  mean-variance portfolio allocation problem \eqref{variance}, proposed in \cite{N24}, can be reformulated as
\begin{align*}
&\inf_{\alpha\in\mathcal{A},0\le\eta\le1,t\in\mathbb{R},z,s}~z^\top b-\frac{\theta^2}{4}-t\theta\\
&~~~~~~~~~~~\text{s.t.}~z\ge 0,~A^\top z\ge s\\
&~~~~~~~~~~~~~~~~~s_i\ge \frac{\left(\widehat{y}_i^{\top} \alpha-\theta/2-t\right)^2}{1-\eta}, ~i=1,...,N,\\
&~~~~~~~~~~~~~~~~~s_{N+1}\ge ||\alpha||_*^2/\eta,~s_{N+2}\ge 0,
\end{align*}
where $A\in\mathbb{R}^{(N+6)\times(N+2)}$,and  $b\in\mathbb{R}^{N+6}$ are defined by 
 \begin{align*}
         A = \begin{pmatrix}
		1& \cdots & 1 & 0 &0 \\
		-1& \cdots & -1 & 0&0\\
		0& \cdots & 0 & 0&1\\
		0& \cdots & 0 & 0&-1\\
		1 & \cdots & 0 & 0&-\frac{1}{N}\\
		\vdots & \ddots & \vdots&\vdots&\vdots \\
		0 & \cdots & 1&0&-\frac{1}{N}\\
		-a_1 & \cdots & -a_{N}&-a_{N+1}&-a_{N+2}\\
		a_1 & \cdots & a_{N}&a_{N+1}& a_{N+2}
	\end{pmatrix},\quad\text{ with }
  a = \begin{pmatrix}-d_1 \\ \vdots \\ -d_m \\ d_{m+1} \\ \vdots \\ d_N \\ -1 \\[.5ex]\displaystyle \delta_0 - \frac1N\sum_{i=m+1}^N d_i\end{pmatrix},\quad\text{ and }
  b = 
    \begin{pmatrix}
      1 \\ -1 \\[3pt]
      \tfrac1{\epsilon_0} \\ -1 \\[3pt]
      0 \\ \vdots \\ 0
    \end{pmatrix}.
    \end{align*}
    \qedhere
\end{proof}

\begin{proof}[{\bf Proof of Proposition \ref{general_min_exp_1}}]
    The proof is similar to Proposition \ref{general_min_exp}.
\end{proof}

\begin{proof}[{\bf Proof of Proposition \ref{pro_rho2_general}}]
    The problem problem \eqref{r11} with \eqref{r2} with $\rho=\rho^{(2)}$ can be reformulated as
    \begin{align*}
		&\min_{\alpha\in\mathcal{A}}\sup_{(p,\delta) \in \mathcal{V}}\sup_{\mathbb{Q}|_{Y}\in \mathbb{B}_{\delta}\left(\sum_{i=1}^N p_i\mathrm{I}_{\widehat{y}_i}\right)}\rho^{(2)}_{\mathbb{Q}|_{Y}}\left(Y^\top\alpha\right)\\
		=&\min_{\alpha\in\mathcal{A}}\sup_{(p,\delta) \in \mathcal{V}}\inf_{t\in\mathbb{R}}~\sup_{\mathbb{Q}|_{Y}\in \mathbb{B}_{\delta}\left(\sum_{i=1}^N p_i\mathrm{I}_{\widehat{y}_i}\right)}t+\left(\E_{\mathbb{Q}|_Y}\left[\ell^q(Y^\top\alpha,t)\right]\right)^{1/q}\\
        =&\inf_{\alpha\in\mathcal{A},t\in\mathbb{R}}\sup_{(p,\delta) \in \mathcal{V}}~t+\sup_{\mathbb{Q}|_{Y}\in \mathbb{B}_{\delta}\left(\sum_{i=1}^N p_i\mathrm{I}_{\widehat{y}_i}\right)}\left(\E_{\mathbb{Q}|_Y}\left[\ell^q(Y^\top\alpha,t)\right]\right)^{1/q}\\
        =&\inf_{\alpha\in\mathcal{A},t\in\mathbb{R}}\sup_{(p,\delta) \in \mathcal{V},\delta\ge0}~t+\left(\sup_{\mathbb{Q}|_{Y}\in \mathbb{B}_{\delta}\left(\sum_{i=1}^N p_i\mathrm{I}_{\widehat{y}_i}\right)}\E_{\mathbb{Q}|_Y}\left[\ell^q(Y^\top\alpha,t)\right]\right)^{1/q}\\
		=&\inf_{\alpha\in\mathcal{A},t\in\mathbb{R}}\sup_{(p,\delta) \in \mathcal{V}_+}~t+\left(\inf_{\lambda\ge 0}\sum_{i=1}^Np_i\cdot\sup_{z\in\mathbb{R}}\left\{\ell^q(z,t)-\lambda\left|z-\widehat{y}_i^\top\alpha\right|^q\right\}+\lambda\delta\left\|\alpha\right\|_*^q\right)^{1/q}\\
        =&\inf_{\alpha\in\mathcal{A},t\in\mathbb{R}}~t+\left(\sup_{(p,\delta) \in \mathcal{V}_+}\inf_{\lambda\ge 0}\sum_{i=1}^Np_i\cdot\sup_{z\in\mathbb{R}}\left\{\ell^q(z,t)-\lambda\left|z-\widehat{y}_i^\top\alpha\right|^q\right\}+\lambda\delta\left\|\alpha\right\|_*^q\right)^{1/q}\\
        =&\inf_{\alpha\in\mathcal{A},t\in\mathbb{R}}~t+\left(\inf_{\lambda\ge 0}\sup_{(p,\delta) \in \mathcal{V}_+}\sum_{i=1}^Np_i\cdot\sup_{z\in\mathbb{R}}\left\{\ell^q(z,t)-\lambda\left|z-\widehat{y}_i^\top\alpha\right|^q\right\}+\lambda\delta\left\|\alpha\right\|_*^q\right)^{1/q}\\
        =&\inf_{\alpha\in\mathcal{A},t\in\mathbb{R},\lambda\ge 0}~t+\left(\sup_{(p,\delta) \in \mathcal{V}_+}\sum_{i=1}^Np_i\cdot\sup_{z\in\mathbb{R}}\left\{\ell^q(z,t)-\lambda\left|z-\widehat{y}_i^\top\alpha\right|^q\right\}+\lambda\delta\left\|\alpha\right\|_*^q\right)^{1/q}\\
        =&\inf_{\alpha\in\mathcal{A},t\in\mathbb{R},\lambda\ge 0}~t+\sup_{(p,\delta) \in \mathcal{V}_+}\left(\sum_{i=1}^Np_i\cdot\sup_{z\in\mathbb{R}}\left\{\ell^q(z,t)-\lambda\left|z-\widehat{y}_i^\top\alpha\right|^q\right\}+\lambda\delta\left\|\alpha\right\|_*^q\right)^{1/q},
			\end{align*}
    where the first equality is from Lemma EC.8 of \cite{W22}, the second and sixth equalities are from minimax theorem in \cite{S58}, the fourth equality is from Theorem \ref{regular-projection}, and other equalities hold for the reason that $x^{1/q}$ is a nonnegative and increasing function and thus it does not change optimal solution of the optimization problem. Using the Fenchel duality, we can reformulate the above problem as
    \begin{align*}
        &\inf_{\alpha\in\mathcal{A},t\in\mathbb{R},\lambda\ge 0}~t+\sup_{(p,\delta) \in \mathcal{V}_+}\left(\sum_{i=1}^Np_i\cdot\sup_{z\in\mathbb{R}}\left\{\ell^q(z,t)-\lambda\left|z-\widehat{y}_i^\top\alpha\right|^q\right\}+\lambda\delta\left\|\alpha\right\|_*^q\right)^{1/q}\\
        =&\inf_{\alpha\in\mathcal{A},t\in\mathbb{R},\lambda\ge 0,v,y}t+\sigma^*\left(v|\mathcal{V}_+\right)+C_q^1y^{\frac{1}{1-q}}\\
        &~~~~~~~~~~\text{s.t.}~v_i\ge y\cdot \sup_{z\in\mathbb{R}}\left\{\ell^q(z,t)-\lambda\left|z-\widehat{y}_i^\top\alpha\right|^q\right\}, i=1,...,N,\\
        &~~~~~~~~~~~~~~~~v_{N+1}\ge \lambda y\left\|\alpha\right\|_*^q, y\ge0.
    \end{align*}
    Let $u=\lambda y$ and $w=y^{\frac{1}{1-q}}$. By the convexity of $\ell$ in $z$, one can verify that $\frac{\ell^q(z,t)}{w^{q-1}}$ is convex in $z$, and thus $\left(\frac{\ell^q(\cdot,t)}{w^{q-1}}\right)^{**}=\frac{\ell^q(\cdot,t)}{w^{q-1}}$. Moreover, we find that $\frac{\ell^q(z,t)}{w^{q-1}}$ is jointly convex in $(z,t,w)$ and thus $\left(\frac{\ell^q(\cdot,t)}{w^{q-1}}\right)^{*}$ is jointly concave in $(t,w)$ according to the definition of concave conjugate function. Let $u=\eta^{1-q}$. Similar to the proof of Theorem \ref{regular-projection}, the problem can be further reformulated as
    \begin{align*}
        &\inf_{\alpha\in\mathcal{A},t\in\mathbb{R},u\ge 0,v,y}t+\sigma^*\left(v|\mathcal{V}_+\right)+C_q^1y^{\frac{1}{1-q}}\\
        &~~~~~~~~~~\text{s.t.}~v_i\ge  \sup_{z\in\mathbb{R}}\left\{y\ell^q(z,t)-u\left|z-\widehat{y}_i^\top\alpha\right|^q\right\}, i=1,...,N,\\
        &~~~~~~~~~~~~~~~~v_{N+1}\ge u\left\|\alpha\right\|_*^q, y\ge0,\\
        =&\inf_{\alpha\in\mathcal{A},t\in\mathbb{R},u\ge 0,v,w}t+\sigma^*\left(v|\mathcal{V}_+\right)+C_q^1w\\
        &~~~~~~~~~~\text{s.t.}~v_i\ge  \sup_{z\in\mathbb{R}}\left\{\frac{\ell^q(z,t)}{w^{q-1}}-u\left|z-\widehat{y}_i^\top\alpha\right|^q\right\}, i=1,...,N,\\
        &~~~~~~~~~~~~~~~~v_{N+1}\ge u\left\|\alpha\right\|_*^q, w\ge0,\\
        =&\inf_{\alpha\in\mathcal{A},\eta\ge 0,t\in\mathbb{R},v,w\ge 0}~t+\sigma^*(v|\mathcal{V}_+)+C_q^1w\\
		&~~~~~~~~~~\text{s.t.}~v_i\ge \sup _{x \in \mathbb{R}}\left\{x \cdot \widehat{y}_i^{\top} \alpha-\left(\frac{\ell^q(\cdot,t)}{w^{q-1}}\right)^*(x)+\eta \cdot \frac{|x|^{q^*}}{q^* q^{q^*-1}}\right\}, ~i=1,...,N,\\
		&~~~~~~~~~~~~~~~~v_{N+1}\ge ||\alpha||_*^q/\eta^{q-1}.
    \end{align*}
    \qedhere
\end{proof}

\begin{proof}[{\bf Proof of Proposition \ref{pro:shortfall_1}}]
    Let $\mathcal{P}=\bigcup\limits_{(p,\delta) \in \mathcal{V}} \mathbb{B}_{\delta}\left(\sum_{i=1}^N p_i\mathrm{I}_{\widehat{y}_i}\right)$ be a convex set. We first prove that 
    \begin{align}\label{shortfall_eq}
       \sup_{{\mathbb{Q}}|_{Y} \in \mathcal{P}}\inf \left\{\kappa \in \mathbb{R}: \mathbb{E}_{{\mathbb{Q}}|_{Y}}\left[u(-Y^\top\alpha-\kappa)\right] \le l\right\}
                =&\inf_{\kappa \in \mathbb{R}}\kappa\\
			&~~\text{s.t.}\sup_{{\mathbb{Q}}|_{Y} \in \mathcal{P}}\mathbb{E}_{{\mathbb{Q}}|_{Y}}\left[u(-Y^\top\alpha-\kappa)\right] \le l.\notag
    \end{align}
     If $\sup_{{\mathbb{Q}}|_{Y} \in \mathcal{P}}\mathbb{E}_{{\mathbb{Q}}|_{Y}}\left[u(-Y^\top\alpha-\kappa)\right] \le l$, we have $\mathbb{E}_{{\mathbb{Q}}|_{Y}}\left[u(-Y^\top\alpha-\kappa)\right] \le l$ for any ${\mathbb{Q}}|_{Y}\in \mathcal{P}$. Then $$\left\{\kappa:\sup_{{\mathbb{Q}}|_{Y} \in \mathcal{P}}\mathbb{E}_{{\mathbb{Q}}|_{Y}}\left[u(-Y^\top\alpha-\kappa)\right] \le l\right\}\subseteq \left\{\kappa : \mathbb{E}_{{\mathbb{Q}}|_{Y}}\left[u(-Y^\top\alpha-\kappa)\right] \le l\right\}.$$ Thus, we have     \begin{align*}        \inf \left\{\kappa \in \mathbb{R}: \mathbb{E}_{{\mathbb{Q}}|_{Y}}\left[u(-Y^\top\alpha-\kappa)\right] \le l\right\}\le \inf\left\{\kappa\in \mathbb{R}:\sup_{{\mathbb{Q}}|_{Y} \in \mathcal{P}}\mathbb{E}_{{\mathbb{Q}}}\left[u(-Y^\top\alpha-\kappa)\right] \le l\right\},~\forall{\mathbb{Q}}|_{Y}\in\mathcal{P},    \end{align*}    and     \begin{align}\label{shortfall_le}      \sup_{{\mathbb{Q}}|_{Y}\in \mathcal{P}}\inf \left\{\kappa \in \mathbb{R}: \mathbb{E}_{{\mathbb{Q}}|_{Y}}\left[u(-Y^\top\alpha-\kappa)\right] \le l\right\}\le\inf\left\{\kappa\in \mathbb{R}:\sup_{{\mathbb{Q}}|_{Y} \in \mathcal{P}}\mathbb{E}_{{\mathbb{Q}}|_{Y}}\left[u(-Y^\top\alpha-\kappa)\right] \le l\right\}.  \end{align}
    Let $\kappa^*=\sup_{{\mathbb{Q}}|_{Y}\in \mathcal{P}}\inf \left\{\kappa \in \mathbb{R}: \mathbb{E}_{{\mathbb{Q}}|_{Y}}\left[u(-Y^\top\alpha-\kappa)\right] \le l\right\}$. We prove the equality \eqref{shortfall_eq} in the following three cases: (i) $\kappa^*=\infty$; (ii) $\kappa^*=-\infty$; (iii) $\kappa^*$ is finite.  \begin{itemize}      \item[(i)] From the inequality \eqref{shortfall_le}, we directly obtain      $$\infty\ge\inf\left\{\kappa\in \mathbb{R}:\sup_{{\mathbb{Q}}|_{Y} \in \mathcal{P}}\mathbb{E}_{{\mathbb{Q}}}\left[u(-Y^\top\alpha-\kappa)\right] \le l\right\}\ge \kappa^*=\infty.$$      Thus, the equality \eqref{shortfall_eq} holds.      \item[(ii)] If $\kappa^*=-\infty$, for any ${\mathbb{Q}}|_{Y}\in \mathcal{P}$, we have $\mathbb{E}_{{\mathbb{Q}}|_{Y}}\left[u(-Y^\top\alpha-\kappa)\right] \le l$ for any $\kappa\in\mathbb{R}$. Then for any $\kappa\in\mathbb{R}$, we have $\mathbb{E}_{{\mathbb{Q}}|_{Y}}\left[u(-Y^\top\alpha-\kappa)\right] \le l$ for any ${\mathbb{Q}}|_{Y}\in \mathcal{P}$. Thus, for any $\kappa\in\mathbb{R}$, we have $\sup_{{\mathbb{Q}}|_{Y}\in\mathcal{P}}\mathbb{E}_{{\mathbb{Q}}|_{Y}}\left[u(-Y^\top\alpha-\kappa)\right] \le l$. Hence,       $$\inf\left\{\kappa\in \mathbb{R}:\sup_{{\mathbb{Q}}|_{Y} \in \mathcal{P}}\mathbb{E}_{{\mathbb{Q}}|_{Y}}\left[u(-Y^\top\alpha-\kappa)\right] \le l\right\}=-\infty=\kappa^*.$$ 
    \item[(iii)] Note that $$\kappa^*\ge\inf\left\{\kappa\in \mathbb{R}:\mathbb{E}_{{\mathbb{Q}}|_{Y}}\left[u(-Y^\top\alpha-\kappa)\right] \le l\right\}~ \forall ~{\mathbb{Q}}|_{Y}\in\mathcal{P}.$$ 
    Thus, $$\mathbb{E}_{{\mathbb{Q}}|_{Y}}\left[u(-Y^\top\alpha-\kappa^*)\right] \le l,~ \forall ~{\mathbb{Q}}|_{Y}\in\mathcal{P},$$
    which implies that $$\sup_{{\mathbb{Q}}|_{Y}\in\mathcal{P}}\mathbb{E}_{{\mathbb{Q}}|_{Y}}\left[u(-Y^\top\alpha-\kappa^*)\right] \le l.$$
    Then $\kappa^*\in \left\{\kappa:\sup_{{\mathbb{Q}}|_{Y} \in \mathcal{P}}\mathbb{E}_{{\mathbb{Q}}|_{Y}}\left[u(-Y^\top\alpha-\kappa)\right] \le l\right\}$, and thus
    $$\kappa^*\ge\inf\left\{\kappa\in \mathbb{R}:\sup_{{\mathbb{Q}}|_{Y} \in \mathcal{P}}\mathbb{E}_{{\mathbb{Q}}|_{Y}}\left[u(-Y^\top\alpha-\kappa)\right] \le l\right\}.$$
    Combining with inequality \eqref{shortfall_le}, we obtain the equality \eqref{shortfall_eq}.
    \end{itemize}
    
    Combining the above three cases, we obtain the equality \eqref{shortfall_eq} in all instances. Note that similar reasoning also appears in Proposition 1 of \cite{G19}, where the finiteness of
    $$\inf\left\{\kappa\in \mathbb{R}:\sup_{{\mathbb{Q}}|_{Y} \in \mathcal{P}}\mathbb{E}_{{\mathbb{Q}}|_{Y}}\left[u(-Y^\top\alpha-\kappa)\right] \le l\right\}$$ 
    is assumed. Here, we establish \eqref{shortfall_eq} in the general case to ensure completeness.

    Next, we apply the equality \eqref{shortfall_eq} to derive the desired result. In particular, for $\rho = \rho^{(3)}$, the problem \eqref{general_rho_union_version} becomes
		\begin{align*}
			&\min_{\alpha \in \mathcal{A}} \sup_{{\mathbb{Q}}|_{Y} \in \bigcup\limits_{(p,\delta) \in \mathcal{V}} 
				\mathbb{B}_{\delta}\left(\sum_{i=1}^N p_i\mathrm{I}_{\widehat{y}_i}\right)} 
			\rho_{{\mathbb{Q}}|_{Y}}\left[Y^\top\alpha \right] \\
			=&\min_{\alpha \in \mathcal{A}} \sup_{{\mathbb{Q}}|_{Y} \in \bigcup\limits_{(p,\delta) \in \mathcal{V}} 
				\mathbb{B}_{\delta}\left(\sum_{i=1}^N p_i\mathrm{I}_{\widehat{y}_i}\right)}\inf \left\{\kappa \in \mathbb{R}: \mathbb{E}_{{\mathbb{Q}}|_{Y}}\left[u(-Y^\top\alpha-\kappa)\right] \le l\right\}\\
			=&\min_{\alpha \in \mathcal{A}}\inf_{\kappa \in \mathbb{R}}\kappa\\
			&~~~~~~~~\text{s.t.}\sup_{{\mathbb{Q}}|_{Y} \in \bigcup\limits_{(p,\delta) \in \mathcal{V}} 
				\mathbb{B}_{\delta}\left(\sum_{i=1}^N p_i\mathrm{I}_{\widehat{y}_i}\right)}\mathbb{E}_{{\mathbb{Q}}|_{Y}}\left[u(-Y^\top\alpha-\kappa)\right] \le l.
		\end{align*}
		And then using the equation \eqref{regular:p=1}, we have 
		\begin{align*}
			&\min_{\alpha \in \mathcal{A}}\inf_{\kappa \in \mathbb{R}}\kappa\\
			&~~~~~~~~\text{s.t.}\sup_{{\mathbb{Q}}|_{Y} \in \bigcup\limits_{(p,\delta) \in \mathcal{V}} 
				\mathbb{B}_{\delta}\left(\sum_{i=1}^N p_i\mathrm{I}_{\widehat{y}_i}\right)}\mathbb{E}_{{\mathbb{Q}}|_{Y}}\left[u(-Y^\top\alpha-\kappa)\right] \le l\\
			=&\min_{\alpha \in \mathcal{A}}\inf_{\kappa \in \mathbb{R}}\kappa\\
			&~~~~~~~~\text{s.t.}\sup_{(p,\delta)\in\mathcal{V}_+}\sup_{{\mathbb{Q}}|_{Y} \in  
				\mathbb{B}_{\delta}\left(\sum_{i=1}^N p_i\mathrm{I}_{\widehat{y}_i}\right)}\mathbb{E}_{{\mathbb{Q}}|_{Y}}\left[u(-Y^\top\alpha-\kappa)\right] \le l\\
			=&\min_{\alpha \in \mathcal{A}}\inf_{\kappa \in \mathbb{R}}\kappa\\
			&~~~~~~~~\text{s.t.} \sup_{(p,\delta)\in\mathcal{V}_+}\sum_{i=1}^N p_i u(-\widehat{y}_i^\top\alpha-\kappa) + \operatorname{Lip}(u)\delta ||\alpha||_*\le l.
		\end{align*}
        Using the Fenchel duality, we have
		\begin{align*}
			&\min_{\alpha \in \mathcal{A}}\inf_{\kappa \in \mathbb{R}}\kappa\\
			&~~~~~~~~\text{s.t.} \sup_{(p,\delta)\in\mathcal{V}_+}\sum_{i=1}^N p_i u(-\widehat{y}_i^\top\alpha-\kappa) + \operatorname{Lip}(u)\delta ||\alpha||_*\le l\\
			=&\min_{\alpha \in \mathcal{A}}\inf_{\kappa \in \mathbb{R}}\kappa\\
			&~~~~~~~~\text{s.t.} \inf_{v}\left\{\sigma^*(v|\mathcal{V}_+)|v- \nu(\alpha,\kappa)\ge 0\right\}\le l\\
			=&\inf_{\alpha,\kappa,v}\kappa\\
			&~~\text{s.t.}~\alpha \in \mathcal{A},\kappa \in \mathbb{R}, \\
			&~~~~~~~~\sigma^*(v|\mathcal{V}_+)\le l,~v- \nu(\alpha,\kappa)\ge 0,
		\end{align*} 
		where $\nu(\alpha,\kappa)^\top=(u(-\widehat{y}_1^\top\alpha-\kappa),...,u(-\widehat{y}_N^\top\alpha-\kappa),\operatorname{Lip}(u)||\alpha||_*)\in\mathbb{R}^{N+1}$.
    \qedhere
\end{proof}

\begin{proof}[{\bf Proof of Proposition \ref{pro:shortfall_2}}]
    Similar to the proof of Proposition \ref{pro:shortfall_1}, applying the equality \eqref{shortfall_eq}, Theorem \ref{regular-projection}, min-max theorem, and the Fenchel duality, we have the desired result. \qedhere
\end{proof}

\subsection*{EC.2.2.2. Proofs of Section \ref{distortion_form}}

\begin{proof}[{\bf Proof of Proposition \ref{pro4*}}]
	\begin{itemize}
		\item [(i)]For $q=1$ and $\mathbb{D}_{\mathcal{Y}}(y_1,y_2)=||y_1-y_2||_{\mathcal{Y}}$, applying Theorem 5 in \cite{W22}, we find that problem \eqref{r11} with \eqref{r2} can be reformulated as
		\begin{align*}
			\min _{\alpha \in \mathcal{A}}\sup_{(p,\delta) \in \mathcal{V}_+}\rho_{\sum_{i=1}^N p_i\mathrm{I}_{\widehat{y}_i}}^h\left[\ell( Y^\top\alpha)\right]+\operatorname{Lip}(\ell)\left\|h_{-}^{\prime}\right\|_{\infty} \delta\|{\alpha}\|_*.
		\end{align*}
		According to Proposition 10.3 of \cite{D13}, we can reformulate the above inner supremum problem as
		\begin{align*}
			&\sup_{(p,\delta) \in \mathcal{V}_+}\sup_{\bar{p}\in\mathcal{V}_p^h}\sum_{i=1}^N\bar{p}_i \ell\left( \widehat{y}_i^\top\alpha\right)+\operatorname{Lip}(\ell)\left\|h_{-}^{\prime}\right\|_{\infty} \delta\|{\alpha}\|_*\\
			=&\sup_{(p,\delta) \in \mathcal{V}_+,\bar{p}\in\mathcal{V}_p^h}\sum_{i=1}^N\bar{p}_i \ell\left( \widehat{y}_i^\top\alpha\right)+\operatorname{Lip}(\ell)\left\|h_{-}^{\prime}\right\|_{\infty} \delta\|{\alpha}\|_*.
		\end{align*}
		Using the Fenchel duality theory, we have 
		\begin{align*}
			&\sup_{(p,\delta) \in \mathcal{V}_+,\bar{p}\in\mathcal{V}_p^h}\sum_{i=1}^N\bar{p}_i \ell\left( \widehat{y}_i^\top\alpha\right)+\operatorname{Lip}(\ell)\left\|h_{-}^{\prime}\right\|_{\infty} \delta\|{\alpha}\|_*\\
			=&\inf_{v}\left\{\sigma^*(v|\mathcal{V}_+^h)|v-\nu(\alpha)\ge 0\right\},
		\end{align*}
		where
		\begin{align*}
			\mathcal{V}_+^h=\left\{p''=\left(p,\delta,\bar{p}\right):(p,\delta)\in\mathcal{V},\delta\ge0,\bar{p}\in\mathcal{V}_p^h\right\},
		\end{align*}
		$$\mathcal{V}_p^h=\left\{\bar{p}:\sum_{i=1}^N \bar{p}_i=1,\bar{p}_i \geq 0, i=1,...,N, h\left(\sum_{i \in J} p_i\right) \leq \sum_{i \in J} \bar{p}_i, \forall J \subset\{1,...,N\}  \right\},$$ 
		and
		\begin{align*}
			\nu(\alpha)^\top=\left(0,...,0,\operatorname{Lip}(\ell)\left\|h_{-}^{\prime}\right\|_{\infty}\|\alpha\|_*, \ell(\widehat{y}_1^\top\alpha),...,\ell(\widehat{y}_N^\top\alpha)\right)\in\mathbb{R}^{2N+1}.
		\end{align*}
		Combing with the outer minimization problem, we obtain the result in (i).
		\item[(ii)] Similar to the case (i), applying Theorem 5 in \cite{W22}, Proposition 10.3 of \cite{D13}, and the Fenchel duality theory, we obtain the result.
	\end{itemize}
    \qedhere
\end{proof}
\begin{proof}[{\bf Proof of Example \ref{distortion_conjugate_full}}]
    For the full optimal transport model in Section \ref{other_model}, we have $${\cal V}:=\left\{(p,\delta):\exists \epsilon \text{~such that }(p,\delta,\epsilon)\in \mathcal{V}_0 \right\},$$
where 
\begin{align*}
	{\cal V}_0:=\Bigg\{(p,\delta,\epsilon):&\frac{1}{\epsilon}\in \mathcal{N}_{\omega}, p_i\in\left[0,\frac{\epsilon}{N}\right], i=1,..., N, \sum_{i=1}^Np_i=1,\\
	&\delta = \epsilon\left(\delta_0-\frac{1}{N}\sum_{i=m+1}^Nd_i\right)-\sum_{i=1}^{m}p_id_i+\sum_{i=m+1}^Np_id_i\Bigg\}.
\end{align*}
Let $${\cal V}_+:=\left\{(p,\delta):(p,\delta)\in \mathcal{V},\delta\ge 0 \right\},$$ 
and \begin{align*}
	{\cal V}_+^h:=\left\{(p,\delta,\bar{p}):(p,\delta)\in\mathcal{V}_+, \bar{p}\in\mathcal{V}_p^h\right\}.
\end{align*}
Then in the following we obtain the conjugate function of the indicator function of $\mathcal{V}_+^h$ to have the tractable reformulation for the full optimal transport model. The conjugate function of the indicator function of $\mathcal{V}_+^h$, \begin{align*}
\sigma^*(v\mid \mathcal{V}_+^h)=\sup_{(p,\delta,\bar{p})\in\mathcal{V}_+^h}v^\top (p,\delta,\bar{p})=\sup_{(p,\delta,\epsilon)\in\mathcal{V}_0,\delta\ge 0,\bar{p}\in\mathcal{V}_p^h}v^\top (p,\delta,\bar{p})+0\cdot \epsilon.
\end{align*} Denote $p''=(p,\delta,\bar{p},\epsilon)$. 
One can verify that the above problem is a convex optimization problem, and thus we can apply strong duality theory. The Lagrangian function is
$$
\begin{aligned}
&\mathcal{L}(p'')\\
=&(v,0)^\top p''+\lambda\left(\sum_{i=1}^N p_i-1\right)+\mu\left(\sum_{i=1}^N \bar{p}_i-1\right)-\sum_{J \subset[N]} \zeta_J\left(h\left(\sum_{i \in J} p_i\right)-\sum_{i \in J} \bar{p}_i\right) \\
& -\sum_{i=1}^N \tau_i^1\left(-p_i\right)-\sum_{i=1}^N \gamma_i\left(-\bar{p}_i\right)-\sum_{i=1}^N \tau_i^2\left(p_i-\frac{\epsilon}{N}\right)-\eta_1(\frac{1}{R_{\mathcal{N}_{\omega}}}-\epsilon)-\eta_2\left(\epsilon-\frac{1}{L_{\mathcal{N}_{\omega}}}\right)\\
&-\xi_1(-\delta)+\xi_2\left(\delta-\epsilon(\delta_0-\frac{1}{N}\sum_{i=m+1}^Nd_i)+\sum_{i=1}^{m}p_id_i-\sum_{i=m+1}^Np_id_i\right),
\end{aligned}
$$
where $p^\top=(p_1,...,p_N)\in[0,1]^N$, $q^\top=(q_1,...,q_N)\in[0,1]^N$, $\lambda,\mu,\xi_2\in\mathbb{R}$, and $\zeta_J,\tau_i^1,\gamma_i,\tau_i^2,\eta_1,\eta_2,\xi_1\ge0$. Then we simplify the problem $\sup_{p''}\mathcal{L}(p'')$. One can verify that $\sup_{p''}\mathcal{L}(p'')$ can be decomposed into four optimization problems, which are only related to $\epsilon$, $p$, $\bar{p}$, and $\delta$, respectively. For the problem related to $\epsilon$ with constant terms, we have
\begin{align*}
&\sup_{\epsilon\ge 0}\Bigg\{\epsilon\left[-\xi_2(\delta_0-\frac{1}{N} \sum_{i=m+1}^Nd_i)+\frac{1}{N} \sum_{i=1}^N \tau_i^2+\eta_1-\eta_2\right]-\frac{\eta_1}{R_{\mathcal{N}_{\omega}}}+\frac{\eta_2}{L_{\mathcal{N}_{\omega}}}-\lambda-\mu\Bigg\} \\  
=&\begin{cases}-\frac{\eta_1}{R_{\mathcal{N}_{\omega}}}+\frac{\eta_2}{L_{\mathcal{N}_{\omega}}}-\lambda-\mu \\~~~~~~~~~\text { if } -\xi_2(\delta_0-\frac{1}{N} \sum_{i=m+1}^Nd_i)+\frac{1}{N} \sum_{i=1}^N \tau_i^2+\eta_1-\eta_2\le 0\\ \infty ~~~~~ \text { else. }\end{cases}
\end{align*}
Let$$c_i= \begin{cases}\xi_2 d_i, & i=1, \ldots, m \\ -\xi_2 d_i, & i=m+1, \ldots, N.\end{cases}$$
For the problem related to $p$, we have
\begin{align*}
&\sup_{p\ge 0}\left\{\sum_{i=1}^Np_i\left[v_i+\lambda+\tau_i^1-\tau_i^2+c_i\right]-\sum_{J\subset\{1,...,N\}}\zeta_Jh\left(\sum_{i\in J}p_i\right)\right\}\\
=&\sup_{p\ge 0, w_1,...,w_{2^N-2}\ge 0}\left\{\sum_{i=1}^Np_i\left[v_i+\lambda+\tau_i^1-\tau_i^2+c_i\right]-\sum_{j=1}^{2^N-2}\zeta_jh(w_j)|w_j=\sum_{k\in I_j}p_k\right\}\\
=&\inf_{\phi_1, ..., \phi_{2^N-2}}\sup_{p\ge 0, w_1,...,w_{2^N-2}\ge 0}\Bigg\{\sum_{i=1}^Np_i\left[v_i+\lambda+\tau_i^1-\tau_i^2+c_i\right]\\
&~~~~~~~~~~~~~~~~~~~~~~~~~~~~~~~~~~~~~~~~~~~-\sum_{j=1}^{2^N-2}\zeta_jh(w_j)-\sum_{j=1}^{2^N-2}\phi_j(w_j-\sum_{k\in I_j}p_k)\Bigg\}\\
=&\inf_{\phi_1, ..., \phi_{2^N-2}}\Bigg\{\sup_{p\ge 0}\left\{\sum_{i=1}^Np_i\left[v_i+\lambda+\tau_i^1-\tau_i^2+c_i\right]+\sum_{j=1}^{2^N-2}\phi_j\sum_{k\in I_j}p_k\right\}\\
&~~~~~~~~~~~~~~~~~~+\sup_{w_1,...,w_{2^N-2}\ge 0}\left\{-\sum_{j=1}^{2^N-2}\zeta_jh(w_j)-\sum_{j=1}^{2^N-2}\phi_jw_j\right\}\Bigg\}.
\end{align*}
We find that
\begin{align*}
&\sup_{p\ge 0}\left\{\sum_{i=1}^Np_i\left[v_i+\lambda+\tau_i^1-\tau_i^2+c_i\right]+\sum_{j=1}^{2^N-2}\phi_j\sum_{k\in I_j}p_k\right\}\\
=&\sup_{p\ge 0}\left\{\sum_{i=1}^Np_i\left[v_i+\lambda+\tau_i^1-\tau_i^2+c_i+\sum_{j:i\in I_j}\phi_j\right]\right\}\\
=& \begin{cases}0 & \text { if } v_i+\lambda+\tau_i^1-\tau_i^2+c_i+\sum_{j:i\in I_j}\phi_j\le 0,  i=1,...,N \\ \infty & \text { else, }\end{cases}
\end{align*}
and that
\begin{align*}
\sup_{w_1,...,w_{2^N-2}\ge 0}\left\{-\sum_{j=1}^{2^N-2}\zeta_jh(w_j)-\sum_{j=1}^{2^N-2}\phi_jw_j\right\}=\sum_{j=1}^{2^N-2}\zeta_jh^*(\frac{-\phi_j}{\zeta_j}).
\end{align*}
Thus, 
\begin{align*}
&\sup_{p\ge 0}\left\{\sum_{i=1}^Np_i\left[v_i+\lambda+\tau_i^1-\tau_i^2+c_i\right]-\sum_{J\subset\{1,...,N\}}\zeta_Jh(\sum_{i\in J}p_i)\right\}\\
=&\inf_{\phi_1, ..., \phi_{2^N-2}}\sum_{j=1}^{2^N-2}\zeta_jh^*(\frac{-\phi_j}{\zeta_j})\\
&~~~~~~\text{s.t.}~v_i+\lambda+\tau_i^1-\tau_i^2+c_i+\sum_{j:i\in I_j}\phi_j\le 0,~  i=1,...,N.
\end{align*}
For the problem related to $\bar{p}$, we have
\begin{align*}
&\sup_{\bar{p}\ge 0}\left\{\sum_{i=1}^N \bar{p}_i\left[v_{N+1+i}+\mu+{\sum_{j:i\in I_j} \zeta_j}+\gamma_i\right]\right\}\\
=&\begin{cases}0 & \text { if } v_{N+1+i}+\mu+{\sum_{j:i\in I_j} \zeta_j}+\gamma_i \le 0,~ i=1,...,N,\\ \infty & \text { else. }\end{cases}
\end{align*}
For the problem related to $\delta$, we have 
\begin{align*}
\sup_{\delta\ge 0}\left\{\delta^*\left[v_{N+1}+\xi_1+\xi_2\right]\right\}=\begin{cases}0 & \text { if } v_{N+1}+\xi_1+\xi_2 \le 0,\\ \infty & \text { else. }\end{cases}
\end{align*}
Then the conjugate function of $\sigma((p,\delta,\bar{p})|\mathcal{V}_+^h)$, $\sigma^*(v|\mathcal{V}_+^h)$, is equivalent to
\begin{align*}
&\inf-\frac{\eta_1}{R_{\mathcal{N}_{\omega}}}+\frac{\eta_2}{L_{\mathcal{N}_{\omega}}}-\lambda-\mu+\sum_{j=1}^{2^N-2}\zeta_jh^*(\frac{-\phi_j}{\zeta_j})\\
&~\text{s.t.}~\lambda, \mu,\xi_2\in\mathbb{R},\left\{\phi_j\right\}_{j=1}^{2^N-2}\in\mathbb{R}^{2^N-2},\left\{\zeta_j\right\}_{j=1}^{2^N-2},\left\{\tau_i^1\right\}_{i=1}^N,\left\{\tau_i^2\right\}_{i=1}^N,\left\{\gamma_i\right\}_{i=1}^N,  \eta_1, \eta_2,\xi_1\ge 0\\
&~~~~~~-\xi_2(\delta_0-\frac{1}{N} \sum_{i=m+1}^Nd_i)+\frac{1}{N} \sum_{i=1}^N \tau_i^2+\eta_1-\eta_2\le 0,\\
&~~~~~~~v_i+\lambda+\tau_i^1-\tau_i^2+c_i+\sum_{j:i\in I_j}\phi_j\le 0,~  i=1,...,N,\\
&~~~~~~~c_i=\xi_2 d_i,  i=1, \ldots, m,~c_i= -\xi_2 d_i,  i=m+1, \ldots, N,\\
&~~~~~~~v_{N+1+i}+\mu+{\sum_{j:i\in I_j} \zeta_j}+\gamma_i \le 0,~i=1,...,N,\\
&~~~~~~~v_{N+1}+\xi_1+\xi_2 \le 0.
\end{align*}
\qedhere
\end{proof}

\begin{proof}[{\bf Proof of Example \ref{distortion_conjugate_partial}}]
    For the partial optimal transport model in Section \ref{other_model}, we have
$${\cal V}:=\left\{(p,\delta):{\epsilon}=\min_{\beta\in\mathcal{N}_{\omega}}\beta, p_i\in[0,\frac{1}{N\epsilon}], i=1,..., N, \sum_{i=1}^Np_i=1,\delta= \delta_0-\sum_{i=1}^mp_id_i\right\}.$$
One can verify that $\epsilon$ is a constant given $\mathcal{N}_\omega$. Then $$\mathcal{V}_+:=\left\{(p,\delta):(p,\delta)\in\mathcal{V}, \delta\ge 0\right\},$$
and
\begin{align*}
\mathcal{V}_+^h:=\left\{p''=(p,\delta,\bar{p}):(p,\delta)\in\mathcal{V}_+, \bar{p}\in\mathcal{V}_p^h\right\},
\end{align*}
and we want to obtain the conjugate function of the indicator function of $\mathcal{V}_+^h$, \begin{align*}
\sigma^*(v\mid \mathcal{V}_+^h)=\sup_{p''\in\mathcal{V}_+^h}v^\top p''.
\end{align*} 
One can verify that the above problem is a convex optimization problem, and thus we can apply strong duality theory. Let$$c_i'= \begin{cases}\xi_2 d_i, & i=1, \ldots, m \\ 0, & i=m+1, \ldots, N.\end{cases}$$
Similar to the proof of Example \ref{distortion_conjugate_full}, the conjugate function of $\sigma((p,\delta,\bar{p})|\mathcal{V}_+^h)$, $\sigma^*(v|\mathcal{V}_+^h)$, is equivalent to
\begin{align*}
&\inf-\lambda-\mu+\frac{1}{N\epsilon}\sum_{i=1}^N\tau_i^2-\xi_2\delta_0+\sum_{j=1}^{2^N-2}\zeta_jh^*(\frac{-\phi_j}{\zeta_j})\\
&~\text{s.t.}~\alpha\in\mathcal{A},~\lambda, \mu,\xi_2\in\mathbb{R},\left\{\phi_j\right\}_{j=1}^{2^N-2}\in\mathbb{R}^{2^N-2},\left\{\zeta_j\right\}_{j=1}^{2^N-2},\left\{\tau_i^1\right\}_{i=1}^N,\left\{\tau_i^2\right\}_{i=1}^N,\left\{\gamma_i\right\}_{i=1}^N, \xi_1\ge 0,\\
&~~~~~~~v_i+\lambda+\tau_i^1-\tau_i^2+c_i'+\sum_{j:i\in I_j}\phi_j\le 0,~  i=1,...,N,\\
&~~~~~~~c_i'=\xi_2 d_i,  i=1, \ldots, m,~c_i= 0,  i=m+1, \ldots, N,\\
&~~~~~~~v_{N+1+i}+\mu+{\sum_{j:i\in I_j} \zeta_j}+\gamma_i \le 0,~i=1,...,N,\\
&~~~~~~~v_{N+1}+\xi_1+\xi_2 \le 0.
\end{align*}
\qedhere
\end{proof}

\begin{proof}[{\bf Proof of Example \ref{re:CVaR}}]
    The robustified conditional  mean-$\CVaR$ portfolio allocation problem in \cite{N24}  is as follows: \begin{align*}
\min _{\alpha \in \mathcal{A}} \sup _{\substack{\mathbb{Q} \in \mathbb{B}_{\delta_0}(\widehat{\mathbb{P}}) \\ \mathbb{Q}\left(X \in \mathcal{N}_\gamma(x_0)\right) \in \mathcal{N}_{\omega}}} \operatorname{CVaR}_{\mathbb{Q}}^{1-\kappa}\left[Y^{\top} \alpha \mid X \in \mathcal{N}_\gamma\left(x_0\right)\right]-\theta \cdot \mathbb{E}_{\mathbb{Q}}\left[Y^{\top} \alpha \mid X \in \mathcal{N}_\gamma\left(x_0\right)\right],    
\end{align*}
where $\mathcal{N}_{\omega}=[\epsilon_0,1]$, $\epsilon_0\in(0,1]$, $\theta\ge 0$,  $\kappa\in[0,1]$, $$\operatorname{CVaR}_{\mathbb{Q}}^{1-\kappa}(Z)=\frac{1}{\kappa}\int_{1-\kappa}^1\operatorname{VaR}_{\mathbb{Q}}^{\phi}(-Z)\operatorname{d}\phi,$$
and $\operatorname{VaR}_{\mathbb{Q}}^{\phi}(Z)= \inf \left\{z \in \mathbb{R}: \mathbb{Q}(Z\le z) \geqslant \phi\right\}$. Suppose in addition that $\mathcal{X}=\mathbb{R}^{n_1}, \mathcal{Y}=\mathbb{R}^{n_2}, \mathbb{D}_{\mathcal{X}}(x, \widehat{x})=\|x-\widehat{x}\|^2$ and $\mathbb{D}_{\mathcal{Y}}(y, \widehat{y})=\|y-\widehat{y}\|_2^2$. One can verify that the objective function can be reformulated as 
\begin{align*}
&\operatorname{CVaR}_{\mathbb{Q}}^{1-\kappa}\left[Y^{\top} \alpha \mid X \in \mathcal{N}_\gamma\left(x_0\right)\right]-\theta \cdot \mathbb{E}_{\mathbb{Q}}\left[Y^{\top} \alpha \mid X \in \mathcal{N}_\gamma\left(x_0\right)\right]\\
=&(1+\theta)\rho_{\mathbb{Q}}^h[-Y^\top \alpha\mid X \in \mathcal{N}_\gamma\left(x_0\right)],
\end{align*}
where \begin{align*}
h(x)=\begin{cases}\frac{\theta}{1+\theta} x & x\in[0,1-\kappa],\\ \frac{1+\theta\kappa}{\kappa(1+\theta)}x-\frac{1-\kappa}{\kappa(1+\theta)} & x\in(1-\kappa,1].\end{cases}
\end{align*}
Applying Theorem 5 in \cite{W22}, we have
\begin{align*}
&\min _{\alpha \in \mathcal{A}} \sup _{\substack{\mathbb{Q} \in \mathbb{B}_{\delta_0}(\widehat{\mathbb{P}}) \\ \mathbb{Q}\left(X \in \mathcal{N}_\gamma(x_0)\right) \in \mathcal{N}_{\omega}}} (1+\theta)\rho_{\mathbb{Q}}^h\left[-Y^\top \alpha \mid X \in \mathcal{N}_\gamma\left(x_0\right)\right]\\
=&\min _{\alpha \in \mathcal{A}}\sup_{(p,\delta) \in \mathcal{V}_+}(1+\theta)\left[\rho_{\sum_{i=1}^N p_i\mathrm{I}_{\widehat{y}_i}}^h\left(-Y^\top\alpha\right)+\left\|h_{-}^{\prime}\right\|_{2} \delta(p,\epsilon)^\frac{1}{2}\|{\alpha}\|_2\right],
\end{align*}
 where $\left\|h_{-}^{\prime}\right\|_{2}=\left(1+\frac{1-\kappa}{\kappa(1+\theta)^2}\right)^\frac{1}{2}$. Using the minimization form of CVaR and minimax theorem, we reformulate the inner problem of the above problem as
\begin{align*}
&\sup_{(p,\delta) \in \mathcal{V}_+}(1+\theta)\left[\rho_{\sum_{i=1}^N p_i\mathrm{I}_{\widehat{y}_i}}^h\left(-Y^\top\alpha\right)+\left\|h_{-}^{\prime}\right\|_{2} \delta^\frac{1}{2}\|{\alpha}\|_2\right]\\
=&\sup_{(p,\delta) \in \mathcal{V}_+}\inf_{t\in\mathbb{R}}\left\{\E_{\sum_{i=1}^N p_i\mathrm{I}_{\widehat{y}_i}}\left[t+\frac{1}{\kappa}(-Y^\top \alpha-t)_+-\theta\cdot Y^\top \alpha\right]+(1+\theta)\left\|h_{-}^{\prime}\right\|_{2} \delta^\frac{1}{2}\|\alpha\|_2\right\}\\
=&\inf_{t\in\mathbb{R}}\sup_{(p,\delta) \in \mathcal{V}_+}\left\{\E_{\sum_{i=1}^N p_i\mathrm{I}_{\widehat{y}_i}}\left[t+\frac{1}{\kappa}(-Y^\top \alpha-t)_+-\theta\cdot Y^\top \alpha\right]+(1+\theta)\left\|h_{-}^{\prime}\right\|_{2} \delta^\frac{1}{2}\|\alpha\|_2\right\}.
\end{align*}
Applying the Fenchel duality theory, the problem \eqref{CVaR} is equivalent to \begin{align*}
&\inf_{\alpha\in\mathcal{A},t\in\mathbb{R},v}~\sigma^*(v|\mathcal{V}_+)+C_{\theta,\kappa}v_{N+1}\left\|\frac{\alpha}{v_{N+1}}\right\|_2^{2}\\
&~~~~~~\text{s.t.}~v_i\ge \ell(\widehat{y}_i^\top\alpha,t), i=1,...,N,\\
&~~~~~~~~~~~~v_{N+1}\ge 0,
\end{align*}
where $C_{\theta,\kappa}=\frac{1}{4}\left(\left(1+\theta\right)^2+\frac{1-\kappa}{\kappa}\right)$, 
$\ell(y^\top\alpha,t)=t+\frac{1}{\kappa}(-y^\top \alpha-t)_+-\theta\cdot y^\top \alpha$, and $\sigma^*(v|\mathcal{V}_+)$, the conjugate function of $\sigma(v|\mathcal{V}_+)$, is given in Example \ref{full-conjugate}. Thus, substituting the $\sigma^*(v|\mathcal{V}_+)$ in Example \ref{full-conjugate}, the robustified conditional  mean-$\CVaR$ portfolio allocation problem \eqref{CVaR}, proposed in \cite{N24}, can be reformulated as
\begin{align*}
&\inf_{\alpha,t,v,z}~z^\top b+C_{\theta,\kappa}v_{N+1}\left\|\frac{\alpha}{v_{N+1}}\right\|_2^{2}\\
&~~~\text{s.t.}~\alpha\in\mathcal{A},~t\in\mathbb{R},~v-\nu(\alpha,t)\ge0,~A^\top z\ge v,~z\ge 0.
\end{align*}
And thus, it is equivalent to
\begin{align*}
&\inf_{\alpha\in\mathcal{A},t\in\mathbb{R},z\ge 0,v_{N+1},s}~z^\top b+s\\
&~~~~~~~~~~~\text{s.t.}~v_{N+1}\ge0,~A^\top z\ge (\ell(\widehat{y}_1^\top\alpha,t),\cdots,\ell(\widehat{y}_N^\top\alpha,t),v_{N+1},0)^\top,\\
&~~~~~~~~~~~~~~~~~s\ge C_{\theta,\kappa}v_{N+1}\left\|\frac{\alpha}{v_{N+1}}\right\|_2^{2},
\end{align*}
where $A\in\mathbb{R}^{(N+6)\times(N+2)}$,and  $b\in\mathbb{R}^{N+6}$ are defined by 
 \begin{align*}
         A = \begin{pmatrix}
		1& \cdots & 1 & 0 &0 \\
		-1& \cdots & -1 & 0&0\\
		0& \cdots & 0 & 0&1\\
		0& \cdots & 0 & 0&-1\\
		1 & \cdots & 0 & 0&-\frac{1}{N}\\
		\vdots & \ddots & \vdots&\vdots&\vdots \\
		0 & \cdots & 1&0&-\frac{1}{N}\\
		-a_1 & \cdots & -a_{N}&-a_{N+1}&-a_{N+2}\\
		a_1 & \cdots & a_{N}&a_{N+1}& a_{N+2}
	\end{pmatrix},\quad\text{ with }
  a = \begin{pmatrix}-d_1 \\ \vdots \\ -d_m \\ d_{m+1} \\ \vdots \\ d_N \\ -1 \\[.5ex]\displaystyle \delta_0 - \frac1N\sum_{i=m+1}^N d_i\end{pmatrix},\quad\text{ and }
  b = 
    \begin{pmatrix}
      1 \\ -1 \\[3pt]
      \tfrac1{\epsilon_0} \\ -1 \\[3pt]
      0 \\ \vdots \\ 0
    \end{pmatrix}.
    \end{align*}
    \qedhere
\end{proof}

\subsection*{EC.3. Full Reformulation for the Conditional Expectation Case Based on Admissible Sets from Theorems 1 and 2, and Comparison with \cite{N24}}

In this section, we present full reformulations of the conditional expectation problem under our union-ball framework, leveraging the admissible sets characterized in Theorems 1 and 2. These yield tractable convex programs for the full and partial optimal transport models, respectively. While the primary focus is on developing these reformulations, we also include a brief comparison with the formulation proposed in \cite{N24}, highlighting key structural differences in terms of tractability, compactness, and interpretability.

\begin{itemize}
        \item[(A)] Note that full optimal transport model \eqref{m01} given by \cite{N24} is equivalent to problem \eqref{pro1:exp1-1} with set $\mathcal{V}$ given in Theorem \ref{thm1}. Applying the conjugate of the indicator function of $\mathcal{V}_+$ in Example \ref{full-conjugate} to Propositions \ref{regular_exp} and \ref{general_exp}, we have the following results.
        \begin{itemize}
	\item [(i)]Let $\mathbb{D}_{\mathcal{Y}}(y_1,y_2)=||y_1-y_2||_{\mathcal{Y}}$ for any $y_1,y_2\in\mathcal{Y}$ and $\ell: \mathbb{R} \rightarrow \mathbb{R}$ be a convex and Lipschitz continuous function. 
	The problem \eqref{pro1:exp1-1} can be solved by the convex program
	\begin{align*}
	 &\inf_{\alpha \in\mathcal{A},z\ge0}~z^\top b\\
		&~~~~\text{s.t.}~ \left[A^\top z\right]_i\ge \ell(\widehat{y}_i^\top\alpha), ~i=1,...,N, \\
		&~~~~~~~~~~\left[A^\top z\right]_{N+1}\ge \operatorname{Lip}(\ell)||\alpha||_*,\\
        &~~~~~~~~~~\left[A^\top z\right]_{N+2}\ge 0.
	\end{align*}
    Note that the objective function of above reformulation is linear. Moreover, given $\alpha$, this formulation has $N+6$ variables, $2N+8$ linear constraints and $1$ convex constraint.
	\item[(ii)]Let $\mathbb{D}_{\mathcal{Y}}(y_1,y_2)=||y_1-y_2||_{\mathcal{Y}}^q$ for any $y_1,y_2\in\mathcal{Y}$ and $q\in(1,\infty)$. 
\begin{itemize}
	\item [(1)]Assume the loss function $\ell$ takes one of the following two forms, multiplied by $C>0$ :
	\begin{itemize}
		\item [(a)]$\ell_1(x)=x+b$ or $\ell_1(x)=-x+b$ with some $b \in \mathbb{R}$;
		\item[(b)]  $\ell_2(x)=\left|x-b_1\right|+b_2$ with some $b_1, b_2 \in \mathbb{R}$.
	\end{itemize}The problem \eqref{pro1:exp1-1} can be solved by the convex program
	\begin{align*}
		&\inf_{\alpha  \in\mathcal{A},v_{N+1}\ge 0,z\ge 0,s}~z^\top b+s\\
		&~~~~~~~~~~~\text{s.t.}~[A^\top z]_i\ge \ell(\widehat{y}_i^\top\alpha), ~i=1,...,N,\\
		&~~~~~~~~~~~~~~~~~[A^\top z]_{N+1}\ge v_{N+1},\\
            &~~~~~~~~~~~~~~~~~[A^\top z]_{N+2}\ge 0,\\
            &~~~~~~~~~~~~~~~~~s\ge C_qv_{N+1}\left\|\frac{\alpha}{v_{N+1}}\right\|_*^{\frac{q}{q-1}},
	\end{align*}
	where $C_q=C^{\frac{q}{q-1}}\left(q^{\frac{1}{1-q}}-q^{\frac{q}{1-q}}\right)$. Note that the objective function of above reformulation is linear. Moreover, given $\alpha$, this formulation has $N+8$ variables, $2N+9$ linear constraints and $1$ convex constraint.
	\item[(2)] Assume the loss function $\ell$ is of the form $\ell(x) = \left(C \cdot \ell_i(x)\right)^q$, where $C > 0$, $q \in (1, \infty)$, and $\ell_i$ is one of the following:	
	\begin{itemize}
		\item [(a)]$\ell_1(x)=(x-b)_{+}$with some $b \in \mathbb{R}$;
		\item[(b)] $\ell_2(x)=(x-b)_{-}$with some $b \in \mathbb{R}$;
		\item[(c)] $\ell_3(x)=\left(\left|x-b_1\right|-b_2\right)_{+}$with some $b_1 \in \mathbb{R}$ and $b_2 \geqslant 0$;
		\item[(d)] $\ell_4(x)=\left|x-b_1\right|+b_2$ with some $b_1 \in \mathbb{R}$ and $b_2>0$.
	\end{itemize} 
	The problem \eqref{pro1:exp1-1} can be solved by the convex program	
	\begin{align*}
		&\inf_{\alpha\in\mathcal{A},v_{N+1}\ge0,y\geq 0,z\ge 0,s}\left[z^\top b+C_q^1y+s\right]^q\\
		&~~~~~~~~~~~~~~\text{s.t.}~[A^\top z]_i\ge \frac{\ell(\widehat{y}_i^\top\alpha)}{y^{q-1}}, ~i=1,...,N,\\
		&~~~~~~~~~~~~~~~~~~~~[A^\top z]_{N+1}\ge v_{N+1},\\
            &~~~~~~~~~~~~~~~~~~~~[A^\top z]_{N+2}\ge 0,\\
            &~~~~~~~~~~~~~~~~~~~~s\ge C_qv_{N+1}\left\|\frac{\alpha}{v_{N+1}}\right\|_*^{\frac{q}{q-1}},
	\end{align*}
	where $C_q^1=q^{\frac{1}{1-q}}-q^{\frac{q}{1-q}}$ and $C_q=C^{\frac{q}{q-1}}\left(q^{\frac{1}{1-q}}-q^{\frac{q}{1-q}}\right)$. And thus the problem \eqref{pro1:exp1-1} can be solved by the convex program	
	\begin{align*}
		&\inf_{\alpha\in\mathcal{A},v_{N+1}\ge0,y\geq 0,z\ge 0,s}z^\top b+C_q^1y+s\\
		&~~~~~~~~~~~~~~\text{s.t.}~[A^\top z]_i\ge \frac{\ell(\widehat{y}_i^\top\alpha)}{y^{q-1}}, ~i=1,...,N,\\
		&~~~~~~~~~~~~~~~~~~~~[A^\top z]_{N+1}\ge v_{N+1},\\
            &~~~~~~~~~~~~~~~~~~~~[A^\top z]_{N+2}\ge 0,\\
            &~~~~~~~~~~~~~~~~~~~~s\ge C_qv_{N+1}\left\|\frac{\alpha}{v_{N+1}}\right\|_*^{\frac{q}{q-1}}.
	\end{align*} Note that the objective function of above reformulation is linear. Moreover, given $\alpha$, this formulation has $N+9$ variables, $N+10$ linear constraints and $N+1$ convex constraint.
\end{itemize}
\item[(iii)] Let $\mathbb{D}_{\mathcal{Y}}(y_1,y_2)=||y_1-y_2||_{\mathcal{Y}}^q$ for any $y_1,y_2\in\mathcal{Y}$ and $q\in(1,\infty)$. {Assuming the loss function $\ell$ is convex and satisfies $\ell(x)-\ell(x_0) \leqslant L|x-x_0|^q+M, x \in \mathbb{R}$ for some $L,M>0$, and some $x_0\in\mathbb{R}$, the problem} \eqref{pro1:exp1-1} can be solved by the convex program
	\begin{align}
		&\inf_{\alpha\in\mathcal{A},\eta \geq 0,z\ge 0}~z^\top b \nonumber \\
		&~~~~~~~~\text{s.t.}~[A^\top z]_i\ge \sup _{z \in \mathbb{R}}\left\{z \cdot \widehat{y}_i^{\top} \alpha-\ell^*(z)+\eta \cdot \frac{|z|^{q^*}}{q^* q^{q^*-1}}\right\}, ~i=1,...,N, \label{1dcon} \\
		&~~~~~~~~~~~~~~[A^\top z]_{N+1}\ge ||\alpha||_*^q/\eta^{q-1},\nonumber\\
            &~~~~~~~~~~~~~~[A^\top z]_{N+2}\ge 0.\nonumber
	\end{align} 
    Note that the objective function of above reformulation is linear. Moreover, given $\alpha$, this formulation has $N+7$ variables, $N+8$ linear constraints and $N+1$ convex constraint.
\end{itemize}
        \item[(B)] Note that partial optimal transport model \eqref{m02} given by \cite{E22} is equivalent to problem \eqref{pro1:exp1-1} with set $\mathcal{V}$ given in Theorem \ref{thm2}. Applying the conjugate of the indicator function of $\mathcal{V}_+$ in Example \ref{partial-conjugate} to Propositions \ref{regular_exp} and \ref{general_exp}, we have the following results.
        \begin{itemize}
	\item [(i)]Let $\mathbb{D}_{\mathcal{Y}}(y_1,y_2)=||y_1-y_2||_{\mathcal{Y}}$ for any $y_1,y_2\in\mathcal{Y}$ and $\ell: \mathbb{R} \rightarrow \mathbb{R}$ be a convex and Lipschitz continuous function. 
	The problem \eqref{pro1:exp1-1} can be solved by the convex program
	\begin{align*}
	 &\inf_{\alpha \in\mathcal{A},z\ge0}~z^\top b'\\
		&~~~~~\text{s.t.}~ \left[(A')^\top z\right]_i\ge \ell(\widehat{y}_i^\top\alpha), ~i=1,...,N, \\
		&~~~~~~~~~~~\left[(A')^\top z\right]_{N+1}\ge \operatorname{Lip}(\ell)||\alpha||_*.
	\end{align*}
	\item[(ii)]Let $\mathbb{D}_{\mathcal{Y}}(y_1,y_2)=||y_1-y_2||_{\mathcal{Y}}^q$ for any $y_1,y_2\in\mathcal{Y}$ and $q\in(1,\infty)$. 
\begin{itemize}
	\item [(1)]Assume the loss function $\ell$ takes one of the following two forms, multiplied by $C>0$ :
	\begin{itemize}
		\item [(a)]$\ell_1(x)=x+b$ or $\ell_1(x)=-x+b$ with some $b \in \mathbb{R}$;
		\item[(b)]  $\ell_2(x)=\left|x-b_1\right|+b_2$ with some $b_1, b_2 \in \mathbb{R}$.
	\end{itemize}The problem \eqref{pro1:exp1-1} can be solved by the convex program
	\begin{align*}
		&\inf_{\alpha  \in\mathcal{A},v_{N+1}\ge 0,z\ge 0}~z^\top b'+C_qv_{N+1}\left\|\frac{\alpha}{v_{N+1}}\right\|_*^{\frac{q}{q-1}}\\
		&~~~~~~~~~~\text{s.t.}~[(A')^\top z]_i\ge \ell(\widehat{y}_i^\top\alpha), ~i=1,...,N,\\
		&~~~~~~~~~~~~~~~~[(A')^\top z]_{N+1}\ge v_{N+1},
	\end{align*}
	where $C_q=C^{\frac{q}{q-1}}\left(q^{\frac{1}{1-q}}-q^{\frac{q}{1-q}}\right)$.
	\item[(2)] Assume the loss function $\ell$ is of the form $\ell(x) = \left(C \cdot \ell_i(x)\right)^q$, where $C > 0$, $q \in (1, \infty)$, and $\ell_i$ is one of the following:	
	\begin{itemize}
		\item [(a)]$\ell_1(x)=(x-b)_{+}$with some $b \in \mathbb{R}$;
		\item[(b)] $\ell_2(x)=(x-b)_{-}$with some $b \in \mathbb{R}$;
		\item[(c)] $\ell_3(x)=\left(\left|x-b_1\right|-b_2\right)_{+}$with some $b_1 \in \mathbb{R}$ and $b_2 \geqslant 0$;
		\item[(d)] $\ell_4(x)=\left|x-b_1\right|+b_2$ with some $b_1 \in \mathbb{R}$ and $b_2>0$.
	\end{itemize} 
	The problem \eqref{pro1:exp1-1} can be solved by the convex program	
	\begin{align*}
		&\inf_{\alpha\in\mathcal{A},v_{N+1}\ge0,y\geq 0,z\ge 0}\left[z^\top b'+C_q^1y+C_qv_{N+1}\left\|\frac{\alpha}{v_{N+1}}\right\|_*^{\frac{q}{q-1}}\right]^q\\
		&~~~~~~~~~~~~~\text{s.t.}~[(A')^\top z]_i\ge \frac{\ell(\widehat{y}_i^\top\alpha)}{y^{q-1}}, ~i=1,...,N,\\
		&~~~~~~~~~~~~~~~~~~~[(A')^\top z]_{N+1}\ge v_{N+1},
	\end{align*}
	where $C_q^1=q^{\frac{1}{1-q}}-q^{\frac{q}{1-q}}$ and $C_q=C^{\frac{q}{q-1}}\left(q^{\frac{1}{1-q}}-q^{\frac{q}{1-q}}\right)$.  
\end{itemize}
\item[(iii)] Let $\mathbb{D}_{\mathcal{Y}}(y_1,y_2)=||y_1-y_2||_{\mathcal{Y}}^q$ for any $y_1,y_2\in\mathcal{Y}$ and $q\in(1,\infty)$. {Assuming the loss function $\ell$ is convex and satisfies $\ell(x)-\ell(x_0) \leqslant L|x-x_0|^q+M, x \in \mathbb{R}$ for some $L,M>0$, and some $x_0\in\mathbb{R}$, the problem} \eqref{pro1:exp1-1} can be solved by the convex program
	\begin{align*}
		&\inf_{\alpha\in\mathcal{A},\eta \geq 0,z\ge 0}~z^\top b'\\
		&~~~~~~~~\text{s.t.}~[(A')^\top z]_i\ge \sup _{z \in \mathbb{R}}\left\{z \cdot \widehat{y}_i^{\top} \alpha-\ell^*(z)+\eta \cdot \frac{|z|^{q^*}}{q^* q^{q^*-1}}\right\}, ~i=1,...,N, \\
		&~~~~~~~~~~~~~~[(A')^\top z]_{N+1}\ge ||\alpha||_*^q/\eta^{q-1}.
	\end{align*} 
\end{itemize}
    \end{itemize}
\begin{remark}\label{comparison_full}
In what follows, we present a more detailed discussion of the full optimal transport model in comparison with the analysis of \cite{N24}. Specifically, for the case with expectation as the risk measure and $\mathcal{N}_\omega = [\epsilon,1]$ for some $\epsilon \in (0,1]$, \cite{N24} were the first to study this model and established the following equivalence:
\begin{align*}
	&\inf_{\alpha\in\mathcal{A}}\sup_{\substack{\mathbb{Q} \in \mathbb{B}_{\delta_0}(\widehat{\mathbb{P}}) \\ \mathbb{Q}\left(X \in \mathcal{N}_\gamma(x_0)\right) \in[\epsilon,1]}}\mathbb{E}_{\mathbb{Q}}\left[\ell(Y, \alpha) \mid X \in \mathcal{N}_\gamma\left(x_0\right)\right]\\
	= &\begin{cases}\inf& \phi+(N \varepsilon)^{-1} \nu^{+}-N^{-1} \nu^{-} \\ \text {s.t. } & \alpha\in\mathcal{A},~\left(\lambda, s, \nu^{+}, \nu^{-}, \phi, \varphi, \psi\right) \in \mathcal{U} \\ & s_i \geq \sup _{{y}_i \in \mathcal{Y}}\left\{\ell\left({y}_i, \alpha\right)-\lambda_i \mathbb{D}_{\mathcal{Y}}\left({y}_i, \widehat{y}_i\right)\right\} \quad i=1,...,N,\end{cases}
\end{align*}
where $$
\mathcal{U} \triangleq\left\{\begin{array}{ll}
	\left(\lambda, s, \nu^{+}, \nu^{-}, \phi, \varphi, \psi\right) \in \mathbb{R}_{+}^N \times \mathbb{R}^N \times \mathbb{R}_{+} \times \mathbb{R}_{+} \times \mathbb{R} \times \mathbb{R}_{+} \times \mathbb{R}_{+}^N & \text { such that: } \\
	\phi+d_i \varphi+\psi_i-s_i \geq 0 & i=1,...,m \\
    \phi-d_i \varphi+\psi_i-s_i \geq 0 & i=m+1,...,N \\
	\nu^{+}-\nu^{-}+\left(\sum_{i= m+1}^N d_i-N \delta_0\right) \varphi-\sum_{i =1}^N \psi_i \geq 0 & \\
	\varphi-\lambda_i \geq 0 & i=1,...,N
\end{array}\right\}.
$$
Note that the objective function of the above reformulation is linear. Moreover, given $\alpha$, this formulation has $3N+4$ variables, $4N+4$ linear constraints, and $N$ convex constraints. A further comparison with \cite{N24} yields the following observations.
\begin{itemize}
	\item[(i)] 
    
    Although our formulations based on the projection property of \cite{W22} also involve the unsolved supremum problem in constraint \eqref{1dcon}, this problem is one-dimensional, whereas in \cite{N24} it is high-dimensional, resulting in substantially greater complexity.
	\item[(ii)] 
    
    For specific loss functions—such as the mean–variance and mean–CVaR loss functions studied in \cite{N24}—the $N$ supremum problems in the constraints \eqref{1dcon} admit explicit forms, which reduce to SOC constraints. Compared with \cite{N24}, our reformulation is considerably simpler and more tractable, requiring simpler constraints (SOC constraints) while reformulation in \cite{N24} requires more complex constraint (SDP constraints). See Remark~\ref{comparison_mean_variance} for a detailed comparison under the mean–variance loss function with the squared $\ell_2$-norm, and Remark~\ref{comparison_mean_CVaR} for the corresponding comparison under the mean–CVaR loss function with the squared $\ell_2$-norm.
	\item[(iii)] 
    Our reformulations are considerably simpler largely thanks to the union-ball characterization, which leverages the structural properties of standard Wasserstein DRO and requires invoking duality only once, introducing a single multiplier $\eta$. By contrast, the analysis in \cite{N24}, which overlooks this structural property and instead relies on a more cumbersome break-and-conquer approach, applies duality repeatedly—most clearly reflected in the $N$ multipliers $\lambda \in \mathbb{R}_+^N$ that grow with $N$.

\end{itemize}
\end{remark}

\end{document}